%% file: exact_solution_DDIOP_MILP_v8_arXiv_EN.tex
\documentclass[reqno]{amsart}

\usepackage[utf8]{inputenc} 
\usepackage[T1]{fontenc}    
\usepackage{url}            
\usepackage{booktabs}       
\usepackage{amsfonts}       
\usepackage{nicefrac}       
\usepackage{microtype}      
\usepackage[]{xcolor}         

\usepackage{amsmath,amssymb,amsfonts,mathrsfs}
\usepackage{nicefrac}
\usepackage{algorithmic}
\usepackage{algorithm}
\usepackage{graphicx}

\usepackage{multirow}

\usepackage{wrapfig}

\usepackage{booktabs}

\usepackage{tikz}
\usetikzlibrary{positioning, arrows.meta, calc, matrix}

\usepackage{xcolor}

\definecolor{strblue}{HTML}{0F1ED2}
\definecolor{strred}{HTML}{E61E8C}

\usepackage{natbib}

\usepackage{comment}

\usepackage{aliascnt}
\usepackage{hyperref}
\usepackage{cleveref}

\theoremstyle{definition}

\newtheorem{theorem}{Theorem}[section]
\crefname{theorem}{Theorem}{Theorems}
\Crefname{theorem}{Theorem}{Theorems}

\newaliascnt{definition}{theorem}
\newtheorem{definition}[definition]{Definition}
\aliascntresetthe{definition}
\crefname{definition}{Definition}{Definitions}
\Crefname{definition}{Definition}{Definitions}

\newaliascnt{proposition}{theorem}
\newtheorem{proposition}[proposition]{Proposition}
\aliascntresetthe{proposition}
\crefname{proposition}{Proposition}{Propositions}
\Crefname{proposition}{Proposition}{Propositions}

\newaliascnt{lemma}{theorem}
\newtheorem{lemma}[lemma]{Lemma}
\aliascntresetthe{lemma}
\crefname{lemma}{Lemma}{Lemmas}
\Crefname{lemma}{Lemma}{Lemmas}

\newaliascnt{corollary}{theorem}
\newtheorem{corollary}[corollary]{Corollary}
\aliascntresetthe{corollary}
\crefname{corollary}{Corollary}{Corollaries}
\Crefname{corollary}{Corollary}{Corollaries}

\newaliascnt{remark}{theorem}
\newtheorem{remark}[remark]{Remark}
\aliascntresetthe{remark}
\crefname{remark}{Remark}{Remarks}
\Crefname{remark}{Remark}{Remarks}

\newaliascnt{example}{theorem}
\newtheorem{example}[example]{Example}
\aliascntresetthe{example}
\crefname{example}{Example}{Examples}
\Crefname{example}{Example}{Examples}

\newaliascnt{assumption}{theorem}
\newtheorem{assumption}[assumption]{Assumption}
\aliascntresetthe{assumption}
\crefname{assumption}{Assumption}{Assumptions}
\Crefname{assumption}{Assumption}{Assumptions}

\renewcommand{\proofname}{\textbf{proof}}

\numberwithin{equation}{section}

\crefname{equation}{Equation}{Equations}
\crefname{figure}{Figure}{Figures}
\crefname{table}{Table}{Tables}
\crefname{algorithm}{Algorithm}{Algorithms}

\crefname{section}{}{}
\creflabelformat{section}{\S #2#1#3}
\crefname{subsection}{}{}
\creflabelformat{subsection}{\S #2#1#3}
\crefname{appendix}{Appendix}{Appendixes}

\DeclareMathOperator*{\argmin}{arg\,min} 
\DeclareMathOperator*{\argmax}{arg\,max} 

\DeclareMathOperator{\Proj}{Proj}

\DeclareMathOperator{\dom}{dom}

\newcommand{\rest}[1]{\Big\rvert_{#1}}

\title[Exact Solution to Data-Driven Inverse Optimization of MILPs]{Exact Solution to Data-Driven Inverse Optimization of MILPs \\ in Finite Time via Gradient-Based Methods}

\author{%
  Akira Kitaoka	
}
\address{NEC Corporation, 1753 Shimonumabe, Nakahara-ku, Kawasaki, Kanagawa, Japan }
\email{akira-kitaoka@nec.com}


\begin{document}

\begin{abstract}

A data-driven inverse optimization problem (DDIOP) is the problem of estimating the objective-function parameters (weights) that explain observed optimal-solution data, and it arises in many applications, including mixed integer linear programming (MILP). In inverse optimization for MILPs, the prediction error of the features is discontinuous with respect to the weights, so applying gradient-based optimization directly is difficult. In this paper we focus on the suboptimality loss. This loss attains its minimum value, zero, if and only if the weights are exactly consistent with the observed data. Using the fact that this loss is convex and piecewise linear and that its set of minimizers has a relative interior point, we show that a broad class of gradient-based optimization methods, including projected subgradient descent, reaches exact consistency with the observed data in finitely many iterations (an exact solution is obtained in finite time). This guarantee holds for plain projected subgradient descent with standard diminishing step sizes, requiring neither prior knowledge of the optimal value nor smoothness of the objective. Through numerical experiments, we confirm this finite-step attainment behavior.

\end{abstract}

\maketitle


\section{Introduction}
Inverse optimization is the problem of estimating an objective function or its parameters from observed optimal solutions, and is one of the fundamental problems spanning mathematical optimization, machine learning, and operations research~\citep{ahuja2001inverse,heuberger2004inverse,chan2019inverse,chan2023inverse}. Its applications are wide-ranging, including transportation~\citep{bertsimas2015data}, power systems~\citep{birge2017inverse}, healthcare~\citep{chan2022inverse}, advertisement scheduling~\citep{Suzuki-2019-TV}, and inverse reinforcement learning~\citep{ng2000algorithms}.

In this paper, we treat the case where the objective function of the forward problem is given as a linear combination of known features, with particular emphasis on settings that include mixed integer linear programs (MILPs).
Let $\mathcal{X} \subset \mathbb{R}^{d_{\mathcal{X}}}$ be the decision space. For each state $s \in \mathcal{S}$, let $X(s) \subset \mathcal{X}$ be a bounded and closed feasible region. For each $i = 1, \ldots, d$, let $f_i \colon \mathcal{X}\times\mathcal{S} \to \mathbb{R}$ be a piecewise linear mapping, and define $f := (f_1, \ldots, f_d)$. Let $\Theta \subset \mathbb{R}^d$ be a closed convex set representing the weight space, and let $\theta = (\theta_1, \ldots, \theta_d) \in \Theta$ denote the weights.
We define the optimal solution to the forward problem and its associated feature vector as follows:
\begin{equation}\label{eq:FOP_linear}
    x^* (\theta , s ) \in \argmax_{x \in X(s)} \theta^{\top} f ( x, s ) = \sum_{i=1}^d \theta_i f_i (x, s) , 
    \quad a^* (\theta , s ) = f (x^* (\theta , s ), s).
\end{equation}
Let $\theta^* \in \Theta$ denote the (unknown) true weights that generated the observed data. For each $n = 1, \ldots, N$, suppose that we are given $s^{(n)} \in \mathcal{S}$ and the observed optimal solution $x^{(n)} = x^*(\theta^*, s^{(n)})$. The data-driven inverse optimization problem (DDIOP) associated with the forward optimization problem~\cref{eq:FOP_linear} is defined as the problem of finding $\theta \in \Theta$ satisfying the following condition for all $n = 1, \ldots, N$:
\begin{equation}
    x^{(n)} \in \argmax_{x \in X(s^{(n)})} \theta^{\top} f ( x, s^{(n)} ) = \sum_{i=1}^d \theta_i f_i (x, s^{(n)}).
    \label{eq:IOP_linear}
\end{equation}
An illustration of solving the DDIOP for MILPs is shown in \cref{fig:IOP_image}.

\input{IOP_image}

As evaluation criteria for whether \cref{eq:IOP_linear} has been solved, one may consider the prediction loss of features (PLF)~\citep{aswani2018inverse,chan2019inverse,babier2021ensemble,chan2023inverse,ferber2023surco,liang2023data} and the suboptimality loss~\citep{Mohajerin-2018-Data}.
The PLF, denoted by $\ell_{\mathrm{plf}}(\theta)$, is given by
$
    \ell_{\mathrm{plf}}(\theta)
    =\sum_{n=1}^N
    \left\lVert a^*(\theta, s^{(n)}) - a^{(n)} \right\rVert_2^2 / N,
$
where $a^{(n)} = f(x^{(n)}, s^{(n)})$.
If the PLF equals $0$, then \cref{eq:IOP_linear} is solved.

However, in inverse optimization for MILPs, $a^*(\theta,s^{(n)})$ can vary discontinuously with respect to the weight $\theta$, so the PLF is in general discontinuous as well. For this reason, directly minimizing the PLF is not easy from the viewpoint of gradient-based optimization~\citep[cf.][]{Beck-2017-First,hazan2022introduction,garrigos2023handbook}.

On the other hand, as a Lipschitz continuous and convex loss function for assessing whether the DDIOP for MILPs~(\cref{eq:IOP_linear}) has been solved, one may consider the suboptimality loss
\[
    \ell_{\mathrm{sub}} (\theta)
    :=
    \frac{1}{N}\sum_{n=1}^N \left(
        \theta^{\top} a^*(\theta , s^{(n)})
        -
        \theta^{\top} a^{(n)}
    \right),
\]
as proposed in~\citep{Mohajerin-2018-Data}.
Solving~\cref{eq:IOP_linear} is equivalent to finding $\theta$ such that $\ell_{\mathrm{sub}}(\theta)=0$.
Therefore, by minimizing the suboptimality loss, one may expect to obtain a solution to~\cref{eq:IOP_linear}.

For minimizing the suboptimality loss, a variety of standard first-order (online) optimization methods have been used in the inverse optimization literature
\citep{Barmann-2018-online,besbes2021online,besbes2025contextual,gollapudi2021contextual,Kitaoka-2023-convergence-IRL,Kitaoka-2023-imitation-WIRL,sakaue2025online}.
Typical examples include projected subgradient descent (PSGD)
\citep{boyd2003subgradient,Beck-2017-First},
multiplicative weights update (MWU) \citep{arora2012multiplicative},
online Newton step (ONS) \citep{hazan2007logarithmic},
and MetaGrad \citep{van2016metagrad,van2021metagrad}.
These methods come with well-established regret/convergence guarantees for Lipschitz convex losses, which imply that $\ell_{\mathrm{sub}}$ decreases as the number of iterations $T$ grows.

However, much of the existing literature primarily provides \emph{asymptotic} guarantees, i.e., $\ell_{\mathrm{sub}}(\theta^t)$ $\to 0$ only in the limit as $T\to\infty$.
In contrast, whether one can \emph{exactly solve} the DDIOP for MILPs \emph{within finitely many iterations}---namely, guarantee $\ell_{\mathrm{sub}}=0$ after a finite number of updates---has not been sufficiently clarified.

In this paper, we show that, for a broad class of data-driven inverse optimization including MILPs, gradient-based optimization methods minimizing the suboptimality loss reach the minimum loss value $0$ in finitely many iterations. The key is the geometric structure that the suboptimality loss is convex and piecewise linear and that its set of minimizers has a relative interior point in the affine hull of the parameter space. Using this structure, the asymptotic guarantees known for standard first-order methods can be strengthened into guarantees of exact attainment in finitely many iterations.

\subsection*{Contributions}
The contributions of this paper are as follows.

\begin{itemize}
\item As a general theory for convex piecewise-linear functions, we show that when the set of minimizers has a relative interior point, a broad class of first-order methods reaches the minimum in finitely many iterations. This guarantee  holds for projected subgradient descent with standard diminishing step sizes, one of the simplest implementable methods. To the best of our knowledge, this is the first theory that guarantees exact finite-step attainment of the minimum for subgradient methods that do not require prior knowledge of the optimal value. The differences from existing finite-termination theories based on weak sharp minima, which target gradient projection methods and sequential quadratic programming in differentiable settings as well as proximal point methods, are described in \cref{sec:weak_sharp_minima}.

\item We show that the suboptimality loss arising in data-driven inverse optimization including MILPs is convex, Lipschitz continuous, and piecewise linear, and that, for almost every true weight, its set of minimizers has a relative interior point.

\item By combining the above, we show that gradient-based optimization methods, including projected subgradient descent, achieve exact consistency with the observed data in finitely many iterations.

\item For gradient-based optimization methods including projected subgradient descent, we provide an upper bound on the number of iterations in terms of the constant $\gamma(\ell_{\mathrm{sub}})$, and also show finite-step attainment of the PLF minimum.

\end{itemize}

\begin{table}[t]
    \vspace{-\intextsep}
    \caption{Comparison of approaches for solving the DDIOP~\cref{eq:IOP_linear} for MILPs (suboptimality loss).
    Whereas prior work primarily provides asymptotic guarantees---i.e., the loss decreases as $T\to\infty$---the present paper differs in that it guarantees $\ell_{\mathrm{sub}}=0$ within finitely many iterations.
    For a detailed comparison (explicit rates and constants), see \cref{tab:pro_con_SL_detail}.}
    \label{tab:pro_con_SL_ess}
    \centering
    \setlength{\tabcolsep}{3.5pt} 
    \renewcommand{\arraystretch}{1.15} 
    \begin{tabular}{p{2.5cm} p{7cm} c}
        \toprule
        Method (representative examples) & Typical guarantee (informal) & $\ell_{\mathrm{sub}}=0$ \\
        \midrule
        Point-set search (UPA/RPA) &
        Approximation based on covering error (strong dimension dependence) &
        No \\
        \hline
        Online optimization (MWU / ONS / MetaGrad, etc.) &
        Asymptotic decrease of $\ell_{\mathrm{sub}}$ (e.g., $O(1/\sqrt{T})$, $O(\log T/T)$) &
        No \\
        \hline
        This paper (gradient-based) &
        Finite-step attainment of $0$ by leveraging the structure of $\ell_{\mathrm{sub}}$ (convex, piecewise linear, plus an interior-point property) &
        \textbf{Yes} \\
        \bottomrule
    \end{tabular}
    \vspace{-\intextsep}
\end{table}

The remainder of this paper is organized as follows.
In \cref{sec:related_work}, we review related work.
In \cref{sec:problem_setup_and_algorithm}, we describe the DDIOP setting and present our proposed algorithm, which applies gradient-based optimization to the suboptimality loss.
In \cref{sec:main_results}, we show that the proposed method solves the DDIOP for MILPs within finitely many iterations and provide an upper bound on the required number of iterations. We further show that, when PSGD is instantiated within the proposed framework, the PLF can be driven to $0$ in finitely many iterations, together with an upper bound on the iteration complexity.
In \cref{sec:proof_overview}, we provide an intuitive explanation of the proofs of the main theorems presented in \cref{sec:main_results}.
The subsequent sections then give the complete proofs of the main results and the technical lemmas, and survey known results such as the relation between the prediction loss and the suboptimality loss and existing regret analyses.
In \cref{sec:experiment}, we report numerical experiments, demonstrating that the proposed method reaches the minimum value of the PLF within finitely many iterations.
Finally, \cref{sec:conclusion} concludes the paper and discusses directions for future work.

\section{Related Work}\label{sec:related_work}

Research on inverse optimization has classically developed for combinatorial and network optimization~\citep{ahuja2001inverse,heuberger2004inverse}, and more recently the data-driven framework---which estimates an objective function or its parameters from observed data---has been actively studied from both theoretical and applied perspectives~\citep{chan2019inverse,chan2023inverse} (the applications and the loss functions---the PLF and the suboptimality loss---were reviewed in the introduction; for systematic surveys see \citet{chan2019inverse,chan2023inverse}). In this section, we avoid restating the background covered in the introduction and focus on the work directly related to the positioning of this paper.

\paragraph{First-order/online optimization and finite-step attainment.}
When the forward problem is convex and the KKT conditions are available, approximate recovery guarantees and stability analyses are well developed~\citep{aswani2018inverse,Mohajerin-2018-Data,chan2023inverse}, whereas in discrete optimization including MILPs---the subject of this paper---the optimal-solution map can be discontinuous, so analyses effective in the continuous setting cannot be applied directly.
For the convex suboptimality loss, regret and best-iterate analyses based on first-order and online optimization methods---projected subgradient descent, multiplicative weights update, online Newton step, and MetaGrad---have been carried out~\citep{Barmann-2018-online,besbes2021online,besbes2025contextual,gollapudi2021contextual,Kitaoka-2023-convergence-IRL,sakaue2025online} (for details of existing regret analyses see \Cref{sec:known_regret_analysis}, and for the offline implications see \Cref{tab:pro_con_SL_ess}). However, these guarantees remain asymptotic, of the form $\ell_{\mathrm{sub}}(\theta^t)\to 0$ or $\min_{1\leq t\leq T}\ell_{\mathrm{sub}}(\theta^t)\to 0$, and do not yield exact attainment of $\ell_{\mathrm{sub}}=0$ at some finite iteration. This paper differs from these works in that it strengthens these asymptotic guarantees into finite-step exact attainment (\Cref{sec:main_results}).

\paragraph{Computational complexity and evaluation of the number of iterations.}
This paper sharpens the qualitative guarantee of attaining exact consistency in finitely many iterations into a quantitative form: an iteration upper bound expressed via the geometric constant $\gamma(\ell_{\mathrm{sub}})$. Our iteration upper bound is described via the problem-dependent geometric constant $\gamma(\ell_{\mathrm{sub}})$, and for PSGD (SRSS, SRSL; defined in \Cref{exa:PSGD}) it takes the form $T=O(1/\gamma(\ell_{\mathrm{sub}})^2)$ up to polynomial factors.

\paragraph{Positioning of this paper.}
In light of the above, this paper is positioned at the intersection of two streams of work: the asymptotic analysis of first-order/online optimization for the suboptimality loss, and the problem awareness of achieving exact consistency with the observed data in discrete inverse optimization including MILPs. Its novelty is to provide a general finite-step attainment principle for convex piecewise-linear functions and, by applying it to the suboptimality loss of data-driven inverse optimization including MILPs, to show exact solvability in finitely many iterations via gradient-based optimization methods. Below, we review supplementary related work topic by topic.

\paragraph{Methods for smoothing the prediction loss of optimal solutions}
\label{sec:smoothing_PLF}

To avoid the discontinuity of the PLF or of the discrete optimal-solution map, approaches that smooth the problem have been proposed.

The input-perturbation approach of \citet{berthet2020learning} requires multiple steps in order to obtain a solution and guarantee its uniqueness.
Moreover, since each gradient computation involves a Monte Carlo procedure (suppose that $M$ samples are drawn), the overall computational cost of learning becomes substantial.
This effectively means that one must solve $M$ times more linear programs.
In addition, the required number of samples $M$ for gradient estimation must itself be selected or estimated.

For LPs, \citet{wilder2019melding} proposed a method that adds a regularization term to the objective function.
However, to apply this method, $X(s)$ must be described by linear constraints.
Hence, it is not directly applicable when $X(s)$ is nonconvex, as in ILPs or MILPs.
To address this limitation, \citet{ferber2020mipaal} proposed reducing an MILP to an LP by using Gomory cuts \citep{gomory1960algorithm} and globally valid cuts \citep{balas1996gomory}.
Smoothing the PLF within the framework of \citet{ferber2020mipaal} provides one approach to solving the PLF minimization problem via backpropagation (cf.\ \citep{ferber2023surco}).
However, to the best of our knowledge, the method proposed in \citet{ferber2023surco} does not provide an approximation-error analysis, nor does it establish theory guaranteeing uniqueness of the solution.
Furthermore, it remains unclear what level of accuracy and computational effort is required for the Gomory cuts and globally valid cuts used in \citet{ferber2020mipaal}.
Thus, these methods generally solve an approximate problem and do not guarantee exact consistency for the original discrete problem in finitely many iterations. This paper does not use smoothing, but exploits the geometric structure of the suboptimality loss itself.

\paragraph{The case where the objective is strongly convex and $L$-smooth and the constraint set is convex}

Assume that the objective function in \cref{eq:FOP_linear} is $\mu$-strongly convex and $L$-smooth, and that $X(s)$ is convex.
In this case, minimizing the PLF is equivalent to minimizing the suboptimality loss (see \cref{sec:sub_0_eqs_pre}).
Moreover, by running projected gradient descent on the suboptimality loss, the resulting sequence of PLF values converges linearly (see \cref{sec:subopt_PSGD_to_pre}).

\paragraph{PLF minimization via the suboptimality loss}

Assume that the feature vector $a^*(\theta,s^{(n)})$ returns, among the solutions to
\[
    \argmax_{f(x, s)\in f(X(s), s)}\ \theta^{\top} f(x, s),
\]
the one that is minimal in the lexicographic order.
Under this assumption, \citet{Kitaoka-2023-imitation-WIRL} showed that $a^*(\theta,s^{(n)})=a^{(n)}$ holds whenever a subgradient of the suboptimality loss is equal to $0$.

\paragraph{SPO loss}

In the context of inverse optimization, there is a line of work using the SPO loss \citep{Mohajerin-2018-Data,elmachtoub2022smart}.
The constant $\gamma(\ell_{\mathrm{sub}})$ can be upper-bounded by employing the SPO loss \citep{Mohajerin-2018-Data,elmachtoub2022smart}.
For details, see \cref{sec:SPO_loss}.

\paragraph{Relation to weak sharp minima}
\label{sec:weak_sharp_minima}

The structure exploited in the finite-step attainment of \Cref{theo:grad_based_opt_achieve_min_SL}---that the set of minimizers has a relative interior point (in the affine hull of the parameter space)---is closely related to the notion of weak sharp minima \citep{burke1993weak}.
Weak sharp minima generalize sharp (strongly unique) minima to the case of a non-unique solution set, and are characterized by the objective growing at least linearly with the distance to the solution set, i.e., $\ell(\theta)-\min_\Theta\ell \ge \alpha\,\mathrm{dist}(\theta,\argmin_\Theta\ell)$ for some $\alpha>0$.

Finite termination of optimization algorithms under weak sharp minima has been studied classically. \citet{burke1993weak} showed that the sequential quadratic programming method and the gradient projection method reach the solution set in finitely many iterations under weak sharp minima (the finite termination of the gradient projection method is also noted by \citet{polyak1987introduction}), and \citet{ferris1991finite} established finite termination of the proximal point algorithm under a sharp minimum.

Our contribution differs from these in the following respects. First, whereas these existing results mainly concern differentiable settings (the gradient projection and sequential quadratic programming methods) or the proximal point algorithm, we provide finite-step attainment, in a unified manner, for a broad class of gradient-based optimization methods---including projected subgradient descent---applied to the piecewise-linear (nonsmooth) suboptimality loss (\Cref{theo:grad_based_opt_achieve_min_SL}). In particular, the Polyak step size---the standard device for obtaining stronger guarantees for subgradient methods---requires prior knowledge of the minimum value $\min_\Theta\ell$ and, moreover, falls outside our framework (\Cref{assu:gradient_based_opt_independ_ell}); in contrast, our finite-step attainment holds for standard diminishing step sizes (SRSS and SRSL) that require no such knowledge. Second, and more importantly, beyond merely establishing whether finite termination occurs, we give an upper bound on the number of iterations via the constant $\gamma(\ell_{\mathrm{sub}})$. This refines the qualitative guarantee of ``reaching exact consistency in finitely many iterations'' into a quantitative form: an iteration upper bound expressed via the geometric constant $\gamma(\ell_{\mathrm{sub}})$---a quantification not provided by the classical weak-sharp-minima theory of finite termination.

\section{Problem Setup and Algorithms}
\label{sec:problem_setup_and_algorithm}

We impose the following assumptions to specify the problem setting studied in this paper.
\begin{assumption}
    \label{assu:WIRL}
    Let the nonempty set $\mathcal{S}$ be the state space, and let the nonempty set $\mathcal{X} \subset \mathbb{R}^{d_{\mathcal{X}}}$ be the decision space. Let $f = (f_1 , \ldots , f_d)\colon \mathcal{X}\times\mathcal{S} \to \mathbb{R}^d$ be a mapping such that each component $f_i$ is a piecewise linear function.
    Let the weight space $\Theta \subset \mathbb{R}^d$ be a bounded, closed, and convex set.
    For each state $s \in \mathcal{S}$, let the feasible region $X(s)$ be a finite union of bounded, closed, and convex polytopes that are subsets of $\mathcal{X}$.
    Let $\mathcal{D} := \left\{ \left( s^{(n)}, x^{(n)} \right) \right\}_{n=1}^N$ be a set of training data such that the samples $s^{(n)} \in \mathcal{S}$ are generated from an unknown distribution $\mathbb{P}_{\mathcal{S}}$, and there exists an unknown $\theta^* \in \Theta$ satisfying, for each $n =1 , \ldots , N$,
    $
       x^{(n)} = x^* ( \theta^* , s^{(n)})
    $.
\end{assumption}

As a technical condition, we further assume the following.
\begin{assumption}
    \label{assu:WIRL-uniqueness}
    For each $n=1,\ldots,N$, the feature $a^*(\theta^*, s^{(n)})$ is uniquely determined.
\end{assumption}

\cref{assu:WIRL-uniqueness} is a natural assumption, in view of the following statement.
\begin{lemma}
    \label{lem:Psi_set_is_almost_Phi}
    Suppose \cref{assu:WIRL} holds and let $\Theta=\Delta^{d-1}\subset \mathbb{R}^d$.
    Then, for $\theta\in\Delta^{d-1}$ almost everywhere (with respect to the measure on $\Delta^{d-1}$ induced by the Lebesgue measure), the feature $a^*(\theta,s^{(n)})$ is uniquely determined for every $n=1,\ldots,N$.
\end{lemma}
The proof of this lemma is given in \cref{sec:proof_lem:Psi_set_is_almost_Phi}.

In this paper, based on the following proposition, we apply a gradient-based optimization method to the suboptimality loss $\ell_{\mathrm{sub}}$ and minimize it.

\begin{proposition}[{\citealt[Proposition~3.1]{Barmann-2018-online}; \citealt[Lemma~4.8]{Kitaoka-2023-convergence-IRL}}]
    \label{prop:SL_is_Lipschitz_convex}
    Suppose that \cref{assu:WIRL} holds.
    Then the following statements hold:
    (A) the suboptimality loss $\ell_{\mathrm{sub}}$ is convex;
    (B) the suboptimality loss $\ell_{\mathrm{sub}}$ is Lipschitz continuous; and
    (C) a subgradient of the suboptimality loss $\ell_{\mathrm{sub}}$ at $\theta\in\Theta$ is given by
    $ g(\theta)
        :=\sum_{n=1}^N \bigl(a^*(\theta, s^{(n)}) - a^{(n)}\bigr) / N $.    
\end{proposition}

\begin{wrapfigure}{r}{8cm}
    \vspace{-2\intextsep}
    \begin{minipage}{\linewidth}
    \begin{algorithm}[H]
        \caption{Minimization of the suboptimality loss}\label{alg:intention-WIRL-gradual-decay}
        \begin{algorithmic}[1]
            \STATE Initialize $\theta^1 \in \Theta$
            \FOR{$t = 1 , \ldots, T-1$}
                \STATE For each $n = 1 , \ldots , N$, solve for $x^*(\theta^t, s^{(n)})$
                \STATE $\theta^{t+1} \leftarrow \mathrm{update}_{t}\!\left(\{\theta^{t^\prime}\}_{t^\prime = 1}^t \mid \ell_{\mathrm{sub}}, g\right)$
            \ENDFOR
            \RETURN $\displaystyle \theta_{\mathrm{best}}^{T} \in \argmin_{\theta \in \{\theta^{t}\}_{t=1}^T} \ell_{\mathrm{sub}}(\theta)$
        \end{algorithmic}
    \end{algorithm}
    \end{minipage}
    \vspace{-\intextsep}
\end{wrapfigure}
Hereafter, when it is clear from the context,
$\min_{\Theta} \ell = \min_{\theta \in \Theta} \ell (\theta)$ and
$\argmin_{\Theta} \ell = \argmin_{\theta \in \Theta} \ell (\theta)$
are used as abbreviations.
Let $\ell \colon \Theta \to \mathbb{R}$ be a Lipschitz-continuous convex function, and let $\partial \ell$ denote the subdifferential of $\ell$.
Let $\nabla \ell \colon \Theta \to \mathbb{R}^d$ be any (sub)gradient selection satisfying, for every $\theta$, $\nabla \ell (\theta) \in \partial \ell (\theta)$.
For $t \in \mathbb{Z}_{\geq 1}$, let the update map be
$\mathrm{update}_{t} \colon \Theta^t \times \mathbb{R}^{t} \times (\mathbb{R}^{d})^{t} \to \Theta$.
Moreover, define
\[
    \begin{split}
        & \mathrm{update}_t \left( \{ \theta^{t^\prime} \}_{t^\prime =1}^t \mid \ell , \nabla \ell \right) 
         :=\mathrm{update}_{t} ( \{ \theta^{t^\prime} \}_{t^\prime = 1}^t , \{ \ell (\theta^{t^\prime}) \}_{t^\prime = 1}^t , \{ \nabla \ell (\theta^{t^\prime}) \}_{t^\prime = 1}^t )    
    \end{split}
\]
The resulting procedure is described in \cref{alg:intention-WIRL-gradual-decay}.

\begin{example}[Projected subgradient descent]
    \label{exa:PSGD}
    Let $\alpha_{t} \colon \mathbb{R}^d \times \mathbb{R} \times \mathbb{R}^{d} \to \mathbb{R}_{\geq 0}$ be a mapping, and refer to $\{ \alpha_{t} \}_{t \in \mathbb{Z}_{\geq 1}}$ as the \emph{learning rate}.
    Let $\Proj_{\Theta}$ denote the projection onto $\Theta$.
    We identify the mapping $\alpha_t$ with its value $\alpha_t ( \theta^t , \ell(\theta^t) , \nabla \ell(\theta^t))$ when no confusion arises.
    We define projected subgradient descent (PSGD) \citep[cf.][]{boyd2003subgradient,Beck-2017-First} as the update rule $\{ \mathrm{update}_t \}_t$ in \cref{alg:intention-WIRL-gradual-decay} by
    $
        \mathrm{update}_t \left(\{ \theta^{t^\prime} \}_{t^\prime =1}^t | \ell , \nabla \ell \right)
        = \Proj_{\Theta} \left( \theta^t - \alpha_{t}  \nabla \ell( \theta^t) \right) .
    $
    The learning rate $\alpha_{t}$ is said to be a \emph{nonsummable step size (NSS)} if there exists a sequence $\{\beta_{t} \}_{t} \subset \mathbb{R}$ such that
        \begin{equation}
            \label{eq:seq_nonsummbale}
            \beta_{t} > 0 ,
            \quad 
            \lim_{t \to \infty} \frac{\sum_{t=1}^\infty \beta_{t}^2}{\sum_{t=1}^\infty \beta_{t}} = 0 
        \end{equation}
    and $\alpha_{t} = \beta_{t}$.
    The learning rate $\alpha_{t}$ is said to be a \emph{square root step size (SRSS)} if, for some $\beta > 0$, it is an NSS learning rate with $\beta_{t} = \beta t^{-1/2}$.
    The learning rate $\alpha_{t}$ is said to be a \emph{nonsummable step length (NSL)} if, for some sequence $\{\beta_{t} \}_{t} \subset \mathbb{R}$ satisfying \cref{eq:seq_nonsummbale}, we set $\alpha_{t} = \beta_{t} \| \nabla \ell(\theta^t) \|^{-1}$ when $\nabla \ell(\theta^t) \not = 0$, and $\alpha_{t} = 0$ when $\nabla \ell(\theta^t)  = 0$.
    The learning rate $\alpha_{t}$ is said to be a \emph{square root step length (SRSL)} if, for some $\beta > 0$, it is an NSL learning rate with $\beta_{t}  = \beta t^{-1/2}$.
    The learning rate $\alpha_{t}$ is said to be \emph{Polyak} if the minimum value $\min_{\Theta} \ell$ is known and
    $ 
        \alpha_{t} = 
        \left( \ell (\theta^t) - \min_{\Theta} \ell \right)\| \nabla \ell (\theta^t ) \|^{-2}
    $
    holds.
\end{example}

We present below an example satisfying \cref{assu:WIRL}, together with an implementation example of \cref{alg:intention-WIRL-gradual-decay}.
\begin{example}
    \label{exa:mip}
    In \cref{assu:WIRL}, let
    $f(x, s) = x$ (i.e., the features are the decision variables themselves).
    Under this setting, \cref{eq:FOP_linear} becomes an MILP.
    Let the sequence $\{ \mathrm{update}_t \}_t$ be PSGD, and implement the projection $\Proj_{\Delta^{d-1}}$ onto $\Theta = \Delta^{d-1}$ as in \citet{Wang-13-projection}.
    Then one can implement \cref{alg:intention-WIRL-gradual-decay}.
    This implementation is identical to \citet[Algorithm~1]{Kitaoka-2023-convergence-IRL}.
\end{example}

\begin{remark}\label{remark:mip}
We explain why it is preferable to exclude $0$ from the weight space $\Theta$.
At the origin $0$, one has $\ell_{\mathrm{sub}}(0)=0$, and thus $0$ attains the minimum value of the suboptimality loss.
Consequently, for the DDIOP for MILPs (\cref{eq:IOP_linear}), $\theta=0$ is a solution.

However, the optimizer $x^*(0,s^{(n)})$ may be any point in $X(s^{(n)})$; hence, unless $X(s^{(n)})$ is a singleton, the optimizer is not uniquely determined.
Therefore, even if the true weight were $\theta^*=0$, Assumption~\ref{assu:WIRL-uniqueness} would not be satisfied.
That is, even if learning returns $\theta=0$, such a $\theta$ does not coincide with the true weight $\theta^*$.
This issue is common to many IOPs for LPs, regardless of whether the setting is online or offline \citep{bertsimas2015data,Mohajerin-2018-Data,Barmann-2018-online,Chen-2020-Online,sun2023maximum,Kitaoka-2023-convergence-IRL,Kitaoka-2023-imitation-WIRL}.
\end{remark}

\section{Main Results}
\label{sec:main_results}

\subsection{Finite-time exact solvability of DDIOP for MILP (gradient based optimization method)}
\label{sec:solve_MILP_DDIOP_grad_theory}

For gradient-based optimization methods, we make the following assumptions: namely, that the update rule depends only on the past sequence of iterates and (sub)gradients, and that a convergence rate in terms of the best iterate is guaranteed, respectively.

\begin{assumption}
    \label{assu:gradient_based_opt_independ_ell}
    For any $t \in \mathbb{Z}_{\geq 1}$, we assume that $\mathrm{update}_{t} \colon \Theta^t \times \mathbb{R}^{t} \times(\mathbb{R}^{d})^{t} \to \Theta$ does not depend on the second component space $\mathbb{R}^{t}$.
\end{assumption}

\begin{assumption}
    \label{assu:gradient_based_opt_Q}
    For any $L$-Lipschitz-continuous convex function $\ell \colon \Theta \to \mathbb{R}$, any initial point $\theta^1 \in \Theta$, and any $t \in \mathbb{Z}_{\geq 1}$, define the iterates inductively by
    $\theta^{t+1} = \mathrm{update}_t \left( \{ \theta^{t^\prime} \}_{t^\prime =1}^t \mid \ell ,\nabla \ell \right)$.
    Then there exists a nonincreasing function $Q_{L , \theta^1 } \colon \mathbb{R}_{> 0 } \to \mathbb{R}$ such that, for every $\varepsilon > 0$, if $T \geq Q_{L , \theta^1 } (\varepsilon)$, then for any $L$-Lipschitz-continuous convex function $\ell \colon \Theta \to \mathbb{R}$, the best iterate is $\varepsilon$-accurate, that is,
    $\min_{t=1 , \ldots , T} \ell (\theta^t) - \min_{\Theta} \ell \leq \varepsilon$.
\end{assumption}

Among PSGD methods, those with NSS or NSL learning rates satisfy \cref{assu:gradient_based_opt_independ_ell}, and explicit convergence rates for the best iterate are available (\cref{prop:convergence_rate_NSS_eg,prop:convergence_rate_NSL_eg}).
In contrast, for PSGD with the Polyak learning rate, $\mathrm{update}_t$ depends on the second argument, i.e., the sequence of objective values $\{\ell(\theta^{t^\prime})\}_{t^\prime=1}^t$, and hence it does not satisfy \cref{assu:gradient_based_opt_independ_ell}.

Since $X(s^{(n)})$ is a finite union of polyhedra and $f$ is such that each component function $f_i$ is piecewise affine, it follows that $f\!\left(X(s^{(n)}), s^{(n)}\right)$ is also a finite union of polyhedra (see \cref{prop:PolyhedralMap,prop:image_polyhedral_map_on_polyhedra}).
Let $Y^{(n)}$ denote the set of vertices of the polyhedral complex $f\!\left(X(s^{(n)}), s^{(n)}\right)$ (for the definition of the vertex set, see \cref{defi:vertex}).
Under \cref{assu:WIRL-uniqueness}, by the maximum principle,
\begin{equation}
    a^* (\theta^* ,s^{(n)}) \in Y^{(n)}
    \label{eq:optimal_in_vertex}
\end{equation}
holds.
We define the Lipschitz constant $L (\ell_\mathrm{sub})$ and a positive constant $\gamma (\ell_\mathrm{sub})$ as follows:
\begin{align}
    L (\ell_\mathrm{sub}) & := \sup_{\substack{ a^n \in f (X (s^{(n)}), s^{(n)}), \\ n =1, \ldots ,N }} \left\| \frac{1}{N} \sum_{n=1}^N  ( a^n - a^{(n)} ) \right\|
    \left( \leq L(f) \mathrm{diam} (\mathcal{X}) \right)
    ,
    \label{eq:Lipschitz_SL} \\
    \gamma (\ell_\mathrm{sub}) & := \max_{\theta \in \Theta} \min_{( a^{n} )_n \in \prod_{n=1}^N Y^{(n)} \setminus \{ (a^{(n)} )_n \} } \frac{1}{N} \sum_{n=1}^N \theta^\top \left( a^{(n)} - a^n \right) .
    \label{eq:gamma_SL}
\end{align}
Here, $L(f)$ denotes the Lipschitz constant of $f$, and $\mathrm{diam} (\mathcal{X}) = \sup_{x, x^\prime \in \mathcal{X}} \| x - x^\prime \|$ denotes the diameter of $\mathcal{X}$.
For an intuitive explanation of the constant $\gamma (\ell_\mathrm{sub})$, see \cref{sec:proof_overview}.
For the fact that $\gamma (\ell_\mathrm{sub})$ is positive, see \cref{prop:grad_based_opt_achieve_min_SL_before,prop:gamma_is_positive}.

Our main theorem is as follows.
\begin{theorem}
    \label{theo:grad_based_opt_achieve_min_SL}
    Suppose that \cref{assu:WIRL,assu:gradient_based_opt_independ_ell,assu:gradient_based_opt_Q,assu:WIRL-uniqueness} hold.
    Then, \\
    if $T \geq Q_{L (\ell_\mathrm{sub}), \theta^1} (\gamma (\ell_\mathrm{sub}))$, we have $\min_{t=1 , \ldots , T} \ell_{\mathrm{sub}} (\theta^t ) = \min_{\Theta} \ell_{\mathrm{sub}} = 0$.
\end{theorem}

See \cref{sec:techinical_lemmas} for the proof.
From \cref{lem:Psi_set_is_almost_Phi} and \cref{theo:grad_based_opt_achieve_min_SL}, we obtain the following corollary.

\begin{corollary}
    \label{cor:grad_based_opt_achieve_min_SL_on_ps}
    Suppose that \cref{assu:WIRL,assu:gradient_based_opt_independ_ell,assu:gradient_based_opt_Q} hold.
    Let $\Theta = \Delta^{d-1}$.
    Then, for Lebesgue-almost every $\theta^* \in \Delta^{d-1}$ (with respect to the measure on $\Delta^{d-1}$ induced by Lebesgue measure), the following holds:
    if $T \geq Q_{L (\ell_\mathrm{sub}), \theta^1} (\gamma (\ell_\mathrm{sub}))$, then
    $\min_{t=1 , \ldots , T} \ell_{\mathrm{sub}} (\theta^t ) = \min_{\Theta} \ell_{\mathrm{sub}} = 0$.
\end{corollary}

\begin{proof}
    By \Cref{lem:Psi_set_is_almost_Phi}, \Cref{assu:WIRL-uniqueness} holds for a.e. $\theta^*\in\Delta^{d-1}$.
    Fix such a $\theta^*$.
    Then \Cref{theo:grad_based_opt_achieve_min_SL} is applicable, which yields the claim.
\end{proof}

\begin{remark}
By strengthening the result of \cref{theo:grad_based_opt_achieve_min_SL}, one can, in the case $\Theta = \Delta^{d-1}$, replace the constant $\gamma (\ell_\mathrm{sub})$ by a constant depending only on $\{ X (s^{(n)} ) \}_n$, $f$, and $\Theta$ (see \cref{theo:universal_grad_based_opt_achieve_min_SL}).
\end{remark}

\subsection{Concrete complexity bounds for PSGD}
\label{sec:solve_MILP_DDIOP_grad_eg}

In this section, by applying \Cref{theo:grad_based_opt_achieve_min_SL}, we bound the number of iterations required to solve the DDIOP associated with an MILP when incorporating a concrete gradient-based optimization method into the update rule of \Cref{alg:intention-WIRL-gradual-decay}.

Let $\mathrm{diam}(\Theta)$ denote the diameter of $\Theta$.
In what follows, we use the fact that each variant of PSGD satisfies \Cref{assu:gradient_based_opt_independ_ell} and admits an error bound for the best iterate (for NSS, see \Cref{prop:convergence_rate_NSS}; for NSL, see \Cref{prop:convergence_rate_NSL}).
Then, by \Cref{theo:grad_based_opt_achieve_min_SL}, the following results hold.

\paragraph{PSGD (NSS)}
\begin{corollary}
    \label{cor:finite_solve_NSS}
    Under \Cref{assu:WIRL,assu:WIRL-uniqueness}, \Cref{alg:intention-WIRL-gradual-decay} with PSGD (NSS) solves \Cref{eq:IOP_linear} exactly in finitely many iterations.
\end{corollary}

When the learning rate is SRSS, one can explicitly construct the function $Q_{L,\theta^1}$ (see \Cref{prop:convergence_rate_NSS_eg}); hence, one can bound the number of iterations required to solve the DDIOP associated with an MILP as follows:
\begin{corollary}
    \label{cor:convergence_rate_NSS_eg_mini}
    Under \Cref{assu:WIRL,assu:WIRL-uniqueness}, \Cref{alg:intention-WIRL-gradual-decay} with PSGD (SRSS) solves \Cref{eq:IOP_linear} exactly within at most\\
    $
        \max \left( 1, \left( \frac{ \mathrm{diam}(\Theta)^2 + (1 + \log 2) \beta^2 L(\ell_{\mathrm{sub}})^2}{\beta \gamma (\ell_{\mathrm{sub}})}  \right)^2 - 2 \right)
    $
    iterations.
\end{corollary}

For the construction of $Q_{L,\theta^1}$ for NSS (including SRSS) and the proof of \Cref{cor:convergence_rate_NSS_eg_mini}, see \Cref{sec:Finite-time_solvability_DDIOP_w_PSGD_NSS}.

\paragraph{PSGD (NSL)}
\begin{corollary}
    \label{cor:finite_solve_NSL}
    Under \Cref{assu:WIRL,assu:WIRL-uniqueness}, \Cref{alg:intention-WIRL-gradual-decay} with PSGD (NSL) solves \Cref{eq:IOP_linear} exactly in finitely many iterations.
\end{corollary}

When the learning rate is SRSL, one can explicitly construct the function $Q_{L,\theta^1}$ (see \Cref{prop:convergence_rate_NSL_eg}); hence, one can bound the number of iterations required to solve the DDIOP associated with an MILP as follows:
\begin{corollary}
    \label{cor:convergence_rate_NSL_eg_mini}
    Under \Cref{assu:WIRL,assu:WIRL-uniqueness}, \Cref{alg:intention-WIRL-gradual-decay} with PSGD (SRSL) solves \Cref{eq:IOP_linear} exactly within at most\\
    $
        \max \left( 1, \left( \frac{ L(\ell_{\mathrm{sub}}) (\mathrm{diam}(\Theta)^2 + (1 + \log 2) \beta^2 )}{\beta \gamma (\ell_{\mathrm{sub}}) }  \right)^2 - 2 \right)
    $
    iterations.
\end{corollary}

For the construction of $Q_{L,\theta^1}$ for NSL (including SRSL) and the proof of \Cref{cor:convergence_rate_NSL_eg_mini}, see \Cref{sec:Finite-time_solvability_DDIOP_w_PSGD_NSL}.

\subsection{Extension: finite-time attainment of the PLF minimum}
\label{sec:finite_time_PLF_min}

In this section, we show that, by using PSGD, one can drive the PLF to $0$ within finitely many iterations.
For a nonempty subset $\Theta^\prime \subset \Theta$ and a point $\theta \in \Theta$, define
$d (\theta , \Theta^\prime) := \inf_{\theta^\prime \in \Theta^\prime } \| \theta - \theta^\prime \|$.

As in the case of the DDIOP for MILPs, by employing PSGD (NSS, NSL), one can attain the minimum value $0$ of the PLF in finitely many iterations (\cref{theo:achieve_min_PLF_PSGD_NSS,theo:achieve_min_PLF_PSGD_NSL}).
In particular, for PSGD (SRSS, SRSL), we can bound the distance between the iterate sequence $\{ \theta^t \}_t$ and the set of PLF minimizers $\argmin_{\Theta} \ell_{\mathrm{plf}}$, as well as the number of iterations required to attain the PLF minimum (\cref{theo:achieve_min_PLF_PSGD_SRSS_readable,theo:achieve_min_PLF_PSGD_SRSL_readable}).
These iteration upper bounds also depend on the constant $\gamma(\ell_{\mathrm{sub}})$.

\begin{theorem}
    \label{theo:achieve_min_PLF_PSGD_SRSS_readable} 
    Suppose that \cref{assu:WIRL,assu:WIRL-uniqueness} hold.
    Let $\{ \theta^t \}_t$ be the sequence obtained by implementing the update rule $\{ \mathrm{update}_t \}_t$ in \cref{alg:intention-WIRL-gradual-decay} via PSGD (SRSS).
    Then:
    (A) For $T \in \mathbb{Z}_{\geq 1}$,
    \begin{align*}
        & d (\theta^{T+1} , \argmin_{\Theta} \ell_{\mathrm{plf}} )^2
        \notag \\
        & \leq 
        \mathrm{diam} (\Theta)^2
        - \beta (2-\sqrt{2}) \left( \sqrt{T} -1 \right)
        \left(
        2 \gamma (\ell_{\mathrm{sub}})
        - \beta L (\ell_{\mathrm{sub}})^2 \frac{2 ( 1 + \log 2)}{\sqrt{T+2}} 
        \right)
    \end{align*}
    or $\ell_{\mathrm{plf} } (\theta^{T}  ) = 0$.
    (B) The condition $\ell_{\mathrm{plf} } (\theta^{T}  ) = 0$ is attained within at most\\
    $
        \max \left( \frac{ 4 (1 + \log 2)^2 L(\ell_{\mathrm{sub}})^4 \beta^2}{\gamma (\ell_{\mathrm{sub}})^2} -2, \left( 1 + \frac{\mathrm{diam} (\Theta)^2}{(2-\sqrt{2}) \beta \gamma (\ell_{\mathrm{sub}}) } \right)^2 \right)
    $
    iterations.
\end{theorem}

\begin{theorem}
    \label{theo:achieve_min_PLF_PSGD_SRSL_readable} 
    Suppose that \cref{assu:WIRL,assu:WIRL-uniqueness} hold.
    Let $\{ \theta^t \}_t$ be the sequence obtained by implementing the update rule $\{ \mathrm{update}_t \}_t$ in \cref{alg:intention-WIRL-gradual-decay} via PSGD (SRSL).
    Then:
    (A) For $T \in \mathbb{Z}_{\geq 1}$,
    \[
       d (\theta^{T+1 } , \argmin_{\Theta} \ell_{\mathrm{plf}})^2
        \leq 
        \mathrm{diam} (\Theta)^2
        - \beta (2-\sqrt{2}) \left( \sqrt{T} -1 \right)
        \left(
        2 \frac{\gamma (\ell_{\mathrm{sub}})}{L(\ell_{\mathrm{sub}})}
        - \beta \frac{2 ( 1 + \log 2)}{\sqrt{T+2}} 
        \right)
    \]
    or $\ell_{\mathrm{plf} } (\theta^{T}  ) = 0$.
    (B) The condition $\ell_{\mathrm{plf} } (\theta^{T}  ) = 0$ holds within at most\\
    $
        \max \left( \frac{ 4 (1 + \log 2)^2 L(\ell_{\mathrm{sub}})^2 \beta^2}{\gamma (\ell_{\mathrm{sub}})^2} -2, \left( 1 + \frac{\mathrm{diam} (\Theta)^2 L(\ell_{\mathrm{sub}})}{(2-\sqrt{2}) \beta \gamma (\ell_{\mathrm{sub}}) } \right)^2 \right)
    $
    iterations.
\end{theorem}
For the proofs of \cref{theo:achieve_min_PLF_PSGD_SRSS_readable,theo:achieve_min_PLF_PSGD_SRSL_readable}, see \cref{sec:finite_time_PLF_min_appx}.

\section{Comparison tables}

\subsection{DDIOP for MILPs}

A performance comparison of methods for solving the DDIOP for MILPs (\cref{eq:IOP_linear}) is provided in \cref{tab:pro_con_SL_detail}.
\begin{table}[ht]
\vspace{-\intextsep} 
    \caption{Performance comparison of methods for solving the DDIOP for MILPs (\cref{eq:IOP_linear}).
    The integer $T$ denotes the number of iterations or the number of optimization calls.
    The integer $t$ denotes the iteration index or the optimization-call index.}
    \label{tab:pro_con_SL_detail}
    \centering
    \begin{tabular}{{p{3.5cm}|p{8.5cm}}}
        \toprule
        Method & Suboptimality loss \\
        \midrule
        UPA & $O\!\left( \mathrm{diam} (\Theta)\, L(\ell_{\mathrm{sub}})\, T^{-\nicefrac{1}{(d-1)} } \right)$ \\
        & \citep[cf.][Theorem 2.1]{flatto1977random}\\
        \hline
        RPA & $O_{\mathbb{P}}\!\left( \mathrm{diam} (\Theta)\, L(\ell_{\mathrm{sub}})\, \left( \log T / T \right)^{\nicefrac{1}{(d-1)} }\right)$
        \\
        & \citep[cf.][Corollary 2.3]{reznikov2015covering}\\
        \hline
        MWU & $O \left( L(\ell_{\mathrm{sub}}) \log d  / \sqrt{T} \right)$, provided that $\Theta = \Delta^{d-1}$ \\
        \citep{arora2012multiplicative} & \citep[Theorem 3.5]{Barmann-2018-online} \\
        \hline
        \citet{besbes2021online,besbes2025contextual} & $O \left( d^4 \log T /T \right)$, provided that $\Theta$ is the unit sphere and $L(\ell_{\mathrm{sub}})\leq 1$\\
        \hline
        \citet{gollapudi2021contextual} & $O \left( d \log T /T \right)$, $O \left( d^{2(d+1)} /T \right)$ provided that $\Theta$ is the unit ball and $L(\ell_{\mathrm{sub}})\leq 1$ \\
        \hline
        ONS & $O \left(\mathrm{diam} (\Theta)\, L(\ell_{\mathrm{sub}})\, d\log (T/d)/T \right)$\\
        \citep{hazan2007logarithmic} &  \citep[Theorem 3.1]{sakaue2025online} \\
        \hline
        MetaGrad & $O \left( \mathrm{diam} (\Theta)\, L(\ell_{\mathrm{sub}})\, d\log (T/d)/T \right)$ \\
        (\citealp{van2016metagrad}; & \citep[Theorem 4.1]{sakaue2025online} \\
        \citealp{van2021metagrad}) & \\
        \hline
        PSGD & $O \left( \mathrm{diam} (\Theta)\, L(\ell_{\mathrm{sub}}) / \sqrt{T} \right)$ \\
        (with step size & \citep[Theorem 3.11]{Barmann-2018-online} \\
        $\mathrm{diam} (\Theta) L(\ell_{\mathrm{sub}})^{-1} t^{-1/2}$) &  \\
        \hline
        PSGD & $0$ if $T \geq \max \left( 1, \left( \frac{ \mathrm{diam}(\Theta)^2 + (1 + \log 2) \beta^2 L(\ell_{\mathrm{sub}})^2}{\beta \gamma (\ell_{\mathrm{sub}})}  \right)^2 -2 \right)$ \\
        (with step size $\beta t^{-1/2}$) & (Ours, \cref{cor:convergence_rate_NSS_eg_mini,cor:convergence_rate_NSS_eg}) \\
        \hline
        PSGD & $0$ if $T \geq \max \left( 1, \left( \frac{ L(\ell_{\mathrm{sub}})\, (\mathrm{diam}(\Theta)^2 + (1 + \log 2) \beta^2 )}{\beta \gamma (\ell_{\mathrm{sub}}) }  \right)^2 -2 \right)$ \\
        (with step length $\beta t^{-1/2}$) & (Ours, \cref{cor:convergence_rate_NSL_eg_mini,cor:convergence_rate_NSL_eg}) \\
        \bottomrule
    \end{tabular}
    \vspace{-\intextsep} 
\end{table}

\subsection{PLF minimization}

A performance comparison of methods for solving the minimum value of the PLF for MILPs is provided in \cref{tab:pro_con_PLF}.
\begin{table}[ht]
    \caption{Performance comparison of methods for minimizing the PLF for MILPs.
    The notation is the same as in \cref{tab:pro_con_SL_detail}.}
    \label{tab:pro_con_PLF}
    \centering
    \begin{tabular}{p{1.5cm}|p{10cm}}
        \toprule
        Method & Distance between the set where the PLF equals $0$ and the learned weight $\theta^T$ (when the PLF is nonzero) \\
        \midrule
        UPA & $O\!\left( \mathrm{diam} (\Theta)\, T^{-\nicefrac{1}{(d-1)} } \right)$ \\
        \hline
        RPA & $O_{\mathbb{P}}\!\left( \mathrm{diam} (\Theta)\, \left( \log T / T \right)^{\nicefrac{1}{(d-1)} }\right)$ \\
        \hline
        PSGD & $
        \mathrm{diam} (\Theta)^2
        - \beta (2-\sqrt{2}) \left( \sqrt{T} -1 \right)
        \left(
        2 \gamma (\ell_{\mathrm{sub}})
        - \beta L (\ell_{\mathrm{sub}})^2 \frac{2 ( 1 + \log 2)}{\sqrt{T+2}}
        \right)$, \\
        (with step size $\beta t^{-1/2}$) & \cref{theo:achieve_min_PLF_PSGD_SRSS_readable} \\
        \hline
        PSGD & $\mathrm{diam} (\Theta)^2
        - \beta (2-\sqrt{2}) \left( \sqrt{T} -1 \right)
        \left(
        2 \frac{\gamma (\ell_{\mathrm{sub}})}{L(\ell_{\mathrm{sub}})}
        - \beta \frac{2 ( 1 + \log 2)}{\sqrt{T+2}}
        \right)$, \\
        (with step length $\beta t^{-1/2}$) & \cref{theo:achieve_min_PLF_PSGD_SRSL_readable} \\
        \bottomrule
    \end{tabular}
\end{table}

\begin{wrapfigure}{r}{7.0cm}
    \vspace{-2\intextsep}
    \begin{center}
    \begin{tikzpicture}[scale=0.95]

        \coordinate (A) at (-3,2 / 1.1);
        \coordinate (B) at (-1,0);
        \coordinate (C) at (1,0);
        \coordinate (D) at (3,2/1.05);
        \coordinate (E) at (0,-1/1.05);
        \coordinate (O) at (0,0);

        \draw (D) node[font=\scriptsize, above]{$\ell(\theta)=\max(|\theta|-1,0)$};
        \draw (E) node[font=\scriptsize, below left]{$\widetilde{\ell}(\theta)=|\theta|-1$};
        \draw (C) node[font=\scriptsize, right]{\textcolor{strblue}{$\min \ell$}};

        \draw[thick] (A) -- (B);
        \draw[thick,strblue] (B) -- (C);
        \draw[thick] (C) -- (D);

        \draw[dashed] (B) -- (E) -- (C);

        \draw (O) to[out=-60,in=60]
          node[midway,left,xshift=0pt, yshift=-2pt,fill=white,inner sep=1pt, font=\scriptsize]
          {$\gamma=\min\ell-\min\widetilde{\ell}$}
        (E);

        \coordinate (ABmid) at ($(A)!0.55!(B)$);
        \node[font=\scriptsize, fill=white, inner sep=1pt, above=2pt] at (ABmid)
          {$\ell=\widetilde{\ell}$};

        \coordinate (CDmid) at ($(C)!0.55!(D)$);
        \node[font=\scriptsize, fill=white, inner sep=1pt, above=2pt] at (CDmid)
          {$\ell=\widetilde{\ell}$};

        \coordinate (P1) at (-2.5,1.5/1.05);
        \coordinate (P2) at (1.9,0.9/1.05);
        \coordinate (P3) at (-1.3,0.3/1.05);
        \coordinate (P4) at (0.7,-0.3/1.05);
        \coordinate (Q4) at (0.7,0);

        \fill[strred] (P1) circle (2pt);
        \draw[strred] (P1) node[below left]{$\theta_1$};
        \fill[strred] (P2) circle (2pt);
        \draw[strred] (P2) node[below right]{$\theta_2$};
        \fill[strred] (P3) circle (2pt);
        \draw[strred] (P3) node[below left]{$\theta_3$};
        \fill[strred] (Q4) circle (2pt);
        \draw[strred] (Q4) node[above]{$\theta_4$};
        \fill[strred] (P4) circle (2pt);
        \draw[strred] (P4) node[below right]{};

        \draw[->,strred] (P1) -- (P2);
        \draw[->,strred] (P2) -- (P3);
        \draw[->,strred] (P3) -- (P4);
        \draw[->,strred] (P4) -- (Q4);

    \end{tikzpicture}
    \end{center}
    \caption{Schematic illustration of reaching the minimum value $\min_{\Theta} \ell = 0$ in finitely many iterations when applying the (sub)gradient method to the convex piecewise-linear function $\ell(\theta)=\max(|\theta|-1,0)$.}
    \label{fig:idea_achieve_minimum_w_GD_for_PL}
    \vspace{-1\intextsep} 
\end{wrapfigure}

\section{Proof Overview / Intuition}
\label{sec:proof_overview}

We provide an intuition for the claim that
``for convex piecewise-linear functions whose minimizer set has a relative interior point,
first-order methods such as projected subgradient descent reach the minimum value in finitely many iterations.''

In this paper, we take $\Theta$ to be a convex compact set and consider first-order methods
(methods whose iterates are updated solely based on the past iterates and (sub)gradients; e.g., PSGD).
In this setting, for a convex piecewise-linear function $\ell:\Theta\to\mathbb{R}$,
if the minimizer set $\argmin_{\Theta} \ell$ has a relative interior point,
then one can construct another convex piecewise-linear function $\widetilde{\ell}$
by decreasing the function values only in the relative interior of $\argmin_{\Theta}\ell$.
Moreover, since one may construct $\widetilde{\ell}$ so that $\ell=\widetilde{\ell}$ holds outside $\argmin_{\Theta}\ell$,
the iterates produced by a first-order method applied to $\ell$ and those produced when applied to $\widetilde{\ell}$
coincide (and so do the subgradients) until the iterates reach $\argmin_{\Theta}\ell$.

We illustrate this behavior via a concrete example.
Consider $\ell(\theta)=\max(|\theta|-1,0)$.
Then $\argmin \ell=[-1,1]$, and the minimizer set has a relative interior point.
Let $\widetilde{\ell}(\theta)=|\theta|-1$.
Since $\ell(\theta)=\widetilde{\ell}(\theta)$ holds for $|\theta|>1$,
the iterates of any first-order method are identical for the two objectives until they enter $[-1,1]$.
Furthermore, since $\min_{\Theta} \ell=0$ and $\min_{\Theta}\widetilde{\ell}<0$,
we can define
$
\gamma:=\min_{\Theta} \ell-\min_{\Theta}\widetilde{\ell}>0
$
(the gap is depicted as $\gamma$ in \cref{fig:idea_achieve_minimum_w_GD_for_PL}).

On the other hand, for first-order methods such as projected subgradient descent,
standard convergence guarantees imply that ``$\widetilde{\ell}(\theta^t)$ converges to $\min_{\Theta}\widetilde{\ell}$.'' 
Hence, there exists a finite iteration index $T$ such that
$
\widetilde{\ell}(\theta^T)\le \min_{\Theta}\widetilde{\ell}+\gamma
= \min_{\Theta}\ell
$.
Since the iterates coincide until the minimum is attained, the same $T$ yields $\ell(\theta^T)=\min_{\Theta}\ell$,
i.e., the minimum value of $\ell$ is attained in finitely many iterations.
Finally, the main object of interest in this paper, the suboptimality loss $\ell_{\mathrm{sub}}$,
satisfies the following proposition:
\begin{proposition}
    \label{prop:argmin_SL_has_inner_pt}
    Assume \cref{assu:WIRL,assu:WIRL-uniqueness}.
    Then: (A) it is piecewise-linear; (B) the minimum value of the suboptimality loss is $0$; and (C) the set of minimizers has a relative interior point.
\end{proposition}
See \cref{sec:proof_prop:argmin_SL_has_inner_pt} for the proof.
By \cref{prop:SL_is_Lipschitz_convex,prop:argmin_SL_has_inner_pt} and \cref{lem:Psi_set_is_almost_Phi},
the suboptimality loss $\ell_{\mathrm{sub}}$ is convex, Lipschitz, and piecewise-linear; moreover,
when $\Theta = \Delta^{d-1}$, its minimizer set has a relative interior point (for Lebesgue-almost every true weight $\theta^*$).
Therefore, by an argument analogous to the above intuition,
one expects that first-order methods including PSGD attain $\ell_{\mathrm{sub}}=0$ in finitely many iterations
(i.e., solve the DDIOP for MILPs exactly), as formalized in \cref{theo:grad_based_opt_achieve_min_SL}.

\section{Technical Lemmas}
\label{sec:techinical_lemmas}

\subsection{A minimum-attainment theorem via gradient-based optimization for convex piecewise-linear functions}
\label{sec:grad_based_opt_for_PWLF}

In this section, we show that, for a convex piecewise-linear function whose set of minimizers has a relative interior point, one can attain the minimum value in finitely many iterations by employing a gradient-based optimization method that includes projected subgradient descent (\cref{theo:PL_convex_achieve_min}).

\begin{assumption}
    \label{assu:piecewise_linear_covex}
    Let $V \subset \mathbb{R}^d$ be an affine subspace.
    Let the \textbf{direction space} associated with the affine space $V$ (the linear subspace obtained by translating $V$ so that it passes through the origin) be defined by
    \[
        V_0 := \{\, x - y \mid x, y \in V \,\} \subset \mathbb{R}^d .
    \]
    Then $V_0$ is a linear subspace of $\mathbb{R}^d$, and for any $\theta_0 \in V$ one can write $V = \theta_0 + V_0$. Let $P_{V_0} \colon \mathbb{R}^d \to V_0$ denote the orthogonal projection onto $V_0$.
    Let $\Theta \subset V$ be a convex body, i.e., a convex set with nonempty interior relative to $V$.
    Let $I$ be a finite set.
    For each $i \in I$, let $a^i \in \mathbb{R}^d$ and $b_i \in \mathbb{R}$.
    Assume that $\ell (\theta) = \max_{i \in I} \left( a^{i \top} \theta + b_i \right)$ is not constant on $\Theta$ and that $\argmin_{\Theta} \ell$ has a relative interior point in $V$.
\end{assumption}

\cref{assu:piecewise_linear_covex} is equivalent to the condition that $\ell$ is convex and piecewise-linear and that its minimizer set has a relative interior point (cf.\ \cref{sec:PWLF}).
Under \cref{assu:piecewise_linear_covex}, the function $\ell$ is $\max_{i \in I } \| a^i \|$-Lipschitz continuous.
Define $\widetilde{\ell} \colon \Theta \to \mathbb{R}$ and $\gamma (\ell) \in \mathbb{R}$ by
\begin{equation}
    \label{eq:def_gamma_ell}
    \widetilde{\ell} (\theta) := \max_{i \in I \,:\, P_{V_0} a^i \not = 0 } \left( a^{i \top} \theta + b_i \right) ,
    \quad
    \gamma (\ell) := \min_{\Theta} \ell - \min_{\Theta} \widetilde{\ell}
\end{equation}
Here the excluded indices are those $i$ with $P_{V_0} a^i = 0$, i.e., those $i$ for which the affine function $a^{i\top}\theta + b_i$ is constant on $V$. Indeed, writing $\theta = \theta_0 + u$ (with $u \in V_0$), we have
\[
    a^{i\top}\theta + b_i = (a^{i\top}\theta_0 + b_i) + (P_{V_0}a^i)^\top u
\]
(since $u \in V_0$ implies $a^{i\top}u = (P_{V_0}a^i)^\top u$), so the variation over $V$ is determined solely by $P_{V_0}a^i$. Since in \cref{assu:piecewise_linear_covex} we assumed that $\ell$ is not constant on $\Theta$, there exists at least one $i$ with $P_{V_0}a^i \neq 0$, and hence $\widetilde\ell$ is well-defined. Note that when $V = \mathbb{R}^d$ one has $P_{V_0}a^i \neq 0 \iff a^i \neq 0$.
\begin{remark}
    For a convex piecewise-linear function $\ell$, the definitions of $\widetilde{\ell}$ and $\gamma (\ell)$ are not necessarily well-defined in an intrinsic sense.
    Note that $\widetilde{\ell}$ is defined here only because $\ell$ is given in the explicit form $\ell (\theta) = \max_{i \in I} \left( a^{i \top} \theta + b_i \right)$.
\end{remark}

Then, the following theorem holds.
\begin{theorem}
    \label{theo:PL_convex_achieve_min}
    Suppose that \cref{assu:piecewise_linear_covex,assu:gradient_based_opt_independ_ell,assu:gradient_based_opt_Q} hold.
    Let the Lipschitz constant be $L(\ell) = \max_{i \in I } \| a^i \|$.
    If $T \geq Q_{L(\ell), \theta^1} (\gamma (\ell))$, then $\min_{t=1 , \ldots , T} \ell (\theta^t ) = \min_{\Theta} \ell$.
\end{theorem}
See \cref{sec:grad_based_opt_for_PWLF_appx} for the proof.

\subsection{Proof Sketch of Theorem \ref{theo:grad_based_opt_achieve_min_SL}}

For the suboptimality loss, by the maximum principle, we can write
\begin{equation}
    \ell_\mathrm{sub} (\theta)
    =  \frac{1}{N} \sum_{n=1}^N \max_{a^n \in Y^{(n)}} \theta^\top ( a^n - a^{(n)} ) =
    \max_{( a^n )_n \in \prod_{n=1}^N Y^{(n)}} \theta^\top \left( \frac{1}{N} \sum_{n=1}^N ( a^n - a^{(n)} ) \right)
    \label{eq:SL_is_PWLF}
\end{equation}
The constant $\gamma (\ell_\mathrm{sub})$ is given as follows.
\begin{proposition}
    \label{prop:grad_based_opt_achieve_min_SL_before}
    Under \cref{assu:WIRL,assu:WIRL-uniqueness}, the constant $\gamma (\ell_\mathrm{sub})$ defined by \cref{eq:gamma_SL} is positive, and the thickness $\min_\Theta\ell_\mathrm{sub} - \min_\Theta\widetilde{\ell_\mathrm{sub}}$ of $\ell_\mathrm{sub}$ (in the sense of \cref{eq:def_gamma_ell}) is at least this $\gamma (\ell_\mathrm{sub})$.
\end{proposition}
See \cref{sec:proof_prop:grad_based_opt_achieve_min_SL_before} for the proof.

{
\renewcommand{\proofname}{Proof of \cref{theo:grad_based_opt_achieve_min_SL}}%
\begin{proof}
    If $\ell_\mathrm{sub}$ is constant on $\Theta$, then by \cref{prop:argmin_SL_has_inner_pt}(B) we have $\ell_\mathrm{sub} \equiv \min_\Theta\ell_\mathrm{sub} = 0$, so the conclusion holds for every $T \geq 1$. In what follows, assume that $\ell_\mathrm{sub}$ is not constant on $\Theta$.
    By \cref{assu:WIRL,assu:WIRL-uniqueness}, \cref{eq:SL_is_PWLF}, and
    \cref{prop:argmin_SL_has_inner_pt,prop:SL_is_Lipschitz_convex},
    the suboptimality loss satisfies \cref{assu:piecewise_linear_covex}.
    By \cref{assu:gradient_based_opt_independ_ell,assu:gradient_based_opt_Q} and \cref{theo:PL_convex_achieve_min}, if $T \geq Q_{L(\ell_\mathrm{sub}),\theta^1}(\min_\Theta\ell_\mathrm{sub} - \min_\Theta\widetilde{\ell_\mathrm{sub}})$, then $\min_{t=1,\ldots,T}\ell_\mathrm{sub}(\theta^t) = \min_\Theta\ell_\mathrm{sub} = 0$. By \cref{prop:grad_based_opt_achieve_min_SL_before}, we have $\min_\Theta\ell_\mathrm{sub} - \min_\Theta\widetilde{\ell_\mathrm{sub}} \geq \gamma(\ell_\mathrm{sub}) > 0$, and since $Q_{L(\ell_\mathrm{sub}),\theta^1}$ is nonincreasing, $Q_{L(\ell_\mathrm{sub}),\theta^1}(\gamma(\ell_\mathrm{sub})) \geq Q_{L(\ell_\mathrm{sub}),\theta^1}(\min_\Theta\ell_\mathrm{sub} - \min_\Theta\widetilde{\ell_\mathrm{sub}})$. Hence, if $T \geq Q_{L(\ell_\mathrm{sub}),\theta^1}(\gamma(\ell_\mathrm{sub}))$, the conclusion of the theorem follows.
\end{proof}
}

\section{Proof of Lemma~\ref{lem:Psi_set_is_almost_Phi}}
\label{sec:proof_lem:Psi_set_is_almost_Phi}

We introduce the following notation:
\begin{align*}
    \Psi
    &:= \bigcap_{n=1}^N
    \bigcup_{\xi^{(n)} \in Y^{(n)}} \Theta(\xi^{(n)},Y^{(n)}) \, ,
    \\
    \Theta(\xi^{(n)},Y^{(n)})
    &:=
    \left\{\theta\in\Theta \,\middle|\, \forall b\in Y^{(n)}\setminus\{\xi^{(n)}\},\ \theta^{\top}\xi^{(n)}>\theta^{\top}b \right\} \, .
\end{align*}
For any $n=1,\ldots,N$ and any $\xi^{(n)}\in Y^{(n)}$, the set $\Theta(\xi^{(n)},Y^{(n)})$ is open in the relative topology on $\Theta$.
That is, for any $\theta^*\in\Psi$, there exist $\varepsilon(\theta^*)>0$ and points $\xi^{(n)}\in Y^{(n)}$ (for each $n$) such that
\begin{equation}
  \left\{\theta\in\Delta^{d-1} \,\middle|\, \|\theta-\theta^*\|<\varepsilon(\theta^*)\right\}
  \subset \bigcap_{n=1}^N \Theta(\xi^{(n)},Y^{(n)})
  \label{eq:PLW_smaller_PLF_smaller}
\end{equation}
holds.

Next, for $D\subset Y^{(n)}$, define
\[
    \Theta(D,Y^{(n)})
    :=
    \left\{
        \theta\in\Delta^{d-1}
        \,\middle|\,
        \forall a_1,a_2\in D,\ \forall a_3\in Y^{(n)}\setminus D,\ 
        \theta^{\top}a_1=\theta^{\top}a_2>\theta^{\top}a_3
    \right\}.
\]
Then,
\[
    \Delta^{d-1} = \bigcup_{D\subset Y^{(n)}} \Theta(D,Y^{(n)}).
\]
We further define
\begin{align*}
    \mathcal{M}(Y^{(n)})
    &:= \left\{D\subset Y^{(n)} \,\middle|\, \dim \Theta(D,Y^{(n)})=d-1\right\} \, .
\end{align*}

\begin{lemma}
    \label{lem:uniqueness_DinM}
    For any $D\in\mathcal{M}(Y^{(n)})$, the set $D$ is a singleton.
\end{lemma}

\begin{proof}
    Assume that $D$ contains two distinct points, and take $a_1,a_2\in D$.
    If there exists $c\neq 0$ such that $a_1-a_2=(c,\ldots,c)$, then, since $\theta\in\Delta^{d-1}$,
    \[
        0=\theta^{\top}(a_1-a_2)=\theta^{\top}(c,\ldots,c)=c
    \]
    holds, which contradicts $c\neq 0$.
    Hence, we may assume that $a_1-a_2$ and $(1,\ldots,1)$ are linearly independent.
    However,
    \[
        \Theta(D,Y^{(n)})
        \subset
        \left\{\theta\in\mathbb{R}^d \,\middle|\,
        \theta^{\top}(a_1-a_2)=0,\ 
        \theta^{\top}(1,\ldots,1)=1
        \right\},
    \]
    implying that $\dim \Theta(D,Y^{(n)})\leq d-2$.
    This contradicts $D\in\mathcal{M}(Y^{(n)})$.
\end{proof}

\begin{lemma}
    \label{lem:disjoint_vertex_set}
    Let $D_1,D_2\subset Y^{(n)}$.
    If $D_1\neq D_2$, then $\Theta(D_1,Y^{(n)})\cap\Theta(D_2,Y^{(n)})=\emptyset$.
\end{lemma}

\begin{proof}
    Suppose that $D_1\neq D_2$ and that there exists $\theta\in\Theta(D_1,Y^{(n)})\cap\Theta(D_2,Y^{(n)})$.
    By symmetry, we may assume $D_1\setminus D_2\neq\emptyset$.
    Take $a\in D_1\setminus D_2$ and $b\in D_2$.
    Since $\theta\in\Theta(D_2,Y^{(n)})$, we have $\theta^{\top}b>\theta^{\top}a$.
    On the other hand, since $\theta\in\Theta(D_1,Y^{(n)})$, we have $\theta^{\top}b\leq \theta^{\top}a$.
    This is a contradiction.
\end{proof}

{
\renewcommand{\proofname}{Proof of \cref{lem:Psi_set_is_almost_Phi}}%
\begin{proof}
    By \cref{lem:disjoint_vertex_set}, we have
    \[
        \Delta^{d-1} \setminus \bigcup_{D\in\mathcal{M}(Y^{(n)})} \Theta(D,Y^{(n)})
        =
        \bigcup_{\substack{D\subset Y^{(n)} \\ \dim \Theta(D,Y^{(n)}) < d-1}} \Theta(D,Y^{(n)}).
    \]
    If $\dim \Theta(D,Y^{(n)})<d-1$, then the convex set $\Theta(D,Y^{(n)})$ has (Lebesgue) measure zero in $\mathbb{R}^d$ (equivalently, zero $(d-1)$-dimensional measure relative to the simplex).
    Since the power set of $Y^{(n)}$ is finite, it follows that
    \[
        \Delta^{d-1} \setminus \bigcup_{D\in\mathcal{M}(Y^{(n)})} \Theta(D,Y^{(n)})
    \]
    has measure zero.
    Moreover,
    \[
        \Delta^{d-1} \setminus \Psi
        =
        \bigcup_{n=1,\ldots,N}
        \left(
            \Delta^{d-1} \setminus \bigcup_{D\in\mathcal{M}(Y^{(n)})} \Theta(D,Y^{(n)})
        \right),
    \]
    and hence $\Delta^{d-1}\setminus\Psi$ has measure zero.
    This proves the lemma.
\end{proof}
}

\section{Finite-time exact solvability of the DDIOP for MILPs}

\subsection{PSGD (NSS)}
\label{sec:Finite-time_solvability_DDIOP_w_PSGD_NSS}

For PSGD (NSS), the following convergence rate is well known.
\begin{proposition}[{\citealp[e.g.][Theorem 8.25]{Beck-2017-First}}]
    \label{prop:convergence_rate_NSS}
    Let $\Theta$ be a bounded, closed, and convex set, and let $L>0$, $\theta^1\in \Theta$, and $T\in\mathbb{Z}_{\geq 1}$.
    Then, \emph{(A)} for any $L$-Lipschitz convex function $\ell\colon\Theta\to\mathbb{R}$, the sequence $\{\theta^t\}_{t=1}^T$ generated by PSGD (NSS) satisfies
    \[
        \min_{t=1,\ldots,T}\ell(\theta^t)-\min_{\Theta}\ell
        \leq
        \frac{1}{2}\frac{\mathrm{diam}(\Theta)^2}{\sum_{t=1}^T \beta_t}
        +
        \frac{L^2}{2}\frac{\sum_{t=1}^T \beta_t^2}{\sum_{t=1}^T \beta_t}
        \, .
    \]
    In particular, \emph{(B)} there exists a function $Q_{L,\theta^1}\colon \mathbb{R}_{>0}\to\mathbb{R}$ satisfying \cref{assu:gradient_based_opt_Q}.
\end{proposition}

\begin{proposition}
    \label{prop:convergence_rate_NSS_eg}
    Assume that \cref{assu:WIRL,assu:WIRL-uniqueness} hold.
    Let $0 < q <1$ and $\beta >0$, and suppose that the step size is given by $\beta_t = \beta t^{-q}$.
    Define
    \begin{equation}
        Q_{L, \theta^1} (\varepsilon)
        =
        \begin{cases}
            \left( \frac{\mathrm{diam}(\Theta)^2
            +
            \frac{L^2 \beta^2}{1 - 2 q}}{2\beta (1-q)\varepsilon}  \right)^{1/q} , & 0< q <1/2, \\
            \left( \frac{\mathrm{diam}(\Theta)^2
        +
        \frac{2 q L^2 \beta^2}{2 q - 1}}{2\beta (1-q)\varepsilon}  \right)^{1/(1-q)} , & 1/2 < q <1 ,\\
        \max \left( 1, \left( \frac{ \mathrm{diam}(\Theta)^2 + (1 + \log 2) \beta^2 L^2}{\beta \varepsilon } \right)^2 -2 \right) , & q =1/2 \, .
        \end{cases}
        \label{eq:QLtheta1_NSS_eg}
    \end{equation}
    Then $Q_{L,\theta^1}$ satisfies \cref{assu:gradient_based_opt_Q}.
\end{proposition}

By \cref{theo:grad_based_opt_achieve_min_SL} and \cref{prop:convergence_rate_NSS_eg}, we obtain the following corollary.
\begin{corollary}
    \label{cor:convergence_rate_NSS_eg}
    Assume that \cref{assu:WIRL,assu:WIRL-uniqueness} hold.
    Then, if we implement \cref{alg:intention-WIRL-gradual-decay} using PSGD in the present setting, the DDIOP \cref{eq:IOP_linear} is solved exactly within at most
    \begin{equation*}
        \begin{cases}
            \left( \frac{\mathrm{diam}(\Theta)^2
            +
            \frac{L(\ell_{\mathrm{sub}})^2 \beta^2}{1 - 2 q}}{2\beta (1-q) \gamma (\ell_{\mathrm{sub}})}  \right)^{1/q} , & 0< q <1/2, \\
            \left( \frac{\mathrm{diam}(\Theta)^2
        +
        \frac{2 q L(\ell_{\mathrm{sub}})^2 \beta^2}{2 q - 1}}{2\beta (1-q) \gamma (\ell_{\mathrm{sub}})}  \right)^{1/(1-q)} , & 1/2 < q <1 ,\\
        \max \left( 1, \left( \frac{ \mathrm{diam}(\Theta)^2 + (1 + \log 2) \beta^2 L(\ell_{\mathrm{sub}})^2}{\beta \gamma (\ell_{\mathrm{sub}})}  \right)^2 -2 \right) , & q =1/2
        \end{cases}
    \end{equation*}
    iterations.
\end{corollary}

We next prepare the proof of \cref{prop:convergence_rate_NSS_eg}.

\begin{proposition}[{\citealp[Lemmas 8.11 and 8.24]{Beck-2017-First}}]
    \label{prop:update_PSGD}
    Let $\Theta$ be a bounded, closed, and convex set, and let $\theta^1\in\Theta$ and $T\in\mathbb{Z}_{\geq 1}$.
    Let $\{\alpha_t\}_{t=1}^T$ be a step-size sequence.
    Let $\ell\colon\Theta\to\mathbb{R}$ be a Lipschitz continuous convex function.
    Take $\theta^*\in\argmin_{\Theta}\ell$.
    Let $\{\theta^t\}_{t=1}^{T+1}$ be the sequence produced by PSGD.
    Then, for any $t=1,\ldots,T$,
    \begin{align}
        \| \theta^{t+1} - \theta^* \|^2
        & \leq
        \| \theta^{t} - \theta^* \|^2
        - 2 \alpha_t \left( \ell(\theta^t) - \min_{\Theta}\ell \right)
        + \alpha_t^2 \| \nabla \ell(\theta^t) \|^2 ,
        \label{eq:update_PSGD_1}
        \\
        \sum_{t=1}^T \alpha_t \left( \ell(\theta^t) - \min_{\Theta}\ell \right)
        & \leq \frac{1}{2} \| \theta^1 - \theta^* \|^2
        + \frac{1}{2} \sum_{t=1}^T \alpha_t^2 \| \nabla \ell(\theta^t) \|^2
        \, .
        \label{eq:update_PSGD_2}
    \end{align}
\end{proposition}

\begin{proposition}
    \label{prop:sum_SQ_SQS}
    Let $\delta>0$.
    Then, for any $T\in\mathbb{Z}_{\geq 1}$,
    \begin{equation*}
        \frac{\delta + \sum_{t=\lceil T/2 \rceil}^T t^{-1}}{\sum_{t=\lceil T/2 \rceil}^T t^{-1/2}}
        \leq
        \frac{2(\delta + 1 + \log 2)}{\sqrt{T+2}}
        \, .
    \end{equation*}
\end{proposition}

\begin{proof}
    For $T=1$, the inequality holds.

    Let $T\in\mathbb{Z}_{\geq 2}$ and set $k:=\lceil T/2\rceil$.
    Then,
    \begin{align*}
        \sum_{t=k}^T t^{-1}
        &\leq 1 + \int_{k}^T z^{-1}\,dz
        = 1 + \log\left(\frac{T}{k}\right)
        \leq 1 + \log\left(\frac{T}{T/2}\right)
        = 1 + \log 2
        \, .
    \end{align*}
    On the other hand,
    \begin{align*}
        \sum_{t=k}^T t^{-1/2}
        &\geq \int_{k}^{T+1} z^{-1/2}\,dz
        = 2\sqrt{T+1} - 2\sqrt{k}
        \\
        &\geq 2\sqrt{T+1} - 2\sqrt{\tfrac{T+1}{2}}
        = (2-\sqrt{2})\,\sqrt{T+1}
        \geq \frac{\sqrt{T+2}}{2}
        \, .
    \end{align*}
    Here, for the second inequality we used $k=\lceil T/2\rceil\leq (T+1)/2$ (and hence $\sqrt{k}\leq\sqrt{(T+1)/2}=\sqrt{T+1}/\sqrt{2}$). For the last inequality, since both sides are nonnegative, squaring and rearranging shows that it is equivalent to $(24-16\sqrt{2})(T+1)\geq T+2$; because $\sqrt{2}\leq 17/12$ gives $24-16\sqrt{2}\geq 24-\tfrac{68}{3}=\tfrac{4}{3}$, for $T\geq 2$ we have $(24-16\sqrt{2})(T+1)\geq\tfrac{4}{3}(T+1)\geq T+2$, which holds.
    Combining these bounds yields the claim.
\end{proof}

{
\renewcommand{\proofname}{Proof of \cref{prop:convergence_rate_NSS_eg}}%
\begin{proof}
Let $0<q<1$ and $q\neq 1/2$.
For the sum of the step sizes, since $t^{-q}\geq T^{-q}$ (for $t\leq T$),
\begin{equation}
    \sum_{t=1}^T \beta_t = \beta \sum_{t=1}^T t^{-q} \geq \beta T^{1-q} \geq \beta (1-q) T^{1-q}
    \label{eq:zeta_q_not_SR_1}
\end{equation}
holds.
Moreover, since $z\mapsto z^{-2q}$ is monotonically decreasing, for the sum of the squared step sizes, when $0<q<1/2$,
\begin{equation}
    \sum_{t=1}^T \beta_t^2 = \beta^2 \sum_{t=1}^T t^{-2q}
    \leq \beta^2 \int_0^T z^{-2q}\,dz
    = \frac{\beta^2}{1-2q} T^{1-2q},
    \label{eq:zeta_q_not_SR_2}
\end{equation}
and when $1/2<q<1$,
\begin{equation}
    \sum_{t=1}^T \beta_t^2 = \beta^2 \sum_{t=1}^T t^{-2q}
    \leq \beta^2 \left( 1 + \int_1^T z^{-2q}\,dz \right)
    = \beta^2 \left( 1 + \frac{1 - T^{1-2q}}{2q-1} \right)
    \leq \frac{2q}{2q-1}\,\beta^2
    \label{eq:zeta_q_not_SR_3}
\end{equation}
hold (for the latter, we used that $T^{1-2q}\in(0,1]$ since $2q>1$).

When $0<q<1/2$, by \Cref{prop:convergence_rate_NSS}, \Cref{eq:zeta_q_not_SR_1,eq:zeta_q_not_SR_2}, and $q<1-q$ (i.e., $T^{-(1-q)}\leq T^{-q}$),
\begin{align*}
    \min_{t=1,\ldots,T} \ell(\theta^t) - \min_{\Theta}\ell
    &\leq
    \frac{1}{2}\frac{\mathrm{diam}(\Theta)^2}{\sum_{t=1}^T \beta_t}
    +
    \frac{L^2}{2}\frac{\sum_{t=1}^T \beta_t^2}{\sum_{t=1}^T \beta_t} \\
    &\leq
    \frac{\mathrm{diam}(\Theta)^2}{2\beta (1-q)} T^{-(1-q)}
    +
    \frac{L^2\beta}{2(1-q)(1-2q)} T^{-q} \\
    &\leq
    \frac{\mathrm{diam}(\Theta)^2 + \frac{L^2\beta^2}{1-2q}}{2\beta (1-q)} T^{-q}
\end{align*}
holds. Therefore, \Cref{assu:gradient_based_opt_Q} holds for $0<q<1/2$.

When $1/2<q<1$, by \Cref{prop:convergence_rate_NSS}, \Cref{eq:zeta_q_not_SR_1,eq:zeta_q_not_SR_3},
\begin{align*}
    \min_{t=1,\ldots,T} \ell(\theta^t) - \min_{\Theta}\ell
    &\leq
    \frac{1}{2}\frac{\mathrm{diam}(\Theta)^2}{\sum_{t=1}^T \beta_t}
    +
    \frac{L^2}{2}\frac{\sum_{t=1}^T \beta_t^2}{\sum_{t=1}^T \beta_t} \\
    &\leq
    \frac{\mathrm{diam}(\Theta)^2}{2\beta (1-q)} T^{-(1-q)}
    +
    \frac{L^2 q\beta}{(2q-1)(1-q)} T^{-(1-q)} \\
    &=
    \frac{\mathrm{diam}(\Theta)^2 + \frac{2q L^2\beta^2}{2q-1}}{2\beta (1-q)} T^{-(1-q)}
\end{align*}
holds. Therefore, \Cref{assu:gradient_based_opt_Q} holds for $1/2<q<1$.

Finally, consider $q=1/2$.
Summing \cref{eq:update_PSGD_1} over $t=\lceil T/2\rceil,\lceil T/2\rceil+1,\ldots,T$ yields
\begin{align*}
    \sum_{t=\lceil T/2\rceil}^T 2\beta_t \left( \ell(\theta^t) - \min_{\Theta}\ell \right)
    &\leq
    \| \theta^{\lceil T/2\rceil} - \theta^* \|^2
    - \| \theta^{T+1} - \theta^* \|^2
    + \sum_{t=\lceil T/2\rceil}^T \beta_t^2 \| \nabla \ell(\theta^t) \|^2 \\
    &\leq
    \mathrm{diam}(\Theta)^2
    + \sum_{t=\lceil T/2\rceil}^T \frac{\beta^2}{t} \| \nabla \ell(\theta^t) \|^2 \\
    &\leq
    \mathrm{diam}(\Theta)^2
    + \beta^2 L^2 \sum_{t=\lceil T/2\rceil}^T \frac{1}{t}
    \, .
\end{align*}
On the other hand,
\[
    \sum_{t=\lceil T/2\rceil}^T 2\beta_t \left( \ell(\theta^t) - \min_{\Theta}\ell \right)
    \geq
    2\beta \left( \sum_{t=\lceil T/2\rceil}^T t^{-1/2} \right)
    \left( \min_{t=1,\ldots,T}\ell(\theta^t) - \min_{\Theta}\ell \right).
\]
Therefore,
\begin{align*}
    \min_{t=1,\ldots,T}\ell(\theta^t) - \min_{\Theta}\ell
    &\leq
    \frac{\mathrm{diam}(\Theta)^2 + \beta^2 L^2\sum_{t=\lceil T/2\rceil}^T t^{-1}}
    {2\beta \left( \sum_{t=\lceil T/2\rceil}^T t^{-1/2} \right)} \\
    &=
    \frac{\beta L^2}{2}
    \frac{\frac{\mathrm{diam}(\Theta)^2}{\beta^2 L^2} + \sum_{t=\lceil T/2\rceil}^T t^{-1}}
    {\sum_{t=\lceil T/2\rceil}^T t^{-1/2}}
    \, .
\end{align*}
By \cref{prop:sum_SQ_SQS},
\begin{align*}
    \frac{\beta L^2}{2}
    \frac{\frac{\mathrm{diam}(\Theta)^2}{\beta^2 L^2} + \sum_{t=\lceil T/2\rceil}^T t^{-1}}
    {\sum_{t=\lceil T/2\rceil}^T t^{-1/2}}
    &\leq
    \frac{\beta L^2}{2}\frac{2\left(\frac{\mathrm{diam}(\Theta)^2}{\beta^2 L^2} + 1 + \log 2\right)}{\sqrt{T+2}} \\
    &=
    \frac{\mathrm{diam}(\Theta)^2 + (1+\log 2)\beta^2 L^2}{\beta}(T+2)^{-1/2}
    \, .
\end{align*}
Hence, \cref{assu:gradient_based_opt_Q} holds for $q=1/2$ as well.
\end{proof}
}

\begin{remark}
    The proof idea for the case $q=1/2$ in \cref{prop:convergence_rate_NSS_eg} follows \citet[Theorem 8.30]{Beck-2017-First}.
\end{remark}

\subsection{PSGD (NSL)}
\label{sec:Finite-time_solvability_DDIOP_w_PSGD_NSL}

The convergence rate for PSGD (NSL) is as follows.
\begin{proposition}
    \label{prop:convergence_rate_NSL}
    Let $\Theta$ be a bounded, closed, and convex set, and let $L>0$, $\theta^1\in\Theta$, and $T\in\mathbb{Z}_{\geq 1}$.
    Then, \emph{(A)} for any $L$-Lipschitz convex function $\ell\colon\Theta\to\mathbb{R}$, the sequence $\{\theta^t\}_{t=1}^T$ produced by PSGD (NSL) satisfies
    \[
        \min_{t=1,\ldots,T} \ell(\theta^t) - \min_{\Theta} \ell
        \leq
        \frac{L}{2}\frac{\mathrm{diam}(\Theta)^2}{\sum_{t=1}^T \beta_t}
        +
        \frac{L}{2}\frac{\sum_{t=1}^T \beta_t^2}{\sum_{t=1}^T \beta_t}
        \, .
    \]
    In particular, \emph{(B)} there exists a function $Q_{L,\theta^1}\colon\mathbb{R}_{>0}\to\mathbb{R}$ satisfying \cref{assu:gradient_based_opt_Q}.
\end{proposition}

\begin{proof}
    (A) Substituting the NSL step size $\alpha_t=\beta_t/\|\nabla \ell(\theta^t)\|$ into \cref{eq:update_PSGD_2} yields
    \[
        \sum_{t=1}^T \frac{\beta_t}{\| \nabla \ell(\theta^t) \|} \left( \ell(\theta^t) - \min_{\Theta} \ell \right)
        \leq \frac{1}{2} \| \theta^1 - \theta^* \|^2
        + \frac{1}{2} \sum_{t=1}^T \beta_t^2
        \, .
    \]
    On the other hand, since $\ell$ is $L$-Lipschitz, we have $\|\nabla \ell(\theta^t)\|\leq L$, and thus
    \begin{align*}
        \sum_{t=1}^T \frac{\beta_t}{\| \nabla \ell(\theta^t) \|} \left( \ell(\theta^t) - \min_{\Theta} \ell \right)
        &\geq
        \sum_{t=1}^T \frac{\beta_t}{\| \nabla \ell(\theta^t) \|} \left( \min_{t=1,\ldots,T} \ell(\theta^t) - \min_{\Theta} \ell \right) \\
        &\geq
        \frac{1}{L}\left( \sum_{t=1}^T \beta_t \right)\left( \min_{t=1,\ldots,T} \ell(\theta^t) - \min_{\Theta} \ell \right).
    \end{align*}
    Combining the two displays proves (A).

    (B) By the definition of NSL, we have $\sum_{t=1}^{\infty}\beta_t=\infty$.
    Hence, the right-hand side of (A) converges to $0$ as $T\to\infty$, which implies the existence of $Q_{L,\theta^1}$ satisfying \cref{assu:gradient_based_opt_Q}.
\end{proof}

\begin{proposition}
    \label{prop:convergence_rate_NSL_eg}
    Consider the same setting as in \cref{prop:convergence_rate_NSL}.
    Let $0<q<1$ and $\beta>0$, and suppose that the step length is given by $\beta_t=\beta t^{-q}$.
    Define
    \[
        Q_{L,\theta^1}(\varepsilon)
        =
        \begin{cases}
            \left( \frac{L(\mathrm{diam}(\Theta)^2 + \frac{\beta^2}{1-2q})}{2\beta (1-q)\varepsilon}  \right)^{1/q} , & 0< q <1/2, \\
            \left( \frac{L(\mathrm{diam}(\Theta)^2 + \frac{2q\beta^2}{2q-1})}{2\beta (1-q)\varepsilon} \right)^{1/(1-q)} , & 1/2 < q <1 ,\\
            \max \left( 1, \left( \frac{ L(\mathrm{diam}(\Theta)^2 + (1 + \log 2) \beta^2 )}{\beta \varepsilon}  \right)^2 -2 \right) , & q =1/2 \, .
        \end{cases}
    \]
    Then $Q_{L,\theta^1}$ satisfies \cref{assu:gradient_based_opt_Q}.
\end{proposition}

\begin{proof}
    Suppose $0<q < 1$ and $q \neq 1/2$.
    Proceeding as in the proof of \cref{prop:convergence_rate_NSS_eg}, we obtain
    \[
        \min_{t = 1 , \ldots, T} \ell (\theta^t ) - \min_{\Theta} \ell
        \leq
        \begin{cases}
            \frac{L (\mathrm{diam}(\Theta)^2 + \frac{\beta^2}{1 - 2 q})}{2\beta (1-q)} T^{-q} , & 0< q <1/2, \\
            \frac{L (\mathrm{diam}(\Theta)^2 + \frac{2 q \beta^2}{2 q - 1})}{2\beta (1-q)} T^{-(1-q)} , & 1/2 < q <1 .
        \end{cases}
    \]
    This implies the claim for $0<q<1$ with $q\neq 1/2$.

    Now consider $q=1/2$.
    Again arguing as in the proof of \cref{prop:convergence_rate_NSS_eg}, we have
    \begin{align}
        \min_{t=1, \ldots, T}\ell (\theta^t ) - \min_{\Theta} \ell
        &\leq
        \frac{L\mathrm{diam}(\Theta)^2
        + \beta^2 L\sum_{t=\lceil T/2 \rceil}^T t^{-1} }{2 \beta \left( \sum_{t=\lceil T/2 \rceil}^T t^{-1/2} \right)}
        \notag\\
        &=
        \frac{\beta L}{2}
        \frac{\frac{\mathrm{diam}(\Theta)^2}{\beta^2}
        + \sum_{t=\lceil T/2 \rceil}^T t^{-1} }{\sum_{t=\lceil T/2 \rceil}^T t^{-1/2}}
        \, .
        \label{eq:NSL_q_half_bound}
    \end{align}
    By \cref{prop:sum_SQ_SQS},
    \begin{align}
        \frac{\beta L}{2}
        \frac{\frac{\mathrm{diam}(\Theta)^2}{\beta^2}
        + \sum_{t=\lceil T/2 \rceil}^T t^{-1} }{\sum_{t=\lceil T/2 \rceil}^T t^{-1/2}}
        &\leq
        \frac{\beta L}{2} \frac{2 \left(\frac{\mathrm{diam}(\Theta)^2}{\beta^2} + 1 + \log 2\right)}{\sqrt{T+2}}
        \notag\\
        &=
        \frac{ L}{\beta} \left( \mathrm{diam}(\Theta)^2 + (1 + \log 2) \beta^2 \right) (T+2)^{-1/2}
        \, .
        \notag
    \end{align}
    Hence, the claim also holds for $q=1/2$.
\end{proof}

By \cref{theo:grad_based_opt_achieve_min_SL} and \cref{prop:convergence_rate_NSL_eg}, we obtain the following corollary.
\begin{corollary}
    \label{cor:convergence_rate_NSL_eg}
    Assume the setting of \cref{prop:convergence_rate_NSL_eg}, and suppose that \cref{assu:WIRL,assu:WIRL-uniqueness} hold.
    Then, if we implement \cref{alg:intention-WIRL-gradual-decay} using PSGD under the present setting, the DDIOP \cref{eq:IOP_linear} is solved exactly within at most
    \begin{equation*}
        \begin{cases}
            \left( \frac{L(\ell_{\mathrm{sub}})\left(\mathrm{diam}(\Theta)^2 + \frac{\beta^2}{1-2q}\right)}{2\beta (1-q)\gamma(\ell_{\mathrm{sub}})}  \right)^{1/q} , & 0< q <1/2, \\
            \left( \frac{L(\ell_{\mathrm{sub}})\left(\mathrm{diam}(\Theta)^2 + \frac{2q\beta^2}{2q-1}\right)}{2\beta (1-q)\gamma(\ell_{\mathrm{sub}})} \right)^{1/(1-q)} , & 1/2 < q <1 ,\\
            \max \left( 1, \left( \frac{L(\ell_{\mathrm{sub}})\left(\mathrm{diam}(\Theta)^2 + (1 + \log 2) \beta^2\right)}{\beta \gamma(\ell_{\mathrm{sub}})}  \right)^2 -2 \right) , & q =1/2
        \end{cases}
    \end{equation*}
    iterations.
\end{corollary}

\section{Characterization of Piecewise Linear Mappings and Polyhedra}
\label{sec:PWLF}

\begin{definition}
    A bounded closed convex set $P$ is called a convex polytope if $P$ is the convex hull of a finite set.
    A set $P \subset \mathbb{R}^m$ is said to be a finite union of polyhedra if $P$ can be represented as the union of finitely many convex polytopes.
\end{definition}

\begin{definition}
    A mapping $f\colon\mathbb{R}^m\to\mathbb{R}$ is called an affine function if there exist $a_1,\ldots,a_m\in\mathbb{R}$ and $b\in\mathbb{R}$ such that,
    for any $x=(x_1,\ldots,x_m)\in\mathbb{R}^m$,
    \[
        f(x)=a_1x_1+\cdots+a_mx_m+b
    \]
    holds.
    A mapping $f\colon\mathbb{R}^m\to\mathbb{R}^n$ is called an affine map if every component $f_i$ of $f$ is an affine function.
\end{definition}

\begin{definition}
    Let $P\subset\mathbb{R}^m$ be a closed convex set.
    A function $f\colon P\to\mathbb{R}$ is said to be piecewise linear if there exists a finite family $\mathcal{Q}$ of convex polytopes such that
    \[
        P=\bigcup_{Q\in\mathcal{Q}} Q
    \]
    and, for every $Q\in\mathcal{Q}$, the restriction $f\rest{Q}$ is an affine function.
    A mapping $f\colon P\to\mathbb{R}^n$ is said to be piecewise linear if every component $f_i$ of $f$ is piecewise linear.
\end{definition}

\begin{definition}
    Let $P$ be a closed convex set.
    A mapping $f\colon P\to\mathbb{R}^n$ is said to be polyhedral if its graph
    \[
        \Gamma(f)=\{(x,f(x))\mid x\in P\}
    \]
    is a finite union of polyhedra.
\end{definition}

In this setting, piecewise linear functions can be characterized as follows.
\begin{proposition}
\label{prop:PolyhedralMap}
    Assume that the domain $\dom f$ of a mapping $f$ is a convex polytope.
    Then the following statements are equivalent.
    (A) The mapping $f$ is polyhedral.
    (B) The mapping $f$ is piecewise linear.
    (C) The mapping $f$ can be represented as a composition of an affine map, $\min$, and $\max$.
\end{proposition}

It is also known that convex piecewise linear functions can be characterized as follows.
\begin{proposition}[{\citealt[Corollary 19.1.2]{rockafellar1970conjugate}, cf.\ \cref{prop:PolyhedralMap}}]
\label{prop:convexPolyhedralMap}
    For a function $\ell\colon\Theta\to\mathbb{R}$, the following statements are equivalent.
    (A) The graph of $\ell$ is a convex polytope.
    (B) The function $\ell$ is convex and piecewise linear.
    (C) There exists a finite set $I$ such that, for any $i\in I$, there exist $a^i\in\mathbb{R}^d$ and $b_i\in\mathbb{R}$ satisfying
    \[
        \ell(\theta)=\max_{i\in I}\left(a^{i\top}\theta+b_i\right)
        .
    \]
\end{proposition}

In order to prove \cref{prop:PolyhedralMap},
we first review definitions and propositions related to polyhedral mappings.

\begin{definition}
    Let $\mathfrak{A}$ be the family of mappings generated by finite compositions of affine maps and
    \[
        \mathrm{min}_2(x_1,x_2,x_3,\ldots,x_m):=(\min(x_1,x_2),x_3,\ldots,x_m)
        .
    \]
\end{definition}

\begin{proposition}[{\citealp[Proposition 1.6.9]{stallings1967lectures}}]
    \label{prop:composition_PM_is_PM}
    The composition of two polyhedral mappings is polyhedral.
\end{proposition}

\begin{proposition}[{\citealp[Theorem 2.2.4]{stallings1967lectures}}]
    \label{prop:PM_is_PL}
    Every polyhedral mapping is piecewise linear.
\end{proposition}

\begin{proposition}[{\citealp[Theorem 2.1, \S 3]{ovchinnikov2002max}}]
    \label{prop:PL_has_minmax_rep}
    Let $f$ be a piecewise linear mapping.
    Then each component $f_i$ of $f$ admits the representation
    \[
        f_i(x)=\min_{j\in J_i}\max_{k\in K_{ij}} g_{ijk}(x)
        \, ,
    \]
    where $g_{ijk}$ are affine functions, and $J_i$ and $K_{ij}$ are finite sets.
\end{proposition}

{
\renewcommand{\proofname}{Proof of \cref{prop:PolyhedralMap}}%
\begin{proof}
    (A)$\Rightarrow$(B) follows from \cref{prop:PM_is_PL}.
    
    (B)$\Rightarrow$(C)
    Let $f$ be a piecewise linear mapping.
    By \cref{prop:PL_has_minmax_rep}, each component $f_i$ of $f$ can be written as
    \[
        f_i(x)=\min_{j\in J_i}\max_{k\in K_{ij}} g_{ijk}(x)
    \]
    , where $g_{ijk}$ are affine functions.
    On the other hand, for any $k\in\mathbb{Z}_{\geq 2}$,
    \begin{align*}
        \mathrm{min}_{k+1}(x_1,\ldots,x_{k+1},x_{k+2},\ldots,x_m)
        &:=
        \mathrm{min}_{2}\Bigl(
            \mathrm{min}_{k}(x_1,\ldots,x_{k}),
            x_{k+1},
            x_{k+2},\ldots,x_m
        \Bigr)
        \, , \\
        \mathrm{max}_{k}(x_1,\ldots,x_{k},x_{k+1},\ldots,x_m)
        &:=
        -
        \mathrm{min}_{k}(-x_1,\ldots,-x_{k},-x_{k+1},\ldots,-x_m)
    \end{align*}
    belong to $\mathfrak{A}$.
    Therefore, $\min_{j\in J_i}$ and $\max_{k\in K_{ij}}$ can be represented as elements of $\mathfrak{A}$.
    Hence, $f\in\mathfrak{A}$.
    
    (C)$\Rightarrow$(A)
    It suffices to show that affine maps and $\mathrm{min}_2$ are polyhedral; then, by \cref{prop:composition_PM_is_PM},
    $\mathfrak{A}$ is contained in the class of polyhedral mappings.
    
    It is immediate that affine maps are polyhedral.
    Next, we show that $\mathrm{min}_2$ is polyhedral.
    Let $P$ be a finite union of convex polytopes. Then there exists a finite family $\mathcal{Q}$ of convex polytopes such that
    \[
        P=\bigcup_{Q\in\mathcal{Q}} Q
    \]
    holds.
    The graph of $\mathrm{min}_2\rest{P}$ is given by
    \[
        \Gamma\left(\mathrm{min}_2\rest{P}\right)
        =
        \Gamma\left(\mathrm{min}_2\rest{\bigcup_{Q\in\mathcal{Q}} Q}\right)
        =
        \bigcup_{Q\in\mathcal{Q}} \Gamma\left(\mathrm{min}_2\rest{Q}\right)
        .
    \]
    For each $Q\in\mathcal{Q}$,
    \begin{align*}
        & \Gamma\left(\mathrm{min}_2\rest{Q}\right)
        \\
        &=
        \{(x,\mathrm{min}_2(x))\mid x\in Q\} \\
        &=
        \{(x,\mathrm{min}_2(x))\mid x\in Q,\ x_1\geq x_2\}
        \cup
        \{(x,\mathrm{min}_2(x))\mid x\in Q,\ x_2\geq x_1\} \\
        &=
        \{(x,x_1,x_3,\ldots,x_m)\mid x\in Q,\ x_1\geq x_2\}
        \cup
        \{(x,x_2,x_3,\ldots,x_m)\mid x\in Q,\ x_2\geq x_1\}
    \end{align*}
    holds.
    Since affine maps send polyhedra to polyhedra, $\Gamma\left(\mathrm{min}_2\rest{Q}\right)$ is a polyhedron.
    Hence, $\mathrm{min}_2$ is polyhedral.
\end{proof}
}

\begin{proposition}
    \label{prop:polyhedral_map_rest_polyhedra}
    Let $X^\prime\subset X\subset\mathbb{R}^m$ be polyhedra, and let $f\colon X\to\mathbb{R}^d$ be a polyhedral mapping.
    Then $f\rest{X^\prime}$ is a polyhedral mapping.
\end{proposition}

\begin{proof}
    Define the graph of the restriction $f\rest{X^\prime}\colon X^\prime\to\mathbb{R}^d$ by
\[
\Gamma(f\rest{X^\prime})
:=\{(x,y)\in\mathbb{R}^{m+d}\mid x\in X^\prime,\ y=f(x)\}
\]
Then,
\[
\Gamma(f\rest{X^\prime})
=\Gamma(f)\cap\bigl(X^\prime\times\mathbb{R}^d\bigr)
\]
holds.
By assumption, $\Gamma(f)$ is a polyhedron; moreover, since $X^\prime$ is a polyhedron,
$X^\prime\times\mathbb{R}^d$ is also a polyhedron.
Therefore, since the intersection of polyhedra is a polyhedron,
$\Gamma(f\rest{X^\prime})$ is a polyhedron.
\end{proof}

\begin{proposition}
    \label{prop:polyhedral_map_image_polyhedron}
    Let $ X\subset\mathbb{R}^m$ be a polyhedron, and let $f\colon X\to\mathbb{R}^d$ be a polyhedral mapping.
    Then $f(X)\subset\mathbb{R}^d$ is a polyhedron.
\end{proposition}

\begin{proof}
    Define the projection $\pi\colon\mathbb{R}^{m+d}\to\mathbb{R}^d$ by $\pi(x,y)=y$.
    By definition,
    \[
        f(X)=\{f(x)\mid x\in X\}=\{y\in\mathbb{R}^d\mid \exists x\in X,\ (x,y)\in\Gamma(f)\}=\pi(\Gamma(f))
    \]
    holds.
    Therefore, it suffices to show that if $\Gamma(f)$ is a polyhedron, then its linear image $\pi(\Gamma(f))$ is a polyhedron.
    
    First, since $\Gamma(f)$ is a polyhedron, there exist finitely many convex polytopes $\{P_j\}_{j=1}^J$ such that
    \[
        \Gamma(f)=\bigcup_{j=1}^J P_j
        .
    \]
    Hence,
    \[
        f(X)=\pi(\Gamma(f))=\pi\Bigl(\bigcup_{j=1}^J P_j\Bigr)=\bigcup_{j=1}^J \pi(P_j)
    \]
    and thus it suffices to show that each $\pi(P_j)$ is a convex polytope.
    
    Fix an arbitrary $j$ and set $P:=P_j$.
    Since $P$ is a convex polytope, there exist points $\{ a^i \}_{i \in I}$ such that
    \[
        P = \mathrm{Conv} (\{ a^i | i \in I \})
        .
    \]
    Since $\pi$ is linear,
    \[
        \pi (P ) = \mathrm{Conv} (\{  \pi (a^i ) | i \in I \})
        .
    \]
    Therefore, $\pi (P)$ is a convex polytope, and the proposition follows.
\end{proof}

\begin{proposition}
    \label{prop:image_polyhedral_map_on_polyhedron}
    Let $X^\prime\subset X\subset\mathbb{R}^m$ be polyhedra, and let $f\colon X\to\mathbb{R}^d$ be a polyhedral mapping.
    Then $f (X^\prime) $ is a polyhedron.
\end{proposition}

\begin{proof}
    This follows from \cref{prop:polyhedral_map_rest_polyhedra,prop:polyhedral_map_image_polyhedron}.
\end{proof}

\begin{proposition}
    \label{prop:image_polyhedral_map_on_polyhedra}
    Let $X^\prime\subset X\subset\mathbb{R}^m$ be a finite union of polyhedra, and let $f\colon X\to\mathbb{R}^d$ be a polyhedral mapping.
    Then $f(X^\prime)$ is a finite union of polyhedra.
\end{proposition}

\begin{proof}
    Suppose that $X^\prime$ can be written, using a family of polyhedra $\{X^i\}_{i\in I}$, as
    \[
        X^\prime=\bigcup_{i\in I} X^i
        .
    \]
    Applying $f$ to both sides yields
    \[
        f(X^\prime)=f\left(\bigcup_{i\in I} X^i\right)=\bigcup_{i\in I} f(X^i)
        .
    \]
    By \cref{prop:image_polyhedral_map_on_polyhedron}, each $f(X^i)$ is a polyhedron; hence $f(X^\prime)$ is a finite union of polyhedra.
\end{proof}

\begin{definition}
    \label{defi:vertex}
    Let $X$ be a finite union of polyhedra.
    A set $Y$ is said to be the vertex set of $X$ if $Y$ is the vertex set of the convex hull $\mathrm{Conv} X$.
\end{definition}

\section{Technical Lemmas: Finite-time exact solvability of the DDIOP for MILPs}

\subsection{A minimum-attainment theorem via gradient-based optimization for convex piecewise-linear functions}
\label{sec:grad_based_opt_for_PWLF_appx}

In this subsection, we prove \cref{theo:PL_convex_achieve_min}.

\begin{proposition}
    \label{prop:PL_convex_eq_PLT_outside_min}
    Under \cref{assu:piecewise_linear_covex}, for any $\theta \in \Theta \setminus \argmin_{\Theta} \ell$, the following hold:
    \emph{(A)} $\ell (\theta) = \widetilde{\ell} (\theta)$;
    \emph{(B)} $\partial \ell (\theta) = \partial \widetilde{\ell} (\theta)$.
\end{proposition}

\begin{proof}
    (A)
    By definition, $\ell \geq \widetilde{\ell}$.
    Fix $\theta \in \Theta \setminus \argmin_{\Theta} \ell$, and take $j\in \argmax_{i \in I} (a^{i \top} \theta + b_i)$.
    Suppose $P_{V_0}a^{j} = 0$. Then the affine function $a^{j\top}\theta + b_j$ is constant on $V$ (with value $\ell(\theta)$), so that for every $\theta' \in \Theta \subset V$ we have $\ell(\theta') \geq a^{j\top}\theta' + b_j = a^{j\top}\theta + b_j = \ell(\theta)$, which means that $\theta$ attains the minimum of $\ell$ on $\Theta$, a contradiction. Hence, $P_{V_0}a^{j} \neq 0$.
    Therefore,
    \begin{align*}
        \ell (\theta)
        &= \max_{i \in I} (a^{i \top } \theta + b_i)
        = a^{j \top } \theta + b_j \\
        &= \max_{i \in I \,:\, P_{V_0}a^i \neq 0 } \left( a^{i \top} \theta + b_i \right)
        = \widetilde{\ell} (\theta).
    \end{align*}

    (B)
    Let $\mathrm{Conv}$ denote the convex hull.
    Then,
    \begin{align*}
        \partial \ell (\theta)
        &=
        \mathrm{Conv}\left\{
            a^i
            \,\middle|\,
            a^{i \top } \theta + b_i = \max_{i \in I} (a^{i \top } \theta + b_i )
        \right\},
        \\
        \partial \widetilde{\ell} (\theta)
        &=
        \mathrm{Conv}\left\{
            a^i
            \,\middle|\,
            P_{V_0}a^i \neq 0,\ 
            a^{i \top } \theta + b_i = \max_{i \in I} (a^{i \top } \theta + b_i )
        \right\}.
    \end{align*}
    By the same reasoning as in (A), at $\theta \notin \argmin_\Theta \ell$ every index $i$ attaining the maximum $\max_{i \in I}(a^{i \top}\theta + b_i)$ satisfies $P_{V_0}a^i \neq 0$. Therefore,
    \[
        \left\{
            a^i
            \,\middle|\,
            a^{i \top } \theta + b_i = \max_{i \in I} (a^{i \top } \theta + b_i )
        \right\}
        =
        \left\{
            a^i
            \,\middle|\,
            P_{V_0}a^i \neq 0,\ 
            a^{i \top } \theta + b_i = \max_{i \in I} (a^{i \top } \theta + b_i )
        \right\}.
    \]
    and thus $\partial \ell (\theta) = \partial \widetilde{\ell} (\theta)$.
\end{proof}

\begin{proposition}
    \label{prop:gamma_is_positive}
    Under \cref{assu:piecewise_linear_covex}, one has $\gamma (\ell ) > 0$.
\end{proposition}

\begin{proof}
    By definition, $\ell \geq \widetilde{\ell}$ (here $\widetilde{\ell}$ is the $\max$ over a subset of the affine functions defining $\ell$), and $\gamma(\ell) = \min_{\Theta}\ell - \min_{\Theta}\widetilde{\ell} \geq 0$.

    Since $\argmin_{\Theta}\ell$ has a relative interior point in $V$, take such a relative interior point $\theta^*$ (relative to $V$). We show that $a^{i\top}\theta^* + b_i < \min_{\Theta}\ell$ holds for every $i$ with $P_{V_0}a^i \neq 0$. Suppose, to the contrary, that $a^{i\top}\theta^* + b_i = \min_{\Theta}\ell$. Since $P_{V_0}a^i \neq 0$, the vector $u := P_{V_0}a^i \in V_0$ satisfies $a^{i\top}u = (P_{V_0}a^i)^\top u = \|P_{V_0}a^i\|^2 > 0$. As $\theta^*$ is a relative interior point of $\argmin_{\Theta}\ell$ (relative to $V$), for all sufficiently small $\varepsilon > 0$ we have $\theta^* + \varepsilon u \in \Theta$ and $\theta^* + \varepsilon u \in \argmin_{\Theta}\ell$, that is, $\ell(\theta^* + \varepsilon u) = \min_{\Theta}\ell$. On the other hand,
    \[
        \ell(\theta^* + \varepsilon u) \geq a^{i\top}(\theta^* + \varepsilon u) + b_i = (a^{i\top}\theta^* + b_i) + \varepsilon\, a^{i\top}u > \min_{\Theta}\ell,
    \]
    a contradiction. Hence, $a^{i\top}\theta^* + b_i < \min_{\Theta}\ell$ for every $i$ with $P_{V_0}a^i \neq 0$.

    Since \cref{assu:piecewise_linear_covex} assumes that $\ell$ is not constant on $\Theta$, the index set $\{i \in I \mid P_{V_0}a^i \neq 0\}$ of $\widetilde{\ell}$ is nonempty. Therefore,
    \[
        \min_{\Theta}\widetilde{\ell} \leq \widetilde{\ell}(\theta^*) = \max_{i \,:\, P_{V_0}a^i \neq 0} (a^{i\top}\theta^* + b_i) < \min_{\Theta}\ell,
    \]
    and we obtain $\gamma(\ell) = \min_{\Theta}\ell - \min_{\Theta}\widetilde{\ell} > 0$.

\end{proof}

\begin{proposition}
    \label{prop:argmin_PL}
    Under \cref{assu:piecewise_linear_covex},
    \[
        \argmin_{\Theta} \ell
        =
        \left\{
            \theta \in \Theta
        \,\middle|\,
            \widetilde{\ell} (\theta) \leq \min_{\Theta} \widetilde{\ell} + \gamma (\ell)
        \right\}.
    \]
\end{proposition}

\begin{proof}
    Since $\gamma(\ell) = \min_{\Theta}\ell - \min_{\Theta}\widetilde{\ell}$, we have $\min_{\Theta}\widetilde{\ell} + \gamma(\ell) = \min_{\Theta}\ell$, so it suffices to show $\argmin_{\Theta}\ell = \{\theta\in\Theta \mid \widetilde{\ell}(\theta) \leq \min_{\Theta}\ell\}$.
    ($\subseteq$) Let $\theta\in\argmin_{\Theta}\ell$. Since $\ell\geq\widetilde{\ell}$ by definition, $\widetilde{\ell}(\theta)\leq\ell(\theta)=\min_{\Theta}\ell$.
    ($\supseteq$) Let $\theta\in\Theta$ satisfy $\widetilde{\ell}(\theta)\leq\min_{\Theta}\ell$. If $\theta\notin\argmin_{\Theta}\ell$, then by \cref{prop:PL_convex_eq_PLT_outside_min}(A) we have $\ell(\theta)=\widetilde{\ell}(\theta)\leq\min_{\Theta}\ell$; since $\ell(\theta)\geq\min_{\Theta}\ell$ always holds, this yields $\ell(\theta)=\min_{\Theta}\ell$, i.e., $\theta\in\argmin_{\Theta}\ell$, a contradiction. Hence $\theta\in\argmin_{\Theta}\ell$.
\end{proof}

{
\renewcommand{\proofname}{Proof of \cref{theo:PL_convex_achieve_min}}%
\begin{proof}
    Assume that, for all $t=1 , \ldots , T$,
    \begin{equation}
        \min_{t=1 , \ldots , T} \ell (\theta^t ) > \min_{\Theta} \ell .
        \label{eq:assu:not_achieve_argmin}
    \end{equation}

    Step 1:
    By \cref{prop:PL_convex_eq_PLT_outside_min}(A), for every $t=1 , \ldots , T$,
    \begin{equation}
        \ell (\theta^t) = \widetilde{\ell} (\theta^t).
        \label{eq:PL_PLT}
    \end{equation}
    Moreover, by \cref{prop:PL_convex_eq_PLT_outside_min}(B), we may identify
    \begin{equation}
        \nabla \widetilde{\ell} (\theta^t) = \nabla \ell (\theta^t).
        \label{eq:nabla_PL_PLT}
    \end{equation}

    Step 2:
    Define a sequence $\{ \theta^{\prime t} \}_{t}$ inductively by setting, for $t=1 , \ldots , T-1$,
    \[
        \theta^{\prime t+1}
        =
        \mathrm{update}_t \left( \{ \theta^{\prime t^\prime} \}_{t^\prime =1}^t , \widetilde{\ell} ,\nabla \widetilde{\ell} \right).
    \]
    We show that $\theta^t = \theta^{\prime t}$ for all $t=1 , \ldots , T$.

    For $t=1$, we have $\theta^{1} = \theta^{\prime 1}$ by definition.
    Assume that $\theta^{t^\prime} = \theta^{\prime t^\prime}$ holds for all $t^\prime \leq t$.
    By the definitions of $\theta^{t+1}$, $\theta^{\prime t+1}$, and $\mathrm{update}_t$,
    \begin{align*}
        \theta^{t+1}
        &=
        \mathrm{update}_t \left( \{ \theta^{t^\prime} \}_{t^\prime =1}^t \mid \ell , \nabla \ell \right)
        =
        \mathrm{update}_t \left(
            \{ \theta^{t^\prime} \}_{t^\prime =1}^t ,
            \{ \ell (\theta^{t^\prime}) \}_{t^\prime =1}^t,
            \{ \nabla \ell (\theta^{t^\prime}) \}_{t^\prime =1}^t
        \right),
        \\
        \theta^{\prime t+1}
        &=
        \mathrm{update}_t \left(
            \{ \theta^{\prime t^\prime} \}_{t^\prime =1}^t ,
            \{ \widetilde{\ell} (\theta^{\prime t^\prime}) \}_{t^\prime =1}^t,
            \{ \nabla \widetilde{\ell} (\theta^{\prime t^\prime}) \}_{t^\prime =1}^t
        \right).
    \end{align*}
    By \cref{eq:PL_PLT,eq:nabla_PL_PLT} and \cref{assu:gradient_based_opt_independ_ell},
    \begin{align*}
        &
        \mathrm{update}_t \left(
            \{ \theta^{t^\prime} \}_{t^\prime =1}^t ,
            \{ \ell (\theta^{t^\prime}) \}_{t^\prime =1}^t,
            \{ \nabla \ell (\theta^{t^\prime}) \}_{t^\prime =1}^t
        \right)
        \\
        &=
        \mathrm{update}_t \left(
            \{ \theta^{\prime t^\prime} \}_{t^\prime =1}^t ,
            \{ \widetilde{\ell} (\theta^{\prime t^\prime}) \}_{t^\prime =1}^t,
            \{ \nabla \widetilde{\ell} (\theta^{\prime t^\prime}) \}_{t^\prime =1}^t
        \right).
    \end{align*}
    Hence, $\theta^{t+1} = \theta^{\prime t+1}$.
    By induction, $\theta^t = \theta^{\prime t}$ for all $t=1 , \ldots , T$.

    Step 3:
    The Lipschitz constant of $\widetilde{\ell}$ is bounded by $\max_i \| a^i \| = L(\ell)$.
    If $T \geq Q_{L(\ell), \theta^1} (\gamma (\ell))$, then by \cref{assu:gradient_based_opt_Q},
    \[
        \min_{t = 1 , \ldots , T } \widetilde{\ell} (\theta^t ) - \min_{\Theta} \widetilde{\ell}
        \leq \gamma (\ell).
    \]
    Hence, there exists $t \in \{1 , \ldots , T\}$ such that
    \[
        \widetilde{\ell} (\theta^t ) \leq \min_{\Theta} \widetilde{\ell} + \gamma (\ell).
    \]
    By \cref{prop:argmin_PL}, there exists $t \in \{1 , \ldots , T\}$ such that
    \[
        \ell (\theta^t ) = \min_{\Theta} \ell,
    \]
    which contradicts \cref{eq:assu:not_achieve_argmin}.
    Therefore, \cref{theo:PL_convex_achieve_min} holds.
\end{proof}
}

\subsection{Proof of Proposition~\ref{prop:argmin_SL_has_inner_pt}}
\label{sec:proof_prop:argmin_SL_has_inner_pt}

{
\renewcommand{\proofname}{Proof of \cref{prop:argmin_SL_has_inner_pt}}%
\begin{proof}
    (A)
    By \cref{eq:SL_is_PWLF}, \cref{assu:WIRL,assu:WIRL-uniqueness}, and \cref{prop:convexPolyhedralMap}, the suboptimality loss is piecewise linear.

    (B)
    By definition,
    \[
        \theta^{ \top} a^* (\theta , s^{(n)} ) - \theta^{ \top} a^{(n)} \geq 0 .
    \]
    Hence, $\ell_\mathrm{sub} (\theta) \geq 0$.
    Moreover, since $a^{(n)} = a^* (\theta^* , s^{(n)} )$ for every $n=1,\ldots,N$, we have $\ell_\mathrm{sub} (\theta^*) = 0$.
    Therefore, $\min_{\Theta} \ell_\mathrm{sub} = 0$.

    (C)
    Fix $n\in\{1,\ldots,N\}$.
    Define the open set $\Theta ( a^{(n)} , Y^{(n)})$ by
    \[
        \Theta ( a^{(n)} , Y^{(n)})
        =
        \left\{
            \theta \in \Theta
            \,\middle|\,
            \forall a^n \in Y^{(n)} \setminus \{ a^{(n)} \},\ 
            \theta^\top a^{(n)} > \theta^\top a^n
        \right\}.
    \]
    By \cref{assu:WIRL-uniqueness} and $Y^{(n)} \subset f (X (s^{(n)}), s^{(n)})$,
    \begin{equation}
        \theta^* \in \Theta ( a^{(n)} , Y^{(n)}).
        \label{eq:true_weight_in_unique_set}
    \end{equation}
    On the other hand, for any $\theta \in \Theta ( a^{(n)} , Y^{(n)})$, we have $a^* (\theta  , s^{(n)} ) = a^{(n)}$, and hence $\ell_\mathrm{sub} (\theta) = 0$.
    By \cref{eq:true_weight_in_unique_set},
    \[
        \theta^* \in \bigcap_{n=1}^N \Theta ( a^{(n)} , Y^{(n)}) \subset \argmin_{\Theta} \ell_\mathrm{sub}.
    \]
    Finally, we verify that $\argmin_{\Theta}\ell_\mathrm{sub}$ has a relative interior point (in the affine hull $V$ of $\Theta$). Each $\Theta(a^{(n)}, Y^{(n)})$ is defined by finitely many strict linear inequalities, so it can be written as $\Theta(a^{(n)},Y^{(n)}) = \Theta \cap U_n$ for some open set $U_n$ of $\mathbb{R}^d$. The set $U := \bigcap_{n=1}^N U_n$ is an open subset of $\mathbb{R}^d$ containing $\theta^*$. Since $\Theta$ is a nonempty convex set, its interior $\mathrm{int}_V \Theta$ relative to $V$ (the relative interior) is nonempty, and $\theta^* \in \Theta \subseteq \overline{\mathrm{int}_V \Theta}$. Hence the open neighborhood $U$ of $\theta^*$ meets $\mathrm{int}_V \Theta$, and every $\theta \in U \cap \mathrm{int}_V \Theta$ is an interior point, relative to $V$, of $\bigcap_{n=1}^N \Theta(a^{(n)},Y^{(n)}) \subseteq \argmin_{\Theta} \ell_\mathrm{sub}$. Therefore $\argmin_{\Theta} \ell_\mathrm{sub}$ has a relative interior point.
\end{proof}
}

\subsection{Proof of Proposition~\ref{prop:grad_based_opt_achieve_min_SL_before}}
\label{sec:proof_prop:grad_based_opt_achieve_min_SL_before}

{
\renewcommand{\proofname}{Proof of \cref{prop:grad_based_opt_achieve_min_SL_before}}%
\begin{proof}
    We first show that the right-hand side of \cref{eq:gamma_SL} (i.e., $\gamma(\ell_\mathrm{sub})$) is positive. Choosing $\theta = \theta^*$ in the $\max_{\theta\in\Theta}$ gives
    \[
        \gamma(\ell_\mathrm{sub}) \geq \min_{(a^n)_n \in \prod_{n=1}^N Y^{(n)} \setminus \{(a^{(n)})_n\}} \frac{1}{N}\sum_{n=1}^N \theta^{*\top}(a^{(n)} - a^n) .
    \]
    For each $n$, the point $a^{(n)} = a^*(\theta^*, s^{(n)})$ is the (by \cref{assu:WIRL-uniqueness} unique) maximizer over $Y^{(n)}$ under $\theta^*$, so $\theta^{*\top}(a^{(n)} - a^n) \geq 0$; moreover, if $(a^n)_n \neq (a^{(n)})_n$, then $a^{n_0} \neq a^{(n_0)}$ for some $n_0$, whence $\theta^{*\top}(a^{(n_0)} - a^{n_0}) > 0$ (a strict inequality by uniqueness). Since $\prod_{n=1}^N Y^{(n)} \setminus \{(a^{(n)})_n\}$ is a finite set, the above $\min$ is attained at a positive value, and hence $\gamma(\ell_\mathrm{sub}) > 0$.

    Next, we show that the thickness $\min_\Theta\ell_\mathrm{sub} - \min_\Theta\widetilde{\ell_\mathrm{sub}}$ of $\ell_\mathrm{sub}$ (in the sense of \cref{eq:def_gamma_ell}) is at least $\gamma(\ell_\mathrm{sub})$.
    Suppose that
    \[
        \sum_{n=1}^N \left( a^n - a^{(n)} \right) = 0 .
    \]
    Applying $\theta^{* \top}$ to both sides yields
    \[
        \sum_{n=1}^N \theta^{* \top}\left( a^n - a^{(n)} \right) = 0 .
    \]
    By the definition of $a^{(n)}$, we have $\theta^{* \top}\left( a^n - a^{(n)} \right) \leq 0$.
    Hence, for every $n=1,\ldots,N$,
    \[
        \theta^{* \top}\left( a^n - a^{(n)} \right) = 0 .
    \]
    By \cref{assu:WIRL-uniqueness}, it follows that
    \[
        a^n = a^{(n)} .
    \]
    Therefore,
    \[
        \sum_{n=1}^N \left( a^n - a^{(n)} \right) = 0
        \Longleftrightarrow
        \forall n =1 , \ldots , N,\ a^n = a^{(n)} .
    \]
    The affine function corresponding to a tuple $(a^n)_n$ has gradient $g = \frac{1}{N}\sum_{n=1}^N (a^n - a^{(n)})$. If $P_{V_0}g \neq 0$, then $g \neq 0$, so by the equivalence above $(a^n)_n \neq (a^{(n)})_n$. Hence the index set $\{(a^n)_n \mid P_{V_0}g \neq 0\}$ of $\widetilde{\ell_\mathrm{sub}}$ is a subset of $\prod_{n=1}^N Y^{(n)} \setminus \{(a^{(n)})_n\}$, and for every $\theta \in \Theta$,
    \[
        \widetilde{\ell_\mathrm{sub}} (\theta) = \max_{(a^n)_n \,:\, P_{V_0}g \neq 0} \frac{1}{N}\sum_{n=1}^N \theta^\top (a^n - a^{(n)})
        \leq \max_{(a^n)_n \neq (a^{(n)})_n} \frac{1}{N}\sum_{n=1}^N \theta^\top (a^n - a^{(n)}) .
    \]
    Taking $\min_\Theta$ of both sides and using $\min_\Theta\ell_\mathrm{sub} = 0$ (\cref{prop:argmin_SL_has_inner_pt}(B)) together with \cref{eq:gamma_SL}, we obtain
    \[
        \min_\Theta\ell_\mathrm{sub} - \min_\Theta\widetilde{\ell_\mathrm{sub}}
        \geq - \min_{\theta\in\Theta} \max_{(a^n)_n \neq (a^{(n)})_n} \frac{1}{N}\sum_{n=1}^N \theta^\top (a^n - a^{(n)})
        = \gamma(\ell_\mathrm{sub}) .
    \]
    That is, the thickness of $\ell_\mathrm{sub}$ is at least $\gamma(\ell_\mathrm{sub})$.
\end{proof}
}

\section{Universal finite-time exact solvability of the DDIOP for MILPs}
\label{sec:universal_finite_time_MILP_DDIOP_appx}

As a technical condition, we introduce the following assumption.

\begin{assumption}
    \label{assu:universal_WIRL-uniqueness}
    For almost every $\theta \in \Theta$, and for every $n=1,\ldots,N$, the feature $a^*(\theta,s^{(n)})$ is uniquely determined.
\end{assumption}

\begin{example}
    By \cref{lem:Psi_set_is_almost_Phi}, the choice $\Theta = \Delta^{d-1}$ satisfies \cref{assu:universal_WIRL-uniqueness}.
\end{example}

Under \cref{assu:universal_WIRL-uniqueness}, define
\begin{align*}
    & \Theta_{\mathrm{unique}}
    \\
    & :=
    \left\{
        \theta \in \Theta
        \,\middle|\,
        \text{for every $n=1,\ldots,N$, the feature $a^*(\theta,s^{(n)})$ is uniquely determined}
    \right\}.
\end{align*}
For any $n=1,\ldots,N$ and any $\theta \in \Theta_{\mathrm{unique}}$, condition \cref{eq:optimal_in_vertex} holds.
Since each $Y^{(n)}$ is finite,
\[
    \left\{
        (a^*(\theta,s^{(n)}))_{n}
        \,\middle|\,
        \theta \in \Theta_{\mathrm{unique}}
    \right\}
    \left( \subset \prod_{n=1}^N Y^{(n)} \right)
\]
is a finite set.
By \cref{prop:grad_based_opt_achieve_min_SL_before,prop:gamma_is_positive}, for any $\theta^* \in \Theta_{\mathrm{unique}}$, the following constants $\gamma_{\theta^*}$ and $\widetilde{\gamma}$ are strictly positive:
\begin{align}
    \gamma_{\theta^*}
    & :=
    \max_{\theta \in \Theta}
    \min_{(a^*(\theta^*,s^{(n)}))_n \in \prod_{n=1}^N Y^{(n)} \setminus \{ (a^{(n)})_n \}}
    \frac{1}{N} \sum_{n=1}^N \theta^\top \left( a^{(n)} - a^*(\theta^*,s^{(n)}) \right),
    \notag \\
    \widetilde{\gamma}
    & :=
    \min_{\theta^* \in \Theta_{\mathrm{unique}}} \gamma_{\theta^*}.
\end{align}
By definition, $\widetilde{\gamma}$ depends only on $\{X(s^{(n)})\}_n$, $f$, and $\Theta$.

\begin{theorem}
    \label{theo:universal_grad_based_opt_achieve_min_SL}
    Assume that \cref{assu:WIRL,assu:gradient_based_opt_independ_ell,assu:gradient_based_opt_Q,assu:universal_WIRL-uniqueness} hold.
    Then, for almost every $\theta^* \in \Theta$, the following implication holds:
    if $T \geq Q_{L (\ell_\mathrm{sub}), \theta^1} (\widetilde{\gamma})$, then
    \[
        \min_{t=1 , \ldots , T} \ell_{\mathrm{sub}} (\theta^t )
        =
        \min_{\Theta} \ell_{\mathrm{sub}}
        =
        0.
    \]
\end{theorem}

\begin{proof}
    Fix any $T \geq Q_{L (\ell_\mathrm{sub}), \theta^1} (\widetilde{\gamma})$.
    By the monotonicity of $Q_{L,\theta^1}$, for any $\theta^* \in \Theta_{\mathrm{unique}}$ we have
    $T \geq Q_{L (\ell_\mathrm{sub}), \theta^1} (\gamma_{\theta^*})$.
    Hence, by \cref{theo:grad_based_opt_achieve_min_SL},
    \[
        \min_{t=1 , \ldots , T} \ell_{\mathrm{sub}} (\theta^t ) = 0.
    \]
    Since \cref{assu:universal_WIRL-uniqueness} implies that $\Theta_{\mathrm{unique}}$ has full measure in $\Theta$, the conclusion holds for almost every $\theta^* \in \Theta$.
\end{proof}

\section{Technical Lemmas: finite-time attainment of the PLF minimum}
\label{sec:finite_time_PLF_min_appx}

We show that PSGD can drive a PLF to $0$ within finitely many iterations.

Let $\theta^{**} \in \Theta$ satisfy
\[
    \theta^{**} \in \argmin_{\Theta} \widetilde{\ell_{\mathrm{sub}} }
    \left( \subset \argmin_{\Theta} \ell_{\mathrm{sub}} \right).
\]

\begin{lemma}
    \label{lem:intention_learning_complete_Psi}
    Assume \cref{assu:WIRL,assu:WIRL-uniqueness}.
    Then, for any $\theta \in \Theta$, the following are equivalent:
    \emph{(A)} $g(\theta)=0$;
    \emph{(B)} $\ell_{\mathrm{plf}}(\theta)=0$.
\end{lemma}

\begin{remark}
Assume that the feature $a^*(\theta, s^{(n)})$ is defined as the lexicographically minimal element of
\[
    \argmax_{f(x, s)\in f(X(s), s)}\ \theta^{\top} f(x, s).
\]
Under this convention, \citet{Kitaoka-2023-imitation-WIRL} argued that once learning via \cref{alg:intention-WIRL-gradual-decay} is completed, i.e., once a subgradient of the suboptimality loss becomes zero, one has $a^*(\theta, s^{(n)}) = a^{(n)}$.
The statement and proof of \cref{lem:intention_learning_complete_Psi} are inspired by \citet{Kitaoka-2023-imitation-WIRL}.
\end{remark}

{
\renewcommand{\proofname}{Proof of \cref{lem:intention_learning_complete_Psi}}%
\begin{proof}
  (A) $\Rightarrow$ (B)
  Assume that a subgradient of $\ell_{\mathrm{sub}}$ is given by
  \[
      \sum_{n=1}^N \left(
          a^*(\theta, s^{(n)})
          -
          a^*(\theta^*, s^{(n)})
      \right)
      = 0.
  \]
  Applying $\theta^{*\top}$ to both sides yields
  \[
      \sum_{n=1}^N \theta^{*\top} a^*(\theta^*, s^{(n)})
      =
      \sum_{n=1}^N \theta^{*\top} a^*(\theta, s^{(n)}).
  \]
  By the definition of $a^*(\theta^*, s^{(n)})$, for every $n$,
  \[
      \theta^{*\top} a^*(\theta^*, s^{(n)})
      \geq
      \theta^{*\top} a^*(\theta, s^{(n)}),
  \]
  and hence, for every $n$,
  \[
      \theta^{*\top} a^*(\theta^*, s^{(n)})
      =
      \theta^{*\top} a^*(\theta, s^{(n)}).
  \]
  By \cref{eq:FOP_linear},
  \[
      a^*(\theta, s^{(n)})
      \in
      \argmax_{f(x, s^{(n)})\in f(X(s^{(n)}), s^{(n)})}\ \theta^{*\top} f(x, s^{(n)}).
  \]
  Since $\theta^*\in\Psi$, for every $n$ there exists $\xi^{(n)}\in Y^{(n)}$ such that
  \[
      \theta^* \in \Theta(\xi^{(n)}, Y^{(n)}).
  \]
  This implies that, for every $n$, the set $\argmax_{f(x, s^{(n)})\in f(X(s^{(n)}), s^{(n)})}\theta^{*\top} f(x, s^{(n)})$ is a singleton.
  Therefore,
  \[
      a^*(\theta^*, s^{(n)}) = a^*(\theta, s^{(n)}).
  \]

  (B) $\Rightarrow$ (A)
  Under (B), for every $n$,
  \[
      a^*(\theta, s^{(n)})
      -
      a^*(\theta^*, s^{(n)})
      =
      0.
  \]
  Summing over $n$ yields (A).
\end{proof}
}

\begin{lemma}
    \label{lem:SPO_bound}
    Assume \cref{assu:WIRL,assu:WIRL-uniqueness}.
    Then, for any $\theta \in \Theta$ satisfying $g(\theta) \neq 0$,
    \[
        - \gamma (\ell_{\mathrm{sub}})
        =
        \max_{( a^{n} )_n \in \prod_{n=1}^N Y^{(n)} \setminus \{ (a^{(n)} )_n \} }
        \frac{1}{N} \sum_{n=1}^N \theta^{**\top} \left( a^{(n)} - a^n \right)
        \geq \theta^{**\top} g (\theta )
    \]
    holds.
\end{lemma}

\begin{proof}
    By \cref{prop:grad_based_opt_achieve_min_SL_before} and the choice of $\theta^{**}$,
    \begin{align*}
        - \gamma (\ell_\mathrm{sub})
        = - \min_{\Theta} \widetilde{\ell_{\mathrm{sub}}}
        &=
        \min_{\theta \in \Theta}
        \max_{( a^{n} )_n \in \prod_{n=1}^N Y^{(n)} \setminus \{ (a^{(n)} )_n \} }
        \frac{1}{N} \sum_{n=1}^N \theta^\top \left( a^n - a^{(n)} \right)
        \\
        &=
        \max_{( a^{n} )_n \in \prod_{n=1}^N Y^{(n)} \setminus \{ (a^{(n)} )_n \} }
        \frac{1}{N} \sum_{n=1}^N \theta^{**\top} \left( a^{(n)} - a^n \right).
    \end{align*}
    The claim follows from \cref{prop:SL_is_Lipschitz_convex}(C).
\end{proof}

\begin{proposition}
    \label{prop:update_estimation_-2}
    Assume \cref{assu:WIRL,assu:WIRL-uniqueness}.
    Let $\{\theta^t\}_t$ be the sequence obtained by implementing the update rule $\{\mathrm{update}_t\}_t$ in \cref{alg:intention-WIRL-gradual-decay} with PSGD.
    Then, for any $t\in\mathbb{Z}_{\geq 1}$,
    \[
        \| \theta^{t+1} - \theta^{**} \|^2
        \leq
        \left\| \theta^{t}  - \theta^{**} \right\|^2
        - 2 \alpha_t \gamma (\ell_{\mathrm{sub}})
        + \alpha_t^2 \left\|  g (\theta^{t}  ) \right\|^2 .
    \]
\end{proposition}

\begin{proof}
    By \cref{lem:SPO_bound} and $\theta^{t \top} g (\theta^{t} ) \geq 0$, the distance between $\theta^{t+1}$ and $\theta^{**}$ can be bounded as follows:
    \begin{align*}
        \| \theta^{t+1} - \theta^{**} \|^2
        &\leq
        \left\| \Proj_{\Theta} \left( \theta^{t}  - \alpha_t g (\theta^{t}  ) \right) - \theta^{**} \right\|^2
        \\
        &\leq
        \left\| \theta^{t}  - \alpha_t g (\theta^{t}  ) - \theta^{**} \right\|^2
        \\
        &=
        \left\| \theta^{t}  - \theta^{**} \right\|^2
        + 2 \alpha_t (\theta^{**} - \theta^{t}  )^{ \top} g (\theta^{t} )
        + \alpha_t^2 \left\|  g (\theta^{t}  ) \right\|^2
        \\
        &\leq
        \left\| \theta^{t}  - \theta^{**} \right\|^2
        - 2 \alpha_t \gamma (\ell_{\mathrm{sub}})
        + \alpha_t^2 \left\|  g (\theta^{t}  ) \right\|^2 .
    \end{align*}
\end{proof}

\subsection{The case of PSGD (NSS)}

\begin{theorem}
    \label{theo:achieve_min_PLF_PSGD_NSS}
    Assume that \cref{assu:WIRL,assu:WIRL-uniqueness} hold.
    Let $\{\theta^t\}_t$ be the sequence generated by implementing the update rules $\{\mathrm{update}_t\}_t$ in \cref{alg:intention-WIRL-gradual-decay} via PSGD (NSS).
    Then, for sufficiently large $T \in \mathbb{Z}_{\geq 1}$, one has $\ell_{\mathrm{plf}}(\theta^{T})=0$.
\end{theorem}

\begin{proof}
    By \cref{lem:intention_learning_complete_Psi}, the problem of minimizing the PLF can be reduced to the problem of minimizing the convex suboptimality loss.
    Assume that $g(\theta^{t})\neq 0$ for all $t=1,\ldots,T$.
    By \cref{prop:update_estimation_-2},
    \[
        \| \theta^{t+1} - \theta^{**} \|^2
        \leq
        \left\| \theta^{t}  - \theta^{**} \right\|^2
        - 2 \beta_t \gamma (\ell_{\mathrm{sub}})
        + \beta_t^2 L (\ell_{\mathrm{sub}})^2
        .
    \]
    Summing both sides over $t=1,\ldots,T$ yields
    \begin{align}
        \| \theta^{T+1} - \theta^{**} \|^2
        & \leq
        \left\| \theta^{1}  - \theta^{**} \right\|^2
        - 2 \gamma (\ell_{\mathrm{sub}}) \sum_{t=1}^{T } \beta_t
        + L (\ell_{\mathrm{sub}})^2 \sum_{t=1}^{T } \beta_t^2
        \notag \\
        & =
        \left\| \theta^{1}  - \theta^{**} \right\|^2
        - \left( \sum_{t=1}^{T} \beta_t \right)
        \left(
        2 \gamma (\ell_{\mathrm{sub}})
        - L (\ell_{\mathrm{sub}})^2 \frac{\sum_{t=1}^{T} \beta_t^2}{\sum_{t=1}^{T } \beta_t}
        \right).
        \label{eq:achieve_min_PLF_PSGD_NSS_01}
    \end{align}
    By the assumption on the sequence $\{\beta_t\}_t$, for sufficiently large $T \in \mathbb{Z}_{\geq 1}$,
    \[
        \frac{\sum_{t=1}^{T } \beta_t^2}{\sum_{t=1}^{T } \beta_t} \leq \gamma (\ell_{\mathrm{sub}})
    \]
    holds, and thus \cref{eq:achieve_min_PLF_PSGD_NSS_01} becomes strictly negative for sufficiently large $T$.
    This contradicts $\| \theta^{T+1} - \theta^{**} \|^2 \geq 0$.

    Therefore, there exists $t\in\{1,\ldots,T\}$ such that $g(\theta^{t})=0$.
    By \cref{lem:intention_learning_complete_Psi}, there exists $t\in\{1,\ldots,T\}$ such that $\ell_{\mathrm{plf}}(\theta^{t})=0$.
\end{proof}

\begin{theorem}
    \label{theo:achieve_min_PLF_PSGD_SRSS}
    Assume that \cref{assu:WIRL,assu:WIRL-uniqueness} hold.
    Let $\{\theta^t\}_t$ be the sequence generated by implementing the update rules $\{\mathrm{update}_t\}_t$ in \cref{alg:intention-WIRL-gradual-decay} via PSGD (SRSS).
    Then, for any $T \in \mathbb{Z}_{\geq 1}$, either $\ell_{\mathrm{plf}}(\theta^{T})=0$, or
    \[
        \| \theta^{T+1} - \theta^{**} \|^2
        \leq
        \mathrm{diam} (\Theta)^2
        - \beta (2-\sqrt{2}) \left( \sqrt{T} -1 \right)
        \left(
        2 \gamma (\ell_{\mathrm{sub}})
        - \beta L (\ell_{\mathrm{sub}})^2 \frac{2 ( 1 + \log 2)}{\sqrt{T+2}}
        \right)
        .
    \]
\end{theorem}

\begin{proof}
    By \cref{lem:intention_learning_complete_Psi}, the problem of minimizing the PLF can be reduced to the problem of minimizing the convex suboptimality loss.
    Assume that, for $t=1,\ldots,T$, one has $g(\theta^{t})\neq 0$.
    By \cref{prop:update_estimation_-2},
    \begin{align}
        \| \theta^{t+1} - \theta^{**} \|^2
        & \leq
        \left\| \theta^{t}  - \theta^{**} \right\|^2
        - 2 \beta_t \gamma (\ell_{\mathrm{sub}})
        + \beta_t^2 L (\ell_{\mathrm{sub}})^2
        \notag \\
        & = \left\| \theta^{t}  - \theta^{**} \right\|^2
        - 2 \beta \gamma (\ell_{\mathrm{sub}}) t^{-1/2}
        + \beta^2 L (\ell_{\mathrm{sub}})^2 t^{-1}
        .
        \notag
    \end{align}
    Summing over $t= \lceil T/2 \rceil , \lceil T/2 \rceil +1 , \ldots , T$, we obtain
    \begin{align}
        & \| \theta^{T+1} - \theta^{**} \|^2
        \notag \\
        & \leq
        \left\| \theta^{\lceil T/2 \rceil}  - \theta^{**} \right\|^2
        - 2 \beta \gamma (\ell_{\mathrm{sub}}) \sum_{t = \lceil T/2 \rceil }^T t^{-1/2}
        + \beta^2 L (\ell_{\mathrm{sub}})^2 \sum_{t = \lceil T/2 \rceil }^T t^{-1}
        \notag \\
        & \leq
        \left\| \theta^{\lceil T/2 \rceil}  - \theta^{**} \right\|^2
        - \beta \left( \sum_{t = \lceil T/2 \rceil }^T t^{-1/2} \right)
        \left(
        2 \gamma (\ell_{\mathrm{sub}})
        - \beta L (\ell_{\mathrm{sub}})^2 \frac{\sum_{t = \lceil T/2 \rceil }^T t^{-1}}{\sum_{t = \lceil T/2 \rceil }^T t^{-1/2}}
        \right)
        \notag \\
        & \leq
        \mathrm{diam} (\Theta)^2
        - \beta \left( \sum_{t = \lceil T/2 \rceil }^T t^{-1/2} \right)
        \left(
        2 \gamma (\ell_{\mathrm{sub}})
        - \beta L (\ell_{\mathrm{sub}})^2 \frac{\sum_{t = \lceil T/2 \rceil }^T t^{-1}}{\sum_{t = \lceil T/2 \rceil }^T t^{-1/2}}
        \right)
        .
        \notag
    \end{align}
    By \cref{prop:sum_SQ_SQS},
    \begin{align}
        \frac{\sum_{t=\lceil T/2 \rceil}^T t^{-1} }{\sum_{t=\lceil T/2 \rceil}^T t^{-1/2}}
        & \leq
       \frac{2 ( 1 + \log 2)}{\sqrt{T+2}}
       .
       \notag
    \end{align}
    Therefore,
    \begin{align}
        &
        \mathrm{diam} (\Theta)^2
        - \beta \left( \sum_{t = \lceil T/2 \rceil }^T t^{-1/2} \right)
        \left(
        2 \gamma (\ell_{\mathrm{sub}})
        - \beta L (\ell_{\mathrm{sub}})^2 \frac{\sum_{t = \lceil T/2 \rceil }^T t^{-1}}{\sum_{t = \lceil T/2 \rceil }^T t^{-1/2}}
        \right)
        \notag \\
        & \leq
        \mathrm{diam} (\Theta)^2
        - \beta \left( \sum_{t = \lceil T/2 \rceil }^T t^{-1/2} \right)
        \left(
        2 \gamma (\ell_{\mathrm{sub}})
        - \beta L (\ell_{\mathrm{sub}})^2 \frac{2 ( 1 + \log 2)}{\sqrt{T+2}}
        \right)
        .
        \notag
    \end{align}
    Moreover,
    \begin{align}
        \sum_{t = \lceil T/2 \rceil }^T t^{-1/2}
        & \geq
        \sum_{t = \lceil T/2 \rceil +1 }^T t^{-1/2}
        \geq \int_{\lceil T/2 \rceil}^T z^{-1/2} dz
        = 2 T^{1/2} - 2 \lceil T/2 \rceil^{1/2}
        \notag \\
        & = 2 T^{1/2} - \sqrt{2} (T+1)^{1/2}
        = \frac{4 T - 2 (T+1) }{2 T^{1/2} + \sqrt{2} (T+1)^{1/2}}
        \geq \frac{2T-2}{(2 + \sqrt{2})T^{1/2}}
        \notag \\
        & = (2-\sqrt{2}) \left( \sqrt{T} -1 \right).
        \notag
    \end{align}
    Hence,
    \begin{align*}
        &
        \mathrm{diam} (\Theta)^2
        - \beta \left( \sum_{t = \lceil T/2 \rceil }^T t^{-1/2} \right)
        \left(
        2 \gamma (\ell_{\mathrm{sub}})
        - \beta L (\ell_{\mathrm{sub}})^2 \frac{2 ( 1 + \log 2)}{\sqrt{T+2}}
        \right)
        \\
        & \leq
        \mathrm{diam} (\Theta)^2
        - \beta (2-\sqrt{2}) \left( \sqrt{T} -1 \right)
        \left(
        2 \gamma (\ell_{\mathrm{sub}})
        - \beta L (\ell_{\mathrm{sub}})^2 \frac{2 ( 1 + \log 2)}{\sqrt{T+2}}
        \right),
    \end{align*}
    which proves the claim.
\end{proof}

{
\renewcommand{\proofname}{Proof of \cref{theo:achieve_min_PLF_PSGD_SRSS_readable}}%
\begin{proof}
    (A) This follows from \cref{theo:achieve_min_PLF_PSGD_SRSS}.

    (B)
    By the assumption on $T$,
    \begin{equation}
        2 \gamma (\ell_{\mathrm{sub}})
        -
        \beta L (\ell_{\mathrm{sub}})^2 \frac{2 ( 1 + \log 2)}{\sqrt{T+2}}
        \geq \gamma (\ell_{\mathrm{sub}}).
        \notag
    \end{equation}
    By \cref{theo:achieve_min_PLF_PSGD_SRSS},
    \begin{align}
        & \| \theta^{T+1} - \theta^{**} \|^2
        \notag \\
        & \leq
        \mathrm{diam} (\Theta)^2
        - \beta (2-\sqrt{2}) \left( \sqrt{T} -1 \right)
        \left(
        2 \gamma (\ell_{\mathrm{sub}})
        - \beta L (\ell_{\mathrm{sub}})^2 \frac{2 ( 1 + \log 2)}{\sqrt{T+2}}
        \right)
        \notag \\
        & \leq
        \mathrm{diam} (\Theta)^2
        - \beta \gamma (\ell_{\mathrm{sub}}) (2-\sqrt{2}) \left( \sqrt{T} -1 \right)
        .
        \notag
    \end{align}
    By the assumption on $T$,
    \[
        \mathrm{diam} (\Theta)^2
        - \beta \gamma (\ell_{\mathrm{sub}}) (2-\sqrt{2}) \left( T^{1/2} -1 \right) \leq 0 .
    \]
    Combining the above inequalities yields $\| \theta^{T+1} - \theta^{**} \|^2 < 0$, which is a contradiction.

    Therefore, there exists $t\in\{1,\ldots,T\}$ such that $g(\theta^{t})=0$.
    By \cref{lem:intention_learning_complete_Psi}, there exists $t\in\{1,\ldots,T\}$ such that $\ell_{\mathrm{plf}}(\theta^{t})=0$.
\end{proof}
}

\subsection{The case of PSGD (NSL)}

\begin{theorem}
    \label{theo:achieve_min_PLF_PSGD_NSL}
    Assume that \cref{assu:WIRL,assu:WIRL-uniqueness} hold.
    Let $\{\theta^t\}_t$ be the sequence generated by implementing the update rules $\{\mathrm{update}_t\}_t$ in \cref{alg:intention-WIRL-gradual-decay} via PSGD (NSL).
    Then, for sufficiently large $T \in \mathbb{Z}_{\geq 1}$, one has $\ell_{\mathrm{plf}}(\theta^{T}) = 0$.
\end{theorem}

\begin{proof}
    By \cref{lem:intention_learning_complete_Psi}, the problem of minimizing the PLF can be reduced to the problem of minimizing the convex suboptimality loss.
    Assume that $g(\theta^{t}) \neq 0$ for all $t=1,\ldots,T$.
    By \cref{prop:update_estimation_-2},
    \[
        \| \theta^{t+1} - \theta^{**} \|^2
        \leq
        \left\| \theta^{t}  - \theta^{**} \right\|^2
        - 2 \beta_t \frac{\gamma (\ell_{\mathrm{sub}})}{\| g (\theta^t) \| }
        + \beta_t^2
        \leq
        \left\| \theta^{t}  - \theta^{**} \right\|^2
        - 2 \beta_t \frac{\gamma (\ell_{\mathrm{sub}})}{L (\ell_{\mathrm{sub}}) }
        + \beta_t^2
        .
    \]
    Summing both sides over $t=1,\ldots,T$ yields
    \begin{align}
        \| \theta^{T+1} - \theta^{**} \|^2
        & \leq
        \left\| \theta^{1}  - \theta^{**} \right\|^2
        - 2 \frac{\gamma (\ell_{\mathrm{sub}})}{L (\ell_{\mathrm{sub}}) } \sum_{t=1}^{T } \beta_t
        + \sum_{t=1}^{T } \beta_t^2
        \notag \\
        & =
        \left\| \theta^{1}  - \theta^{**} \right\|^2
        - \left( \sum_{t=1}^{T} \beta_t \right)
        \left(
        2 \frac{\gamma (\ell_{\mathrm{sub}})}{L (\ell_{\mathrm{sub}}) }
        - \frac{\sum_{t=1}^{T} \beta_t^2}{\sum_{t=1}^{T } \beta_t}
        \right).
        \label{eq:achieve_min_PLF_PSGD_NSL_01}
    \end{align}
    By the assumption on the sequence $\{\beta_t\}_t$, for sufficiently large $T \in \mathbb{Z}_{\geq 1}$,
    \[
        \frac{\sum_{t=1}^{T } \beta_t^2}{\sum_{t=1}^{T } \beta_t} \leq \frac{\gamma (\ell_{\mathrm{sub}})}{L (\ell_{\mathrm{sub}}) } .
    \]
    Hence, for sufficiently large $T$, the right-hand side of \cref{eq:achieve_min_PLF_PSGD_NSL_01} becomes strictly negative, contradicting $\| \theta^{T+1} - \theta^{**} \|^2 \geq 0$.

    Therefore, there exists $t\in\{1,\ldots,T\}$ such that $g(\theta^{t})=0$.
    By \cref{lem:intention_learning_complete_Psi}, there exists $t\in\{1,\ldots,T\}$ such that $\ell_{\mathrm{plf}}(\theta^{t})=0$.
\end{proof}

\begin{theorem}
    \label{theo:achieve_min_PLF_PSGD_SRSL}
    Suppose that \cref{assu:WIRL,assu:WIRL-uniqueness} hold.
    Let $\{ \theta^t \}_t$ be the sequence generated by implementing the update rule $\{ \mathrm{update}_t \}_t$ of \cref{alg:intention-WIRL-gradual-decay} via PSGD (SRSL).
    Then, for any $T \in \mathbb{Z}_{\geq 1}$, either $\ell_{\mathrm{plf}}(\theta^{T}) = 0$, or
    \[
        \| \theta^{T+1} - \theta^{**} \|^2
        \leq
        \mathrm{diam}(\Theta)^2
        - \beta (2-\sqrt{2}) \left( \sqrt{T} -1 \right)
        \left(
        2 \frac{\gamma (\ell_{\mathrm{sub}})}{L(\ell_{\mathrm{sub}})}
        - \beta \frac{2 ( 1 + \log 2)}{\sqrt{T+2}}
        \right).
    \]
\end{theorem}

\begin{proof}
    The claim can be proved in the same manner as \cref{theo:achieve_min_PLF_PSGD_SRSS}.
\end{proof}

{
\renewcommand{\proofname}{Proof of \cref{theo:achieve_min_PLF_PSGD_SRSL_readable}}%
\begin{proof}
    (A) This follows immediately from \cref{theo:achieve_min_PLF_PSGD_SRSL}.

    (B) By the assumption on $T$, we have
    \begin{equation}
        2 \frac{\gamma (\ell_{\mathrm{sub}})}{L(\ell_{\mathrm{sub}})}
        -
        \beta\frac{2 ( 1 + \log 2)}{\sqrt{T+2}}
        \geq \frac{\gamma (\ell_{\mathrm{sub}})}{L(\ell_{\mathrm{sub}})}.
        \notag
    \end{equation}
    Hence, by \cref{theo:achieve_min_PLF_PSGD_SRSL},
    \begin{align}
        \| \theta^{T+1} - \theta^{**} \|^2
        & \leq
        \mathrm{diam} (\Theta)^2
        - \beta (2-\sqrt{2}) \left( \sqrt{T} -1 \right)
        \left(
        2 \frac{\gamma (\ell_{\mathrm{sub}})}{L(\ell_{\mathrm{sub}})}
        - \beta \frac{2 ( 1 + \log 2)}{\sqrt{T+2}}
        \right)
        \notag \\
        & \leq
        \mathrm{diam} (\Theta)^2
        - \beta \frac{\gamma (\ell_{\mathrm{sub}})}{L(\ell_{\mathrm{sub}})} (2-\sqrt{2}) \left( \sqrt{T} -1 \right).
        \notag
    \end{align}
    By the assumption on $T$, it follows that
    \[
        \mathrm{diam} (\Theta)^2
        - \beta \frac{\gamma (\ell_{\mathrm{sub}})}{L(\ell_{\mathrm{sub}})} (2-\sqrt{2}) \left( T^{1/2} -1 \right) \leq 0.
    \]
    Combining these inequalities yields $\| \theta^{T+1} - \theta^{**} \|^2 < 0$, which is a contradiction.

    Therefore, there exists some $t \in \{1,\ldots,T\}$ such that $g(\theta^{t}) = 0$.
    By \cref{lem:intention_learning_complete_Psi}, there exists some $t \in \{1,\ldots,T\}$ such that $\ell_{\mathrm{plf}}(\theta^{t}) = 0$.
\end{proof}
}

\subsection{Some Remarks}

\begin{remark}
    \label{rem:lower_bound_first_order_method}
    Let $\ell \colon \mathbb{R}^d \to \mathbb{R}$ be an $L$-Lipschitz convex function.
    \citet{nemirovskii1983problem} established matching (tight) lower and upper bounds on the worst-case complexity of (deterministic) gradient-based first-order methods for ensuring that the function-value suboptimality is at most $\varepsilon>0$ in the problem
    \[
        \min_{x \in \mathbb{R}^d,\ \| x \| \leq B} \ell(x),
    \]
    namely, $\Theta(L^2 B^2/\varepsilon^2)$ iterations.

    By \cref{prop:convergence_rate_NSL_eg,prop:convergence_rate_NSS_eg}, PSGD (SRSS, SRSL) attains this rate with respect to $\varepsilon$.
\end{remark}

\subsection{Remark: SPO loss}
\label{sec:SPO_loss}

We define the estimation loss, namely the SPO loss \citep{Mohajerin-2018-Data,elmachtoub2022smart}, by
\[
    \ell_{\mathrm{est}} (\theta ; \theta^*)
    :=
    \frac{1}{N}\sum_{n=1}^N \left( \theta^{* \top} a^{(n)} - \theta^{* \top} a^* (\theta , s^{(n)} ) \right).
\]

\begin{proposition}
    \label{prop:SPO_bound_2}
    Assume that \cref{assu:WIRL,assu:WIRL-uniqueness} hold.
    Then, for any $\theta \in \Theta$ satisfying $g(\theta) \neq 0$,
    \[
        \gamma (\ell_{\mathrm{sub}})
        \leq \ell_{\mathrm{est}} (\theta ; \theta^{**}).
    \]
\end{proposition}

\begin{proof}
    By \cref{lem:SPO_bound},
    \begin{align*}
        - \gamma (\ell_{\mathrm{sub}})
        \geq
        \theta^{**\top} g (\theta )
        =
        - \ell_{\mathrm{est}} (\theta ; \theta^{**}),
    \end{align*}
    which implies the claim.
\end{proof}

\section{Prediction loss and suboptimality loss: the case of strongly convex and smooth objectives with convex constraints}
\label{sec:sub_0_eqs_pre}

In this section, we establish the following result.
\begin{proposition}
   \label{prop:sub_0_eqs_pre}
   Let the weight space $\Theta$ be a convex set.
   Assume that, for any $\theta \in \Theta$, the function $\theta^\top f(x, s)$ is $\mu$-strongly convex and differentiable in $x$, and that, for every $n$, there exists $x^{(n)}\in X(s^{(n)})$ satisfying $a^{(n)}=f(x^{(n)}, s^{(n)})$.
   Then, minimizing the suboptimality loss is equivalent to minimizing the PLF; namely,
   \begin{equation*}
      \ell_{\mathrm{sub}}(\theta)=0
      \Leftrightarrow \ell_{\mathrm{plf}}(\theta)=0
      \Leftrightarrow
      \forall n,\ a^{(n)}=a^*(\theta,s^{(n)}) \, .
   \end{equation*}
\end{proposition}

Before proving the above, we recall the following known proposition.
\begin{proposition}
  \label{prop:mohajerin-esfahani}
  \textrm{\citet[Proposition 2.5]{Mohajerin-2018-Data}}
  Let the weight space $\Theta$ be a convex set.
  Assume that, for any $\theta \in \Theta$, the function $\theta^\top f(x, s)$ is $\mu$-strongly convex and differentiable in $x$, and that, for every $n$, there exists $x^{(n)}\in X(s^{(n)})$ satisfying $a^{(n)}=f(x^{(n)}, s^{(n)})$.
  Then,
  \[
    \ell_{\mathrm{sub}}(\theta)\geq \frac{\mu}{2N}\sum_{n=1}^N \left\|x^*(\theta,s^{(n)})-x^{(n)}\right\|^2
  \]
  holds.
\end{proposition}

{
\renewcommand{\proofname}{Proof of \cref{prop:sub_0_eqs_pre}}%
\begin{proof}
Assuming that the mapping $f$ is $L(f)$-Lipschitz continuous, we obtain
\begin{align}
    L(f)^2 \ell_{\mathrm{sub}}(\theta)
    &\geq L(f)^2 \frac{\mu}{2N}\sum_{n=1}^N \left\|x^*(\theta,s^{(n)})-x^{(n)}\right\|^2
    \notag\\
    &\geq \frac{\mu}{2N}\sum_{n=1}^N \left\|a^*(\theta,s^{(n)})-a^{(n)}\right\|^2
    = \frac{\mu}{2}\ell_{\mathrm{plf}}(\theta)
    \notag
\end{align}
On the other hand, by the Cauchy--Schwarz inequality,
\[
    \ell_{\mathrm{sub}}(\theta)
    \leq \sup_{\theta\in\Theta}\|\theta\|\,\ell_{\mathrm{plf}}(\theta)^{\nicefrac{1}{2}} \, .
\]
Combining the above, we conclude that
\[
  \ell_{\mathrm{sub}}(\theta)=0
  \Leftrightarrow \ell_{\mathrm{plf}}(\theta)=0
  \Leftrightarrow
  \forall n,\ a^{(n)}=a^*(\theta,s^{(n)})
\]
holds.
\end{proof}
}

\section{Projected gradient descent when the objective is strongly convex and smooth and the constraints are convex}
\label{sec:subopt_PSGD_to_pre}

In this section, we establish the following result.
\begin{proposition}
    \label{prop:subopt_PSGD_to_pre}
    Let the weight space $\Theta$ be a convex set.
    Assume that, for any $\theta\in\Theta$, the function $\theta^\top f(x, s)$ is $\mu$-strongly convex and differentiable in $x$, and that, for every $n$, there exists $x^{(n)}\in X(s^{(n)})$ satisfying $a^{(n)}=f(x^{(n)}, s^{(n)})$.
    Then, if one runs projected gradient descent on the iterate sequence $\{\theta^{t}\}$ with step size $\alpha_t=1/\beta$, the PLF converges linearly; namely,
    \begin{equation*}
        \ell_{\mathrm{plf}}(\theta^{t})
        \leq \frac{2\ell_{\mathrm{sub}}(\theta^1)}{\mu}\exp\left(-\frac{\mu}{4\beta}(t-1)\right) \, .
    \end{equation*}
\end{proposition}

\begin{proof}
By \citet[Theorem 2.6]{hazan2022introduction}, $\ell_{\mathrm{sub}}(\theta^{t})$ admits the bound
\[
    \ell_{\mathrm{sub}}(\theta^{t})
    \leq \ell_{\mathrm{sub}}(\theta^1)\exp\left(-\frac{\mu}{4\beta}(t-1)\right) \, .
\]
By \cref{prop:mohajerin-esfahani}, it follows that
\begin{equation*}
    \ell_{\mathrm{plf}}(\theta^{t})
    \leq \frac{2}{\mu}\ell_{\mathrm{sub}}(\theta^{t})
    \leq \frac{2\ell_{\mathrm{sub}}(\theta^1)}{\mu}\exp\left(-\frac{\mu}{4\beta}(t-1)\right) \, .
\end{equation*}
\end{proof}

\section{Regret analysis for the suboptimality loss}
\label{sec:known_regret_analysis}

Let $\Theta$ be a nonempty subset of $\mathbb{R}^d$.
Let $\ell\colon\Theta\to\mathbb{R}$ be a loss function.
Consider a sequence $\{\theta^t\}_{t=1}^T\subset\Theta$.
We define the regret by
\[
    \mathrm{Regret}(T)
    = \sum_{t=1}^T \left(\ell(\theta^{t}) - \min_{\theta\in\Theta}\ell(\theta)\right) \, .
\]
In what follows, we review existing work on regret analyses for the suboptimality loss.
The results are summarized in \cref{tab:pro_con_SL_regret_detail}.

\begin{table}[ht]
\vspace{-\intextsep} 
    \caption{Performance comparison of methods for solving the DDIOP for MILPs (\cref{eq:IOP_linear}).
    The integer $T$ denotes the number of iterations or the number of optimization calls.
    The integer $t$ denotes the iteration index or the optimization-call index.}
    \label{tab:pro_con_SL_regret_detail}
    \centering
    \begin{tabular}{p{3cm}|p{9cm}}
        \toprule
        Method & Regret for the suboptimality loss \\
        \midrule
        MWU & $O \left( L(\ell_{\mathrm{sub}}) \log d  \sqrt{T} \right)$, provided that $\Theta = \Delta^{d-1}$ \\
        \citep{arora2012multiplicative} & \citep[Theorem 3.5]{Barmann-2018-online} \\
        \hline
        PSGD (online gradient descent) & $O \left( \mathrm{diam} (\Theta) L(\ell_{\mathrm{sub}}) \sqrt{T} \right)$ \\
        (with step size & \citep[Theorem 3.11]{Barmann-2018-online} \\
        $\mathrm{diam} (\Theta) L(\ell_{\mathrm{sub}})^{-1} t^{-1/2}$) &  \\
        \hline
        \citet{besbes2021online,besbes2025contextual} & $O \left( d^4 \log T  \right)$, provided that $\Theta$ is the unit sphere and $L(\ell_{\mathrm{sub}})\leq 1$\\
        \hline
        \citet{gollapudi2021contextual} & $O \left( d \log T \right)$, provided that $\Theta$ is the unit ball and $L(\ell_{\mathrm{sub}})\leq 1$ \\
         & $O \left( d^{2(d+1)} \right)$, provided that $\Theta$ is the unit ball and $L(\ell_{\mathrm{sub}})\leq 1$ \\
        \hline
        ONS & $O \left(\mathrm{diam} (\Theta) L(\ell_{\mathrm{sub}}) d\log (T/d) \right)$\\
        \citep{hazan2007logarithmic} &  \citep[Theorem 3.1]{sakaue2025online} \\
        \hline
        MetaGrad & $O \left( \mathrm{diam} (\Theta) L(\ell_{\mathrm{sub}}) d\log (T/d) \right)$ \\
        (\citealp{van2016metagrad}; & \citep[Theorem 4.1]{sakaue2025online} \\
        \citealp{van2021metagrad}) & \\
        \bottomrule
        Cf.\ lower bound & $\Omega (d)$ \citep[\S 5]{sakaue2025online}\\
    \end{tabular}
    \vspace{-\intextsep} 
\end{table}

\section{Online and offline optimization}

Let $\Theta$ be a nonempty subset of $\mathbb{R}^d$.
Let $\ell\colon\Theta\to\mathbb{R}$ be a function.
Consider a sequence $\{\theta^t\}_{t=1}^T\subset\Theta$.
Then, the following holds:
\begin{equation}
    \min_{t=1,\ldots,T}
    \left(\ell(\theta^{t}) - \min_{\Theta}\ell\right)
    \leq
    \frac{1}{T}
    \sum_{t=1}^T \left(
        \ell(\theta^t) - \min_{\Theta}\ell
    \right)
    \leq \frac{\mathrm{Regret}(T)}{T}
    \, .
    \label{eq:offline-online-estimate}
\end{equation}

By applying the existing regret bounds reviewed in \cref{sec:known_regret_analysis} to \cref{eq:offline-online-estimate}, one can upper-bound the best-iterate performance for the suboptimality loss.
For the resulting bounds, see \cref{tab:pro_con_SL_detail}.

\section{Numerical Experiment}
\label{sec:experiment}

In this section, under \cref{assu:WIRL}, we report numerical results comparing \cref{alg:intention-WIRL-gradual-decay} with existing methods.
Implementation details and the computational device used are provided in \cref{sec:devices}.

When the objective function is linear and $X(s)$ is a convex region, it is known that one can solve formulation (15) of \citet{chan2023inverse} by taking the number of point clouds $T$ uniformly.
We refer to this method as CHAN.

\begin{wraptable}{r}{0.55\textwidth}
    \vspace{-2\intextsep} 
    \caption{Computational time per iteration, or equivalently, the computational cost required for a single call to the forward problem, for each method when $\Theta=\Delta^{d-1}$.
    The constant $\mathrm{Opt}$ denotes the computational complexity of running \cref{eq:FOP_linear}.
    }
    \label{tab:wall-time}
    \centering
    \begin{tabular}{ll}
        \toprule
        Method & Computational complexity \\
        \midrule
        UPA & $N \mathrm{Opt}$ \\
        RPA & $N \mathrm{Opt}$ \\
        CHAN & at least $N \mathrm{Opt}$ \\
        PSGD & $N \mathrm{Opt} + O(d \log d )$ \\
        \bottomrule
    \end{tabular}
    \vspace{-\intextsep} 
\end{wraptable}

When $\Theta=\Delta^{d-1}$, the computational time per iteration, or equivalently, the computational cost required for a single call to the forward problem, is summarized in \cref{tab:wall-time}.
The computational complexity of one iteration of \cref{alg:intention-WIRL-gradual-decay} is the sum of the complexity of \cref{eq:FOP_linear} and the complexity of the projection $\Proj_{\Delta^{d-1}}$, namely $O(d \log d)$ \citep{Wang-13-projection}.
Since the present experiments have dimension $d \leq 8$, we may regard the runtime of methods other than CHAN as comparable to, or faster than, that of CHAN.

\subsection{Definitions of UPA and RPA}
\label{sec:def:UPA-RPA-CHAN}

We briefly review UPA and RPA.
For a subset $\Theta^\prime \subset \Delta^{d-1}$, define
\[
    \psi(\Theta^\prime):=\argmin_{\Theta^\prime}\ell_{\mathrm{plf}} \, .
\]
In UPA, we define
\[
    \Theta_k :=
    \left\{
        \left(
            \frac{2k_1+1}{2k+d},
            \ldots,
            \frac{2k_d+1}{2k+d}
        \right)
        \,\middle|\,
        \forall i,\ k_i \in \mathbb{Z}_{\geq 0},\ \sum_i k_i = k
    \right\} \, .
\]
UPA returns $\psi(\Theta_k)$.
In RPA, letting $\Theta_{\mathrm{rand}}$ be a set of $k$ random points sampled from $\Delta^{d-1}$, RPA returns $\psi(\Theta_{\mathrm{rand}})$.

\subsection{Linear programming}
\label{sec:exp-LP}

\begin{figure*}[t]
    \centering
    \begin{minipage}[ht]{0.3\linewidth}
        \includegraphics[width = \linewidth ]{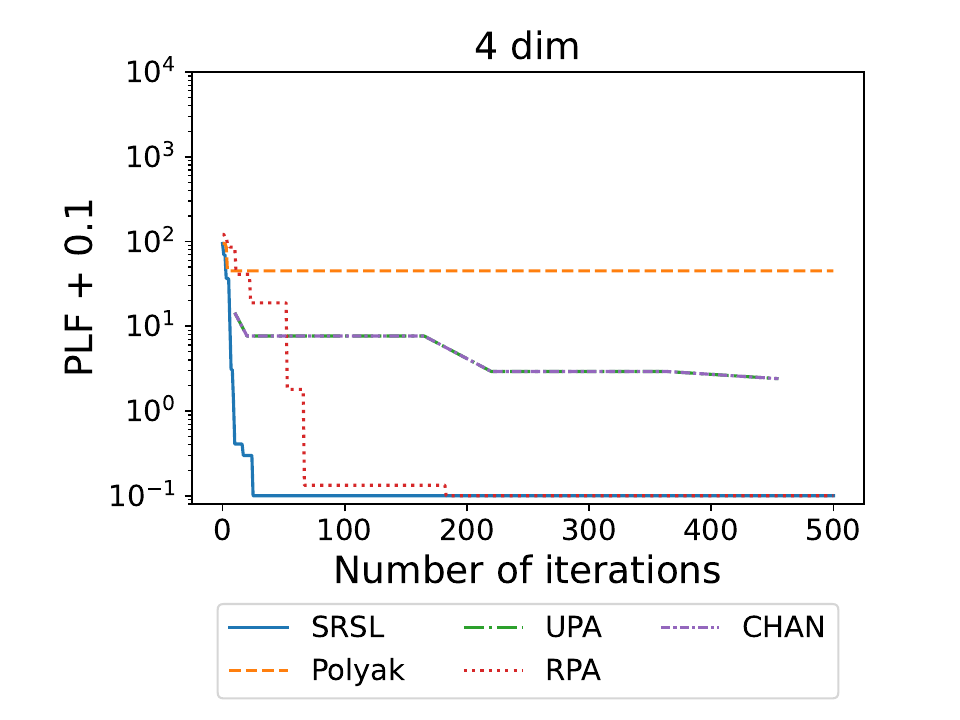}
    \end{minipage}
    \begin{minipage}[ht]{0.3\linewidth}
        \includegraphics[width = \linewidth ]{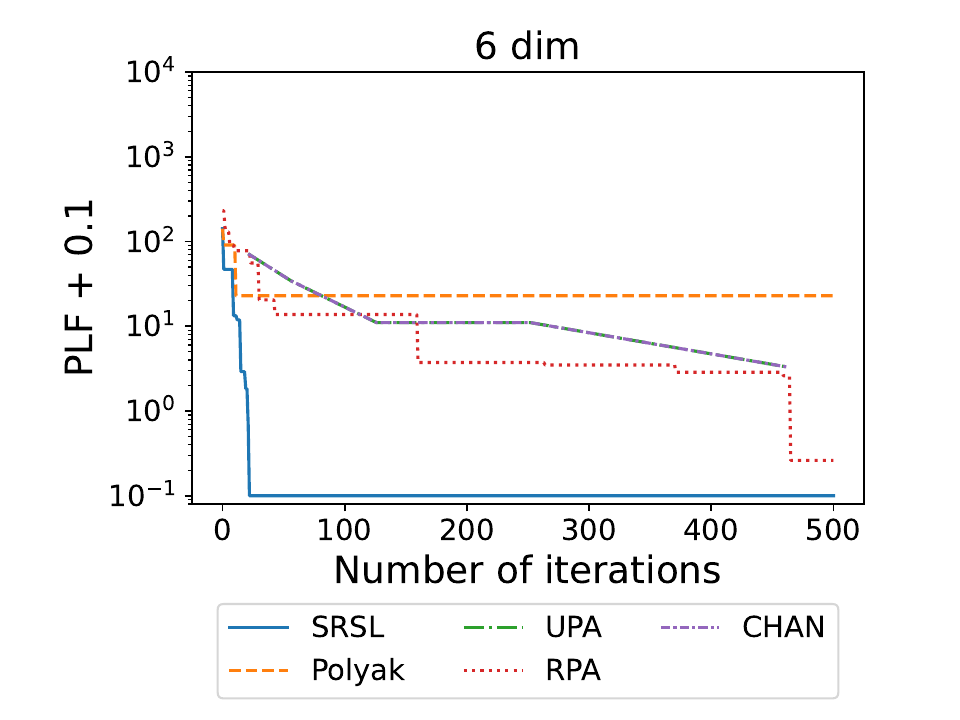}
    \end{minipage}
    \begin{minipage}[ht]{0.3\linewidth}
        \includegraphics[width = \linewidth ]{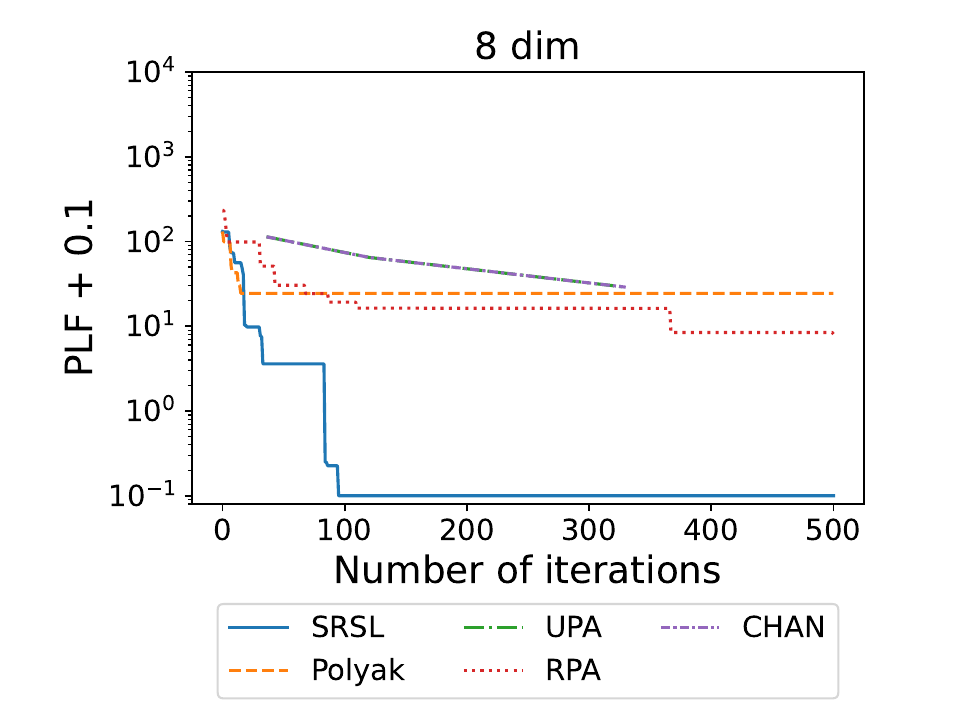}
    \end{minipage}
    \caption{Worst-case behavior of $\ell_{\mathrm{plf}} (\theta^{t} (\mathcal{D})) +0.1$ for the LP setting.}
    \label{pic:eliptic_LP-1-PLA-100}
\end{figure*}

The LP setup is described in detail in \cref{sec:exp-LP-appendix}.
Under this setup, we compare PSGD (SRSL), PSGD (Polyak), UPA, RPA, and CHAN.

Each experiment consists of the following five steps.
(Step 1) Sample $\theta^* \in \Delta^{d-1}$ according to the uniform distribution on $\Delta^{d-1}$.
(Step 2) Sample $s^{(n)} \in \mathcal{S}$ according to the random variable specified in \cref{sec:exp-LP-appendix}.
(Step 3) For $n=1,\ldots,N$, set $a^{(n)}=a^*(\theta^*,s^{(n)})$.
(Step 4) Run PSGD (SRSL), PSGD (Polyak), UPA, RPA, and CHAN with the number of iterations (or the number of acquired data points) set to $500$.
(Step 5) Repeat (Step 1)--(Step 4) $100$ times.

\begin{remark}
    UPA and CHAN cannot be run for an arbitrary number of iterations: the number of evaluation points can only take the discrete values $|\Theta_j|$ indexed by an integer $j$. We therefore ran these methods only at the values $t = |\Theta_j|$ with $t < 500$, and linearly interpolated the loss values at the remaining values of $t$.
\end{remark}

We summarize the results as follows.
\cref{pic:eliptic_LP-1-PLA-100} shows, for $d=4,6,8$, the worst-case trajectories of $\ell_{\mathrm{plf}}(\theta^{t}(\mathcal{D}))+0.1$.
For the PLF minimization problem, \cref{pic:eliptic_LP-1-PLA-100} suggests that PSGD (SRSL) is the most effective method for reducing the PLF, achieving PLF minimization in fewer than $1/7$ of the iterations required by UPA, RPA, and CHAN.
In particular, when $d \geq 6$, once training with PSGD (SRSL) is complete, PSGD (SRSL) improves the PLF by more than two orders of magnitude compared with the existing algorithms.

By \cref{pic:eliptic_LP-1-PLA-100}, as the dimension $d$ increases, UPA, RPA, and CHAN are expected to move toward the upper right.
In contrast, we observe that PSGD (SRSL) is less sensitive to the dimension than the other methods.

Moreover, PSGD (SRSL) drives the PLF to $0$ with probability $100\%$.
This phenomenon can be theoretically justified by \cref{theo:universal_grad_based_opt_achieve_min_SL}.

\begin{remark}
    \label{rem:behavior-PSGD-Polyak-LP}
    The behavior of the suboptimality loss is described in detail in \cref{pic:eliptic_LP-1-SL-100}.
    From a theoretical viewpoint, PSGD (SRSS, SRSL) is the most effective method for reducing the suboptimality loss (see \cref{rem:lower_bound_first_order_method}).
    However, in our experiments, we found that both PSGD (SRSL) and PSGD (Polyak) are efficient.

    It remains necessary to investigate why PSGD (Polyak) fails to attain the minimum value of the PLF.
    Under PSGD (Polyak), the suboptimality loss value becomes extremely small, and consequently the update step size also becomes extremely small.
    As a result, the update magnitude becomes negligible, preventing the method from attaining the PLF minimum value $0$.
\end{remark}

\subsection{Single-machine weighted sum of completion times scheduling}
\label{sec:1-machine}

\begin{figure*}[ht]
    \centering
    \begin{minipage}[ht]{0.3\linewidth}
        \includegraphics[width=\linewidth ]{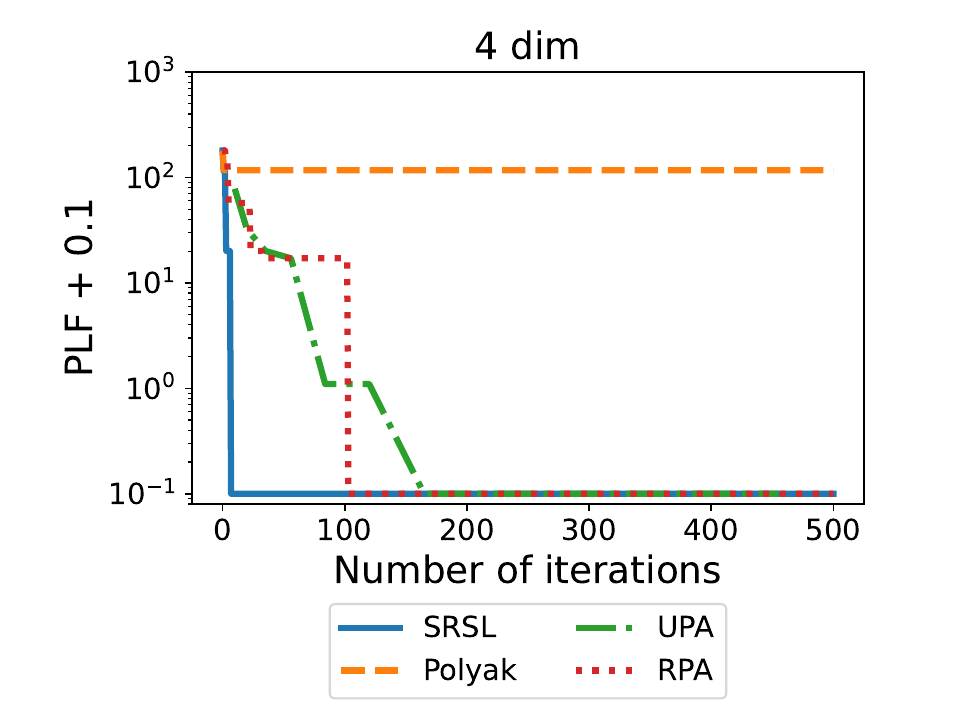}
    \end{minipage}
    \begin{minipage}[ht]{0.3\linewidth}
        \includegraphics[width=\linewidth ]{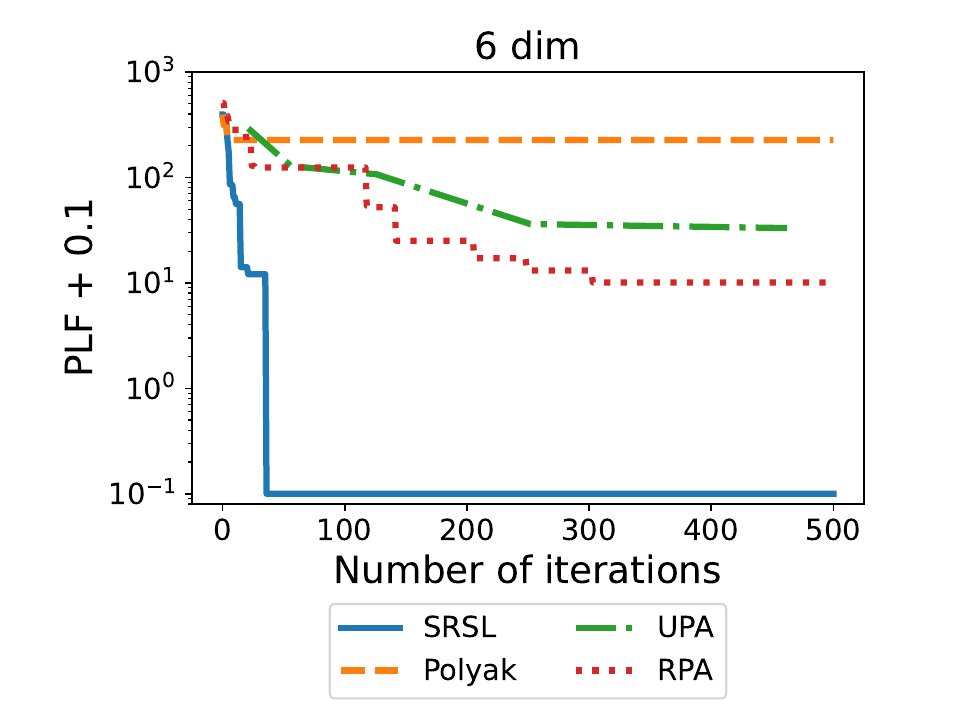}
    \end{minipage}
    \begin{minipage}[ht]{0.3\linewidth}
        \includegraphics[width=\linewidth ]{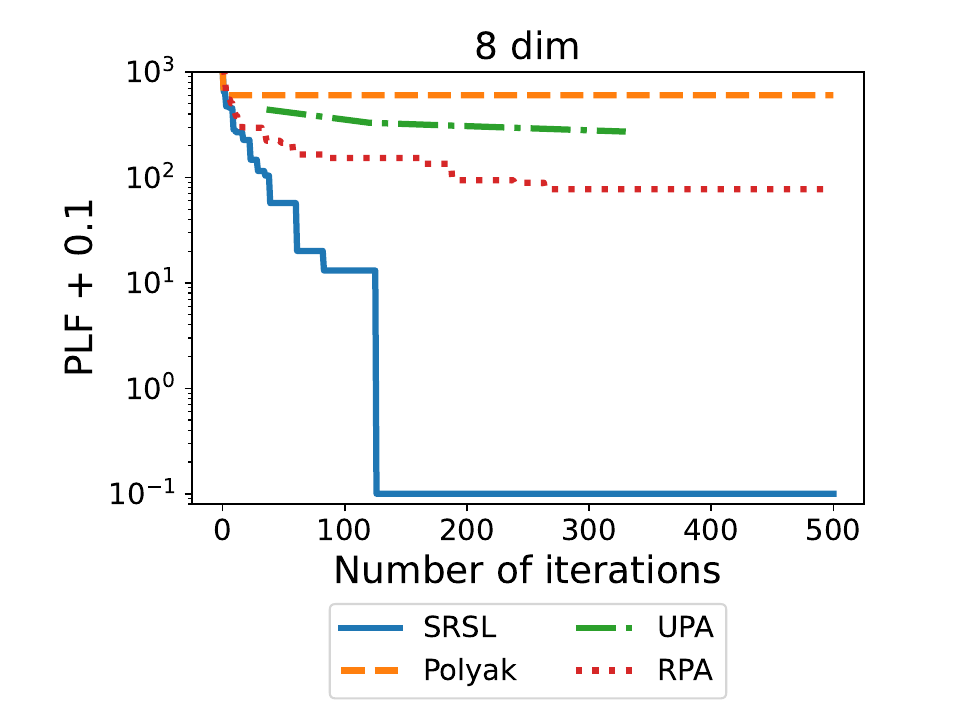}
    \end{minipage}
    \caption{Worst-case behavior of $\ell_{\mathrm{plf}} (\theta^{t} (\mathcal{D})) +0.1$ for the single-machine scheduling instance.}
    \label{pic:scheduling-1-PLA-100}
\end{figure*}

The single-machine weighted sum of completion times scheduling problem $1|r_j|\sum \theta_j C_j$ is an ILP.
The experimental setup for $1|r_j|\sum \theta_j C_j$ is given in \cref{sec:1-machine-appendix}.

We conducted experiments following the same steps (Step 1)--(Step 5) as in \cref{sec:exp-LP}, with $t \leq 1000$ and excluding CHAN.
We summarize the results as follows.
\cref{pic:scheduling-1-PLA-100} reports, for $d=4,6,8$, the worst-case trajectories of $\ell_{\mathrm{plf}}(\theta^{t}(\mathcal{D}))+0.1$.
As in \cref{sec:exp-LP}, we confirm that PSGD (SRSL) exhibits weaker dependence on the dimension than existing methods and minimizes the PLF in fewer than $1/10$ of the iterations required by the other methods.
The runtimes of each method are reported in \cref{tab:schedulingd_4,tab:schedulingd_6,tab:schedulingd_8}.
PSGD (SRSL) terminates within $2.27$ seconds for $d=4,6,8$.

\begin{remark}
    \label{rem:behavior-PSGD-Polyak-scheduling}
    The behavior of the suboptimality loss is shown in \cref{pic:scheduling-1-SL-100}.
\end{remark}

\section{Conclusion}
\label{sec:conclusion}

To solve the DDIOP for MILPs, we focused on the suboptimality loss, which is Lipschitz continuous and convex.
By leveraging its convex piecewise-linear structure and the existence of a relative interior point in its minimizer set, we showed that a broad class of gradient-based optimization methods, including PSGD, attains the minimum value $0$ of the suboptimality loss in finitely many iterations, and hence solves the DDIOP for MILPs exactly.
This finite-step attainment rests on a general principle for convex piecewise-linear functions---if the minimizer set has a relative interior point, a broad class of first-order methods, including projected subgradient descent, reaches the minimum in finitely many iterations---which we applied by showing that the suboptimality loss satisfies this condition.
As a further consequence, we proved that PSGD attains the minimum value of the PLF in finitely many iterations.
We also derived an explicit upper bound on the number of iterations required for PSGD to reach the minimum in finite time, and verified the finite-time attainment behavior through numerical experiments.

Several directions for future work are as follows. One is to develop algorithms that can efficiently learn inverse optimization problems when the given data contain noise. Real-world data tend to be noisy~\citep[cf.][]{aswani2018inverse}. Establishing methods that simultaneously achieve computational efficiency and robustness to noise is expected to broaden the applicability to practical optimization problems.

Another direction is whether the finite-step attainment established in the offline setting ($\ell_{\mathrm{sub}}=0$) also holds in online inverse optimization---that is, whether, for sequentially arriving data, the predictions can become exactly consistent with the observations after some finite round, incurring no further loss thereafter. For online inverse optimization, a finite cumulative regret independent of the number of iterations is already known~\citep{gollapudi2021contextual}. However, finiteness of the cumulative regret is a guarantee weaker than, and different from, achieving exact consistency in finitely many rounds (zero loss thereafter). The relationship between the two, and a characterization of what obstructs finite-step attainment of exact consistency (such as the vanishing of the thickness $\gamma(\ell_{\mathrm{sub}})$ of the consistent set), remain open.

\section*{Acknowledgement }
    We would like to thank Riki Eto, Kei Takemura, Yuzuru Okajima, and Yoichi Sasaki for carefully reviewing this paper.
    We also thank GPT-5.2, GPT-5.4, Opus 4.7, and Opus 4.8, Fable 5 for their assistance with proofreading the manuscript.

\bibliographystyle{apalike} 
\bibliography{suboptimality_loss.bib} 


\newpage

\appendix
\crefalias{section}{appendix}

\section{Experiments}
\label{sec:exp_appendix}

\subsection{Implementation and device details}
\label{sec:devices}

The primary libraries used in our experiments are OR-Tools v9.8 \citep{ortools}, NumPy 1.26.3 \citep{harris2020array}, and Python 3.9.0 \citep{Python3}.
All computations were conducted on a machine equipped with 192 Intel CPU cores and 1.0~TB of RAM.

\subsection{Linear programming}
\label{sec:exp-LP-appendix}

\paragraph{Theoretical setup}

We set $d \in \{4,6,8\}$, $N=1$, $J^{(n)}=100$, and $r_{\max}=10$.
For each $n=1,\ldots,N$, we draw $\log_{0.1} r_i^{(n)}$ i.i.d.\ from the uniform distribution on $[0,\log_{0.1}(1/r_{\max})]$.
For each $n=1,\ldots,N$ and $j=1,\ldots,J^{(n)}$, we draw $b^{(n,j)}$ i.i.d.\ from the uniform distribution on $\mathbb{R}_{\geq 0}^d$.
We then rescale each $b^{(n,j)}$ by a positive scalar so that $\sum_{i=1}^d r_i^{(n)2} b_i^{(n,j)2}=1$.
Under this setting, we consider the following LP:
\begin{equation}
    \begin{split}
        \text{maximize}_{x}\,
        & \theta^{\top} x
        \\
        \text{subject to}\,
        &
        \sum_{i=1}^d r_i^{(n)2} b_i^{(n,j)} x_i \leq 1,
        \\
        & \forall j=1,\ldots,J^{(n)},
        \,
        x \geq 0 .
    \end{split}
    \label{opt:LP-elliptic}
\end{equation}
Let $\mathcal{S}$ denote the set of pairs $(r,b)$.

\paragraph{CHAN setup}

We apply the discussion in \citet[(15)]{chan2023inverse} to \cref{opt:LP-elliptic}.
We rewrite \cref{opt:LP-elliptic} as
\begin{equation*}
    \begin{split}
        \text{maximize}_{x}\,
        & \theta^{\top} x
        \\
        \text{subject to}\,
        &
        A^{(n)} x \leq 1_{J^{(n)}} := (\underbrace{1,\ldots,1}_{J^{(n)}\text{ times}}),
        \quad
        x \geq 0 .
    \end{split}
\end{equation*}
Applying the discussion in \citet[(15)]{chan2023inverse} to \cref{opt:LP-elliptic}, we obtain
\begin{equation}
    \begin{split}
        \text{minimize}_{x,\theta\in\Theta_k}\,
        & \sum_{n=1}^N \|x^{(n)}-a^{(n)}\|_2^2
        \\
        \text{subject to}\,
        &
        A^{(n)\top}\lambda^{(n)} \geq \theta,
        \,
        \lambda^{(n)} \geq 0,
        \,
        \theta^{\top} x^{(n)} = 1_{J^{(n)}}^{\top}\lambda^{(n)},
        \\
        &
        A^{(n)} x^{(n)} \leq 1_{J^{(n)}},
        \,
        x^{(n)} \geq 0,
        \,
        \forall n=1,\ldots,N
        \, .
    \end{split}
    \label{opt:LP-elliptic-chan}
\end{equation}

\paragraph{Results}

We report the experimental results here.
\cref{pic:eliptic_LP-1-SL-100} shows, for $d=4,6,8$, the worst-case behavior of $\ell_{\mathrm{sub}}(\theta^{t}(\mathcal{D}))+0.001$, respectively.

\begin{figure}[ht]
    \vskip 0.2in
    \begin{center}
        \begin{minipage}[ht]{0.3\linewidth}
            \centerline{
                \includegraphics[width=\columnwidth]{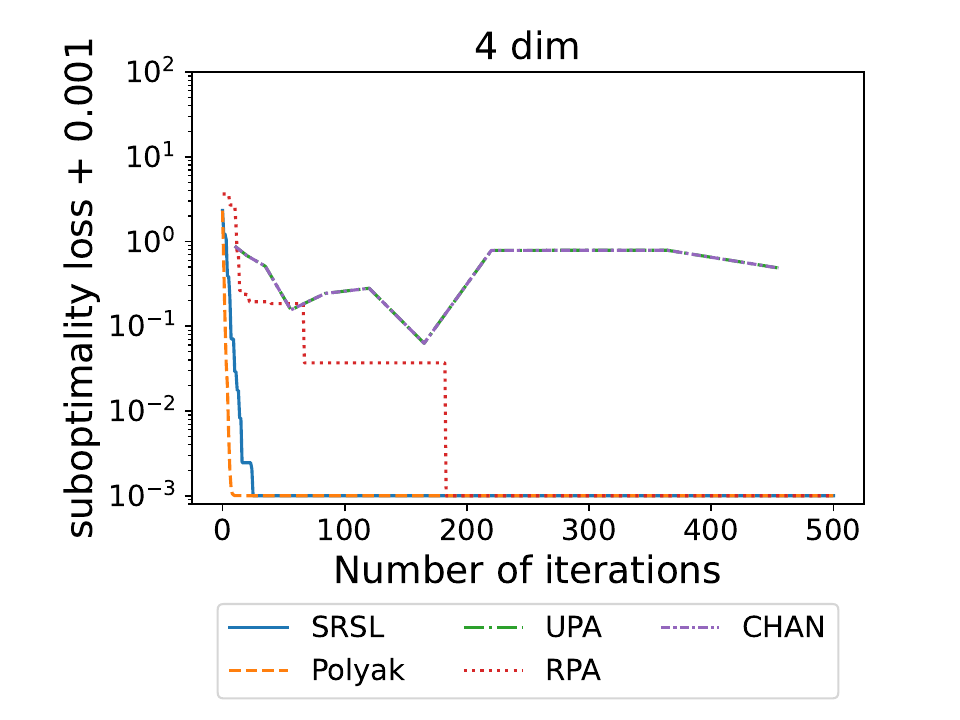}
            }
        \end{minipage}
        \begin{minipage}[ht]{0.3\linewidth}
            \centerline{
                \includegraphics[width=\columnwidth]{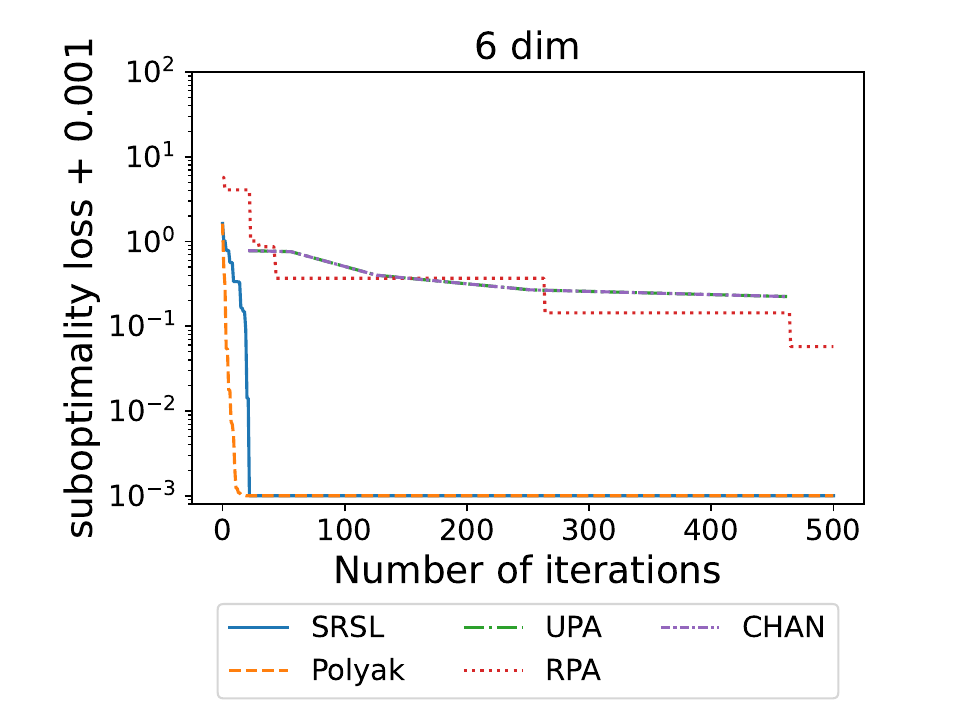}
            }
        \end{minipage}
        \begin{minipage}[h]{0.3\linewidth}
            \centerline{
                \includegraphics[width=\columnwidth]{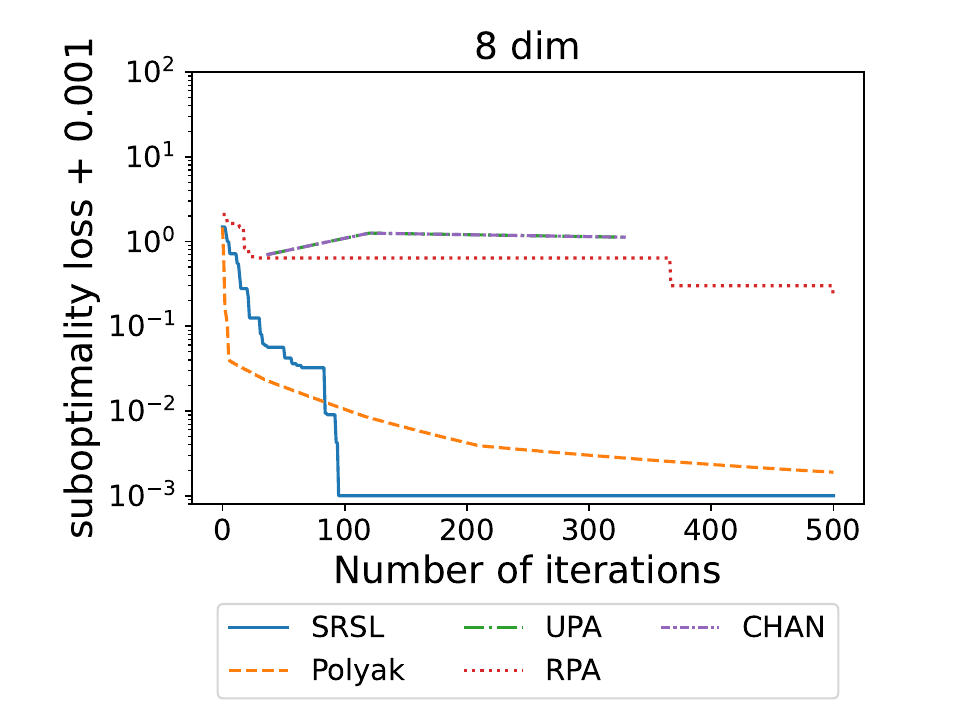}
            }
        \end{minipage}
    \end{center}
    \caption{Worst-case behavior of $\ell_{\mathrm{sub}} (\theta^{t} (\mathcal{D})) + 0.001$ for the LP setting.}
    \label{pic:eliptic_LP-1-SL-100}
\end{figure}

\subsection{Single-machine weighted completion-time scheduling}
\label{sec:1-machine-appendix}

\paragraph{Setup}

We consider the single-machine weighted completion-time scheduling problem $1\mid r_j\mid \sum_{j} \theta_j C_j$, in which $d$ jobs are to be processed on a single machine.
We assume that the machine can process at most one job at a time and that processing is non-preemptive.
Let $j=1,\ldots,d$ index the jobs, and let $J=\{1,\ldots,d\}$ denote the job set.
For each job $j$, let $p_j$ be its processing time, $\theta_j$ its importance weight, and $r_j$ its release time (i.e., the earliest time at which job $j$ can start).
The goal is to find an execution order (schedule) on the machine so as to minimize the weighted sum of completion times, where $C_j$ denotes the completion time of job $j$.

\begin{table}[ht]
    \vspace{-\intextsep}
    \caption{Decision variables and constraints in the scheduling formulation.}
    \label{tab:scheduling_variable}
    \centering
    \begin{tabular}{l|rrr}
        $d$ & 4 & 6 & 8 \\
        \hline
        Decision variables & 16 & 36 & 64 \\
        Constraints & 28 & 66 & 120 \\
    \end{tabular}
\end{table}

Let the continuous variable $b_j$ denote the start time of job $j$, and let $x_{jk}$ be a binary variable that equals $1$ if job $j$ precedes job $k$, and equals $0$ otherwise.
Define $M(p,r):=\max_{j} r_j + \sum_{j} p_j$.
Then, the problem can be formulated as follows:
\begin{align*}
    \text{minimize}_{b,x}\,
    &
    \sum_{j\in J} \theta_j (b_j+p_j)
    \\
    \text{subject to}\,
    &
    b_j+p_j-M(p,r)(1-x_{jk}) \leq b_k, \quad \forall j\neq k, \\
    &
    x_{jk}+x_{kj}=1,\ \ x_{jk}\in\{0,1\}, \quad \forall j\neq k, \\
    &
    b_j \geq r_j,\ \ b_j\in\mathbb{R}, \quad \forall j.
\end{align*}
The numbers of decision variables and constraints are summarized in \cref{tab:scheduling_variable}.

We set
\[
    \Theta
    =
    \left\{
        \theta\in\mathbb{R}^d
        \,\middle|\,
        \theta_j \geq 0,
        \,
        \sum_{j\in J}\theta_j = 1
    \right\}
    +10^{-3}(1,\ldots,1)
    \, .
\]
We sample $r_j$ i.i.d.\ from the uniform distribution on $[0,10]$ and $p_j$ i.i.d.\ from the uniform distribution on $[1,5]$.
Let $\mathcal{S}$ denote the set of pairs $s=(p,r)$, and let $X(s)$ denote the feasible set of $(b,x)$ satisfying the constraints.
We set $N=1$.
We run \cref{alg:intention-WIRL-gradual-decay} on $\Theta$ in place of $\Delta^{d-1}$.

\paragraph{Results}

We report the experimental results here.
\cref{pic:scheduling-1-SL-100} shows, for $d=4,6,8$, the worst-case behavior of $\ell_{\mathrm{sub}} (\theta^{t} (\mathcal{D})) + 0.001$, respectively.

\begin{figure}[ht]
    \vskip 0.2in
    \begin{center}
        \begin{minipage}[ht]{0.3\linewidth}
            \centerline{
                \includegraphics[width=\columnwidth]{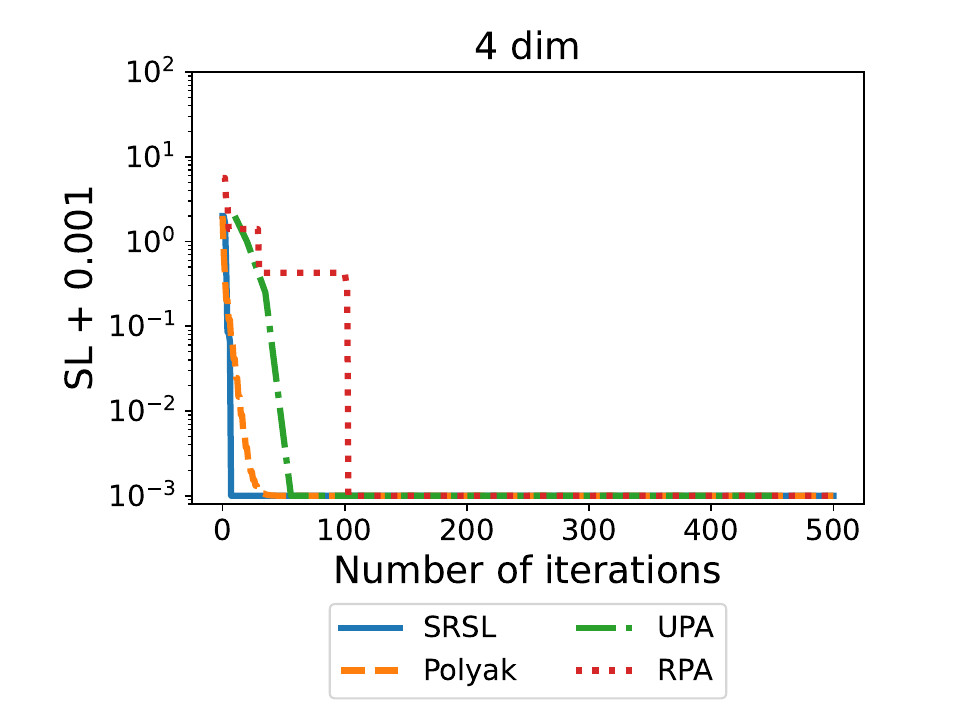}
            }
        \end{minipage}
        \begin{minipage}[ht]{0.3\linewidth}
            \centerline{
                \includegraphics[width=\columnwidth]{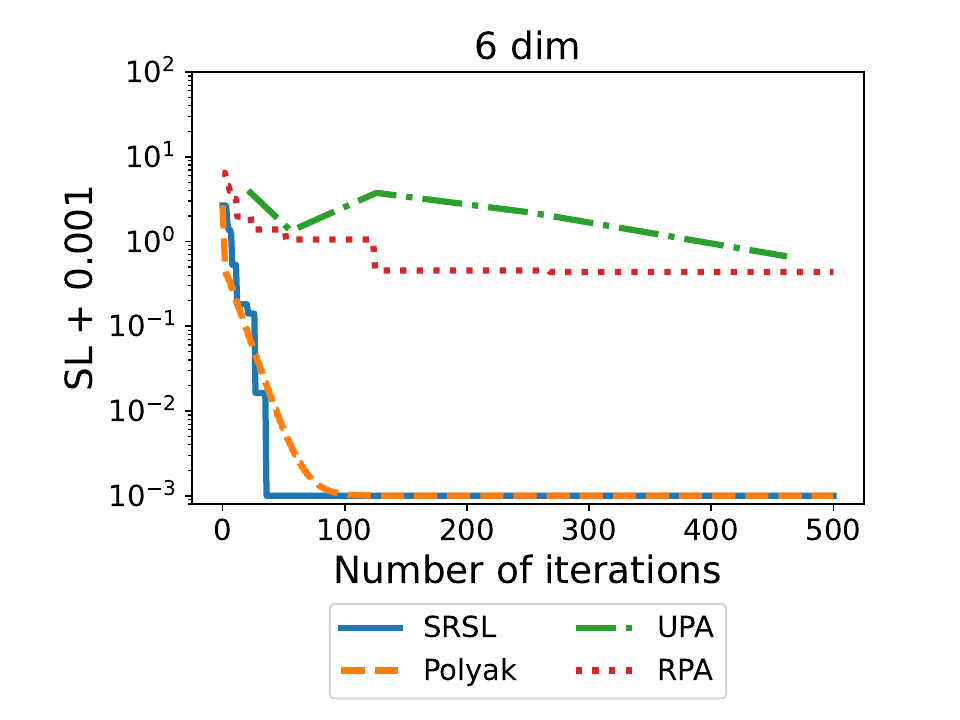}
            }
        \end{minipage}
        \begin{minipage}[ht]{0.3\linewidth}
            \centerline{
                \includegraphics[width=\columnwidth]{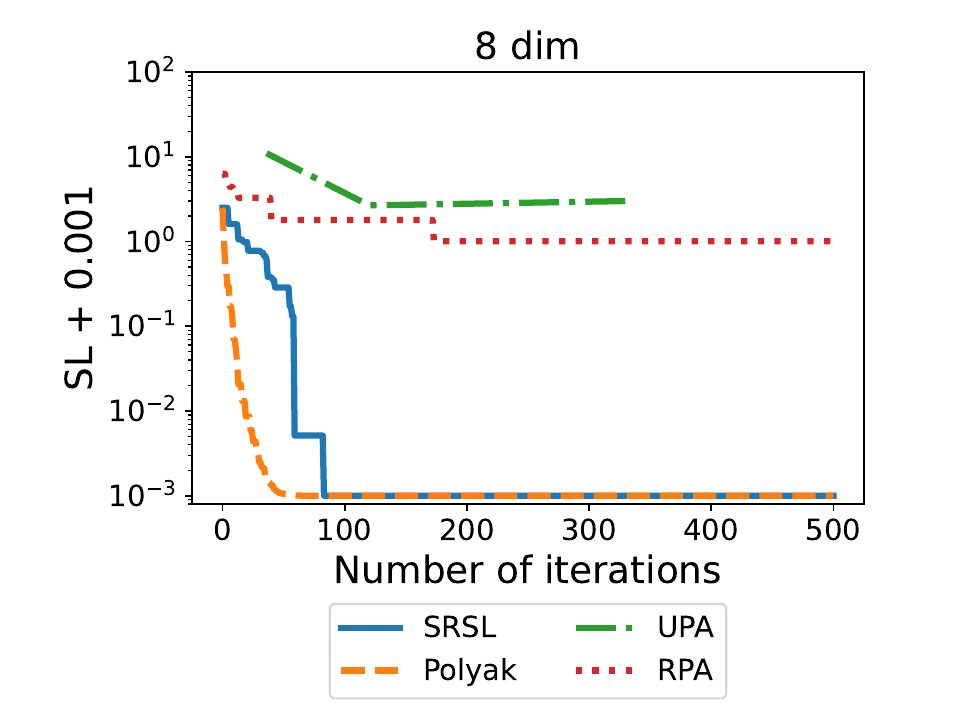}
            }
        \end{minipage}
    \end{center}
    \caption{
        Worst-case behavior of $\ell_{\mathrm{sub}} (\theta^{t} (\mathcal{D})) + 0.001$ for the single-machine scheduling problem.
    }
    \label{pic:scheduling-1-SL-100}
\end{figure}

\begin{table}[ht]
  \centering
  \begin{minipage}{\linewidth}

    \caption{Runtime of each method for $d=4$ jobs.}
    \label{tab:schedulingd_4}
    \centering
    \begin{tabular}{lrrrr}
        \toprule
        & PSGD & PSGD & UPA & RPA \\
        & (SRSL) & (Polyak) &  &  \\
        \midrule
        Mean (s)    & 0.010 & 3.70   & 9.12  & 2.06 \\
        Max (s)     & 0.024 & 103.49 & 12.17 & 2.86 \\
        Median (s)  & 0.009 & 0.024  & 8.74  & 1.98 \\
        \bottomrule
    \end{tabular}

    \vspace{2.0ex}

    \caption{Runtime of each method for $d=6$ jobs.}
    \label{tab:schedulingd_6}
    \centering
    \begin{tabular}{lrrrr}
        \toprule
        & PSGD & PSGD & UPA & RPA \\
        & (SRSL) & (Polyak) &  &  \\
        \midrule
        Mean (s)    & 0.04 & 18.30  & 4.89 & 4.61 \\
        Max (s)     & 1.14 & 132.86 & 6.06 & 11.15 \\
        Median (s)  & 0.02 & 0.27   & 4.84 & 4.19 \\
        \bottomrule
    \end{tabular}

    \vspace{2.0ex}

    \caption{Runtime of each method for $d=8$ jobs.}
    \label{tab:schedulingd_8}
    \centering
    \begin{tabular}{lrrrr}
        \toprule
        & PSGD & PSGD & UPA & RPA \\
        & (SRSL) & (Polyak) &  &  \\
        \midrule
        Mean (s)    & 0.382 & 86.72   & 7.52  & 33.83 \\
        Max (s)     & 2.27  & 1354.91 & 8.61  & 93.85 \\
        Median (s)  & 0.211 & 12.25   & 7.51  & 29.24 \\
        \bottomrule
    \end{tabular}

  \end{minipage}
  \vspace{-\intextsep}
\end{table}

\end{document}

%% file: IOP_image.tex
\begin{wrapfigure}{r}{6.5cm}
\vspace{-\intextsep}
\begin{center}
\begin{tikzpicture}[
    scale=0.75,
    every node/.style={transform shape},
    node distance=0.5cm,
    arrow/.style={-{Stealth[length=3mm]}, line width=1.5pt}
]

\node (obj1) {
    \begin{tabular}{p{4cm}|p{2.5cm}}
        \toprule
        \textbf{Objective Function} $f_i$ & \textbf{Coefficient (cost/time)} $\theta_i$ \\
        \midrule
        Tardiness of job $1$ & \textcolor{strblue}{???} \\
        \hline
        \hspace{0.1cm}$\vdots$ & \hspace{0.1cm}\textcolor{strblue}{$\vdots$} \\
        \hline
        Tardiness of job $D$ & \textcolor{strblue}{???} \\
        \bottomrule
    \end{tabular}
};

\coordinate (before-nw) at ($(obj1.north west) + (-0.3, 0.3)$);
\coordinate (before-se) at ($(obj1.south east) + (0.3, -0.5)$);
\draw[strblue, thick, rounded corners=5pt] (before-nw) rectangle (before-se);

\node[strblue, font=\bfseries, anchor=south west] at ($(before-nw |- before-se) + (0.05, 0.05)$) {Input 1};

\coordinate (before-bottom) at ($(before-nw)!0.5!(before-se) + (0, 0)$);
\coordinate (before-bottom-center) at ($(before-bottom |- before-se) + (0.5,0)$);

\node[below=3.0cm of obj1.south west, anchor=north west] (obj2) {
    \begin{tabular}{p{4cm}|p{2.5cm}}
        \toprule
        \textbf{Objective Function} $f_i$ & \textbf{Coefficient (cost/time)} $\theta_i$ \\
        \midrule
        Tardiness of job $1$ & \textcolor{strblue}{0.15} \\
        \hline
        \hspace{0.1cm}$\vdots$ & \hspace{0.1cm}\textcolor{strblue}{$\vdots$} \\
        \hline
        Tardiness of job $D$ & \textcolor{strblue}{0.31} \\
        \bottomrule
    \end{tabular}
};

\coordinate (after-nw) at ($(obj2.north west) + (-0.3, 0.3)$);
\coordinate (after-se) at ($(obj2.south east) + (0.3, -0.5)$);
\draw[strblue, thick, rounded corners=5pt] (after-nw) rectangle (after-se);

\node[strblue, font=\bfseries, anchor=south west] at ($(after-nw |- after-se) + (0.05, 0.05)$) {Output};

\coordinate (after-top-center) at (after-nw -| before-bottom-center);
\draw[strblue,arrow] (before-bottom-center) -- (after-top-center) node[midway, right] {\hspace{0.1cm}Inverse Optimization};

\coordinate (io-mid) at ($(before-bottom-center)!0.5!(after-top-center)$);

\coordinate (sched-pos) at ($(before-nw) + (0, 0)$);
\coordinate (sched-pos-y) at ($(sched-pos |- io-mid) + (0.0,0)$);

\node[anchor=west] at (sched-pos-y) (sched-text) {
    \textbf{\textcolor{strblue}{Input 2}}: Schedules
};

\coordinate (sched-nw) at ($(sched-text.north west) + (-0.3, 0.3)$);
\coordinate (sched-se) at ($(sched-text.south east) + (0.3, -0.3)$);
\draw[strblue, thick, rounded corners=5pt] (sched-nw) rectangle (sched-se);

\coordinate (sched-ne) at (sched-se |- sched-nw);
\coordinate (sched-right-mid) at ($(sched-ne)!0.5!(sched-se)$);

\draw[strblue,arrow] (sched-right-mid) -- (io-mid);

\end{tikzpicture}
\end{center}
\caption{Illustration of the DDIOP. This figure depicts the estimation of the objective-function coefficients such that the given shift schedule becomes an optimal solution.}
\label{fig:IOP_image}
\vspace{-1\intextsep}
\end{wrapfigure}